\documentclass[twoside,11pt]{article}

\usepackage{jmlr2e}
\usepackage{amsmath}
\usepackage[skip=0pt]{caption}
\usepackage{bm}
\usepackage{enumitem}
\usepackage[]{algorithm2e}
\usepackage{mathtools}
\usepackage{booktabs}
\usepackage{hyperref}
\usepackage{stmaryrd}

\newcommand{\RN}[1]{\textup{\uppercase\expandafter{\romannumeral#1}}}

\newcommand{\rstab}[3]{{#1}({#2})_{#3}|_{#3}}
\newcommand{\hcube}[2]{{\mathbb{Z}_{[{#1},{#2}]}^n}}

\newcommand{\ie}{i.e.\ }
\newcommand{\eg}{e.g.\ }
\newcommand{\aka}{a.k.a.\ }

\newcommand{\reals}{\mathbb{R}}
\newcommand{\integers}{\mathbb{Z}}

\newcommand{\acal}{\mathcal{A}}

\newcommand{\hcal}{\mathcal{H}}
\newcommand{\pcal}{\mathcal{P}}
\newcommand{\qcal}{\mathcal{Q}}
\newcommand{\rcal}{\mathcal{R}}

\newcommand{\transf}{{\mathsf{F}}}
\newcommand{\affine}{{\mathsf{AFF}}}
\newcommand{\orthmat}{{\mathsf{O}}}
\newcommand{\iso}{{\mathsf{ISO}}}
\newcommand{\tra}{{\mathsf{T}}}
\newcommand{\rot}{{\mathsf{R}}}
\newcommand{\lin}{{\mathsf{L}}}
\newcommand{\genlin}{{\mathsf{GL}}}
\newcommand{\perm}{{\mathsf P}}
\newcommand{\nega}{{\mathsf N}}
\newcommand{\dist}{{\rm d}}
\newcommand{\id}{{\rm id}}
\newcommand{\diag}{{\rm diag}}

\ShortHeadings{A Group-Theoretic Approach to Computational Abstraction}{Yu, Mineyev, and Varshney}
\firstpageno{1}

\begin{document}

\title{A Group-Theoretic Approach to Computational Abstraction: \\Symmetry-Driven Hierarchical Clustering}

\author{%
\name Haizi Yu
\email haiziyu7@illinois.edu \\
\addr Coordinated Science Laboratory \\
Department of Computer Science \\
University of Illinois at Urbana-Champaign \\
Urbana, IL 61801, USA
\AND
\name Igor Mineyev
\email mineyev@illinois.edu \\
\addr Department of Mathematics \\
University of Illinois at Urbana-Champaign \\
Urbana, IL 61801, USA
\AND
\name Lav R. Varshney
\email varshney@illinois.edu \\
\addr Coordinated Science Laboratory \\
Department of Electrical and Computer Engineering \\
University of Illinois at Urbana-Champaign \\
Urbana, IL 61801, USA
}

\editor{}

\maketitle


\begin{abstract}
Abstraction plays a key role in concept learning and knowledge discovery; this paper is concerned with \emph{computational abstraction}.
In particular, we study the nature of abstraction through a group-theoretic approach, formalizing it as symmetry-driven---as opposed to data-driven---hierarchical clustering.
Thus, the resulting clustering framework is \emph{data-free}, \emph{feature-free}, \emph{similarity-free}, and \emph{globally hierarchical}---the four key features that distinguish it from common data clustering models such as $k$-means.
Beyond a theoretical foundation for abstraction, we also present a top-down and a bottom-up approach to establish an algorithmic foundation for practical abstraction-generating methods.
Lastly, via both a theoretical explanation and a real-world application, we illustrate that further coupling of our abstraction framework with statistics realizes Shannon's information lattice and even further, brings learning into the picture.
This not only presents one use case of our proposed computational abstraction, but also gives a first step towards a principled and cognitive way of automatic concept learning and knowledge discovery.
\end{abstract}



\section{Introduction}
\label{sec:introduction}

Abstraction refers to the process of generalizing high-level concepts from specific instances by \emph{``forgetting the details''} \citep{Weinberg1968,GiunchigliaW1992,SaittaZ1998}.
This conceptual process is pervasive in humans, where more advanced concepts can be abstracted once a ``conceptual base'' is established \citep{Mandler2000}.
However, it remains mysterious how concepts are abstracted originally, with the source of abstraction generally attributed to innate biology \citep{Mandler2000,GomezL2004,Biederman1987}.

There are algorithms automating abstraction in various concept learning tasks \citep{SaittaZ2013,LecunBH2015,BredecheSZ2006,YuVGK2016};
yet, almost all require handcrafted priors---counterpart to innate biology \citep{Marcus2018,Dietterich2018}.
While a prior can take many forms, such as rules in automatic reasoning, distributions in Bayesian inference, features in classifiers, or architectures in neural networks, it is typically domain-specific \citep{RainaNK2006,YuJSD2007,KrupkaT2007}.
Sometimes, extensive hand design from domain knowledge is considered ``cheating'' if one hard codes all known abstractions as priors \citep{RamJ1994}.
This motivates us to consider a \emph{universal prior}---a prior that is generic and thus can be used across many subject domains.

In this paper, we study the nature of abstraction through a group-theoretic approach, adopting \emph{partition (clustering)} as its formulation and adopting \emph{symmetry} as its generating source.
Our abstraction framework is universal in two ways.
First, we consider the general question of conceptualizing a domain: a task-free preparation phase before specific problem solving \citep{Zucker2003}.
Second, we use symmetries in nature (groups in mathematics) as our universal prior.
The ultimate goal is to learn domain concepts/knowledge when our computational abstraction is attached to statistical learning.
This is contrary to much prior work at the intersection of group theory and learning \citep{Kondor2008} that often encodes domain-relevant symmetries in features or kernels rather than learning them as new findings.

The main contribution of this paper is to establish both a theoretical and an algorithmic foundation for our group-theoretic approach to computational abstraction:
the theory formalizes abstractions and abstraction hierarchy, whereas the algorithms give computational means to build them.
In the end, we exemplify a use case of our approach in music concept learning.
However, a more general and more comprehensive exposition of abstraction-based concept learning (requiring search in the abstraction hierarchy) is beyond the scope of this paper, and will be presented in subsequent work.

\subsection{Theoretical Foundation for Abstraction}
\label{sec:theoretical-foundation-for-abstraction}

Existing formalizations of abstraction form at least two camps.
One uses \emph{mathematical logic} where abstraction is explicitly constructed from abstraction operators and formal languages \citep{SaittaZ2013,Zucker2003,BundyGW1990};
another uses \emph{deep learning} where abstraction is hinted at by the layered architectures of neural networks \citep{LecunBH2015,Bengio2009}.
Their key properties---commonly known as rule-based (deductive) and data-driven (inductive)---are quite complimentary.
The former enjoys model interpretability, but requires handcrafting complicated logic with massive domain expertise;
the latter shifts the burden of model crafting to data, but makes the model less transparent.

This paper takes a new viewpoint, aiming for a middle ground between the two camps.
We formalize abstraction as \emph{symmetry-driven clustering} (as opposed to \emph{data clustering}).
By clustering, we forget within-cluster variations and discern only between-cluster distinctions \citep{BelaiJ1998,SheikhalishahiMT2016}, revealing the nature of abstraction as treating distinct instances the same \citep{Livingston1998}.
While data clustering is widely seen in machine learning, our clustering model differs from it in the following four ways.

\vspace{-0.05in}
\begin{itemize}
\setlength\itemsep{0in}
\item [1.] \textbf{Data-free.} Our clustering model considers partitioning an input space rather than data samples.
It is viewed more as conceptual clustering than data clustering like $k$-means \citep{MichalskiS1983,Fisher1987}: clusters are formed in a \emph{mechanism-driven}, not data-driven, fashion; and the mechanisms considered here are \emph{symmetries}.
So, the process is causal, and the results are interpreted as the consequences caused by the mechanims.
More importantly, a single \emph{clustering mechanism} (a symmetry) transfers to multiple domains, and a single \emph{clustering result} transfers to various datasets.
\item [2.] \textbf{Feature-free.} Our clustering model involves no feature engineering, so no domain expertise.
This particularly means three things.
First, no feature design for inputs: we directly deal with mathematical spaces, \eg vector spaces or manifolds.
Second, no feature/assignment function for cluster designation: this differs from algorithms that hand-design abstraction operators \citep{Zucker2003}, arithmetic descriptors \citep{YuVGK2016}, or decision-tree-like feature thresholding \citep{SheikhalishahiMT2016}.
Third, no meta-feature tuning such as pre-specifying the number of clusters.
\item [3.] \textbf{Similarity-free.} Our clustering model does not depend on a predefined notion of similarity.
This differs from most clustering algorithms where much effort has been expended in defining ``closeness'' \citep{RamanV2019,Rand1971}.
Instead, pairwise similarity is replaced by an \emph{equivalence relation} induced from symmetry.
Note that the definitions of certain symmetries may require additional structure of the input space, \eg topology or metric, but this is not used as a direct measurement for inverse similarity.
Therefore, points that are far apart (in terms of metric distance) in a metric space can be grouped together (in terms of equivalence) under certain symmetries, resulting in a ``disconnected'' cluster comprising disconnected regions in the input space.
This is not likely to happen for algorithms such as $k$-means.
\end{itemize}
\vspace{-0.05in}

It is noteworthy that being feature-free and similarity-free makes a clustering model \emph{universal} \citep{RamanV2019}, becoming more of a \emph{science} than an \emph{art} \citep{VonWG2012}.
Besides the above three distinguishing features, our clustering model exhibits one more distinction regarding hierarchical clustering for multi-level abstractions.

\vspace{-0.05in}
\begin{itemize}
\setlength\itemsep{0in}
\item [4.] \textbf{Global hierarchy.} Like many hierarchical clusterings \citep{JainD1988,RokachM2005}, our clustering model outputs a family of multi-level partitions and a hierarchy showing their interrelations.
However, here we have a global hierarchy formalized as a \emph{partition (semi)lattice}, which is generated from another hierarchy of symmetries represented by a subgroup lattice.
This is in contrast with greedy hierarchical clusterings such as agglomerative/divisive clustering \citep{Cormack1971,KaufmanR2009} or topological clustering via persistent homology \citep{Oudot2015}.
These greedy algorithms lose many possibilities for clusterings since the hierarchy is constructed by local merges/splits made in a one-directional procedure, \eg growing a dendrogram or a filtration.
In particular, greedy hierarchical clustering is oft-criticized since it is hard to recover from bad clusterings in early stages of construction \citep{Oudot2015}.
Lastly, our global hierarchy is represented by a directed acyclic graph rather than tree-like charts such as dendrograms or barcodes.
\end{itemize}
\vspace{-0.05in}

\subsection{Algorithmic Foundation for Abstraction}
\label{sec:algorithmic-foundation-for-abstraction}

To go from mathematically formalized abstractions to computational abstractions, we introduce two general principles---a top-down approach and a bottom-up approach---to algorithmically generate hierarchical abstractions from hierarchical symmetries enumerated in a systematic way.
The two principles leverage different dualities developed in the formalism, and lead to practical algorithms that realize the abstraction generating process.

\vspace{-0.05in}
\begin{itemize}
\setlength\itemsep{0in}
\item [1.] \textbf{A top-down approach.}
Starting from all possible symmetries, we gradually restrict our attention to certain types of symmetries in a principled way.
The restrictions are initially selected to attain practical abstraction-generating algorithms, and such selections can be arbitrarily made in general.
However, we find a complete identification of all symmetries induced from affine transformations, where we explicitly give a full parametrization of affine symmetries.
This complete identification not only decomposes a large symmetry-enumeration problem into smaller enumeration subproblems, but also suggests ways of adding restrictions to obtain desired symmetries.
This approach from general symmetries to more restrictive ones corresponds to top-down paths in the symmetry hierarchy, which explains the name.

\item [2.] \textbf{A bottom-up approach.}
Starting from a set of atomic symmetries (the seeds), we generate all symmetries that are seeded from the given set.
Based on a strong duality result developed in the formalism, we introduce an induction algorithm which computes a hierarchical family of abstractions without explicitly enumerating the corresponding symmetries.
This induction algorithm allowing abstractions to be made from earlier abstractions, is more efficient than generating all abstractions from scratch (\ie from symmetries).
It is good for quickly acquiring an abstraction family from user-customized seeds, after which one can further fine tune the generating set.
This approach from atomic symmetries to more complicated ones corresponds to bottom-up paths in the symmetry hierarchy, which explains the name.
\end{itemize}
\vspace{-0.05in}

\noindent
The main technical problem in this paper is: given an input space and a class of symmetries, compute a hierarchical family of symmetry-driven abstractions of the input space.
To motivate the problem, Section~\ref{sec:abstraction-informal-description} informally describes abstraction to help picture the intuitions that will drive the following technical sections.
To formalize the problem, Section~\ref{sec:abstraction-math-formalism} sets up the theoretical foundation.
It formalizes abstraction, symmetry, and hierarchy, and casts symmetry-driven abstractions in a primal-dual viewpoint.
To solve the problem, Sections~\ref{sec:the-top-down-approach-sp-subgroups} and \ref{sec:the-bottom-up-approach-gen-set} set up the algorithmic foundation.
Regarding the top-down and the bottom-up approaches respectively, the symmetry class in the problem takes the form of all subgroups of a group (usually up to an equivalence relation, \eg isomorphism) and subgroups generated by all subsets of a generating set.
Section~\ref{sec:restrict-to-finite-subspaces} presents implementation tricks and cautions in a more general setting where abstractions are restricted to finite subspaces of a possibly infinite input space.
To show a use case of the problem, Section~\ref{sec:discussion-to-info-lattice-and-learning} discusses connections to Shannon's information lattice---a special case under our abstraction formalism---and a music application that realizes learning in an information lattice to learn music concepts.


\section{Abstraction: Informal Description}
\label{sec:abstraction-informal-description}

We informally discuss \emph{abstraction} to provide initial intuitions.
To do so, we consider toy examples and examples from everyday life to illustrate abstraction in specific contexts.
Nevertheless, the intuitions here cover everything about abstraction that we will generalize, formalize, and further operationalize (computationally) in the following sections.
Therefore, this section, despite its informality, establishes an agenda for the remainder of the paper.

\subsection{Nature of Abstraction}
\label{sec:nature-of-abstraction}

Abstraction is everywhere in our daily life (even though we may not realize it all the time).
Examples of people making abstractions can be as simple as observing ourselves through social categories such as race or gender \citep{MacraeB2000};
or as complicated as a systematic taxonomy of a subject domain, such as modern taxonomy of animals and Western classification of music chords.
While the contexts might change, there are common intuitions for viewing the nature of abstraction.

\begin{figure}[t]
\begin{center}
\includegraphics[width=0.72\columnwidth]{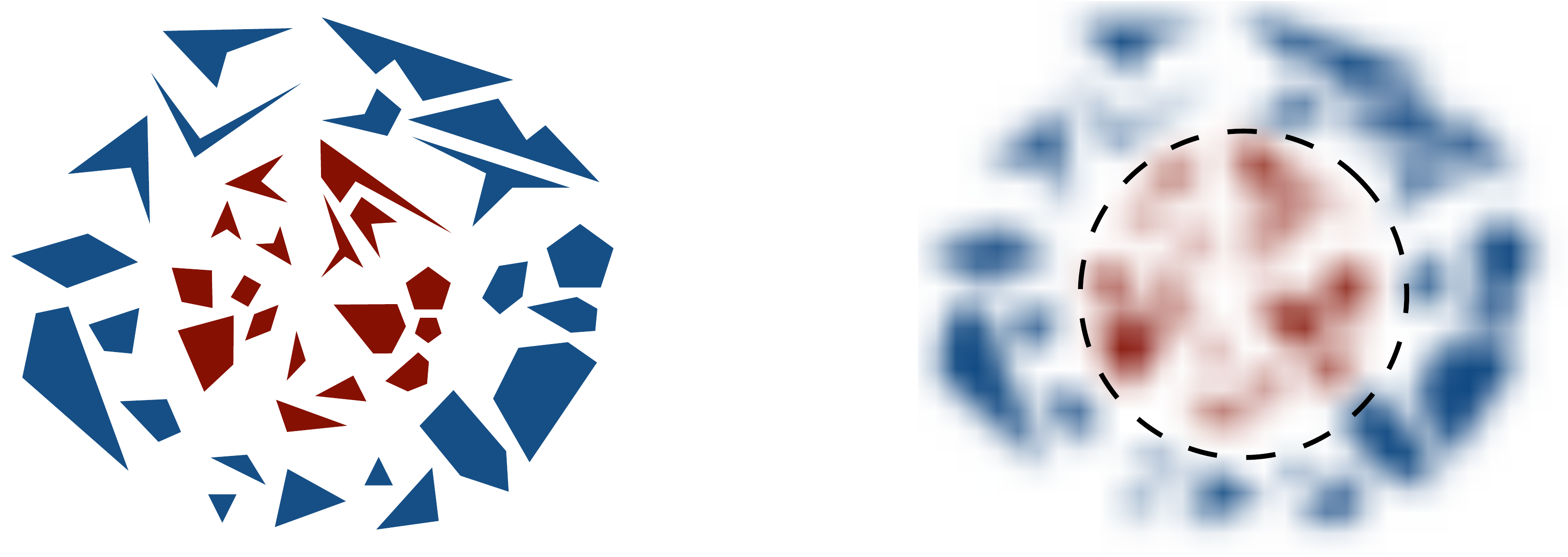}
\end{center}
\caption{By clustering, an abstraction of polygons presents a coarse-grained view of the given polygons, where we forget all shape-related information and only discern their colors. This results in two clusters: \{red polygons, blue polygons\}.}
\label{fig:abstraction-nature}
\end{figure}

\paragraph{Abstraction as partition (clustering).}
One way to view an abstraction is through a \emph{partition}, or \emph{clustering}, and then \emph{forgetting} within-cluster variations.
For instance, we cluster people into \{men, women\}, forgetting the difference between John and David, Mary and Rachel;
we cluster vertibrates into \{fish, amphibians, reptiles, birds, mammals\}, forgetting the difference between penguins and eagles, dogs and bats;
we cluster music triads into \{major, minor, augmented, diminished, $\ldots$\}, forgetting the difference between CEG and FAC, CE$\flat$G and ACE.

\paragraph{Abstraction as equivalence relation.}
Another (equivalent) way to view an abstraction is through an \emph{equivalence relation}, where the earlier idea of ``clustering and forgetting'' is reinterpreted as identifying things (in the same cluster) indistinguishably.
For instance, identifying John and David (as men); identifying dogs and bats (as mammals); identifying CEG and FAC (as major triads).
This view is consistent with mathematics, since a partition of a set is basically the same thing as an equivalence relation on that set.
This view is also consistent with psychology, stating the nature of abstraction is to treat instances as if they were qualitatively identical, although in fact they are not \citep{Livingston1998}.

\paragraph{Classification is more than needed.}
This idea of clustering is pervasive in various definitions of abstraction, but more often termed as classification (or categorization, taxonomy).
While clustering and classification (likewise clusters and classes) are more or less synonyms in everyday life, they are clearly different in machine learning.
The former generally falls under unsupervised learning, whereas the latter falls under supervised learning.
The difference is merely whether or not there is a label for each cluster.
Labels are important in supervised learning, since a perfect binary classifier with a 100\% accuracy is clearly different from a bad one with a 0\% accuracy.
Yet in light of clustering, the two classifiers are identical: the ``bad'' one, for instance, simply calls all men as women and all women as men, but still accurately captures the concept of gender.
So, to make this abstraction, all we need are two clusters of people;
yet, labeling the two by men and women is not essential.

In summary, treating an abstraction as clustering (or a partition, an equivalence relation) presents a \emph{coarse-grained view} (Figure~\ref{fig:abstraction-nature}) of the observed instances, where we deliberately forget within-cluster variations by collapsing equivalent instances into one cluster and discerning only between-cluster variations.
The intuitions here are formalized in Section~\ref{sec:abstraction-as-partition}.

\subsection{Abstraction Hierarchy}
\label{sec:abstraction-hierarchy}

More than one abstraction exists for the same set of instances, simply because there are many different ways of partitioning the set and there are many different ways of identifying things to be equivalent.
For instance, considering \{bats, dogs, eagles\}, we can identify bats and dogs indistinguishably as mammals, but distinguish them from eagles which are birds; we can also identify bats and eagles indistinguishably as flying animals, but now distinguish them from dogs since dogs cannot fly.
Intuitively, how different abstractions compare with one another induces the notion of a hierarchy.

\begin{figure}[t]
\begin{center}
\includegraphics[width=0.92\columnwidth]{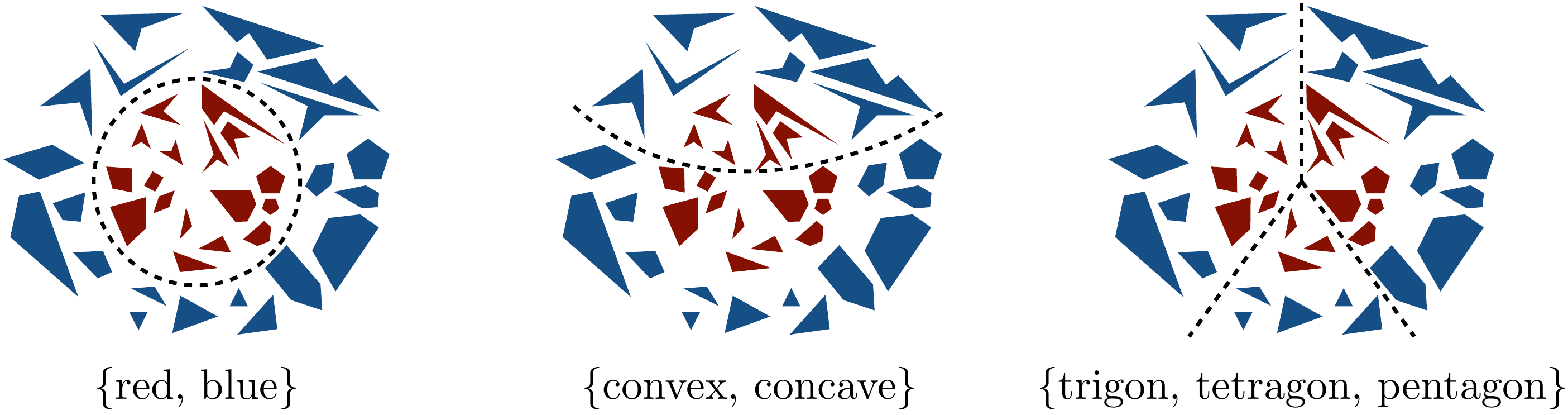}
\end{center}
\caption{Different and incomparable abstractions of the same set of polygons.}
\label{fig:abstraction-incomp}
\end{figure}

\paragraph{Linear hierarchy: total order.}
Perhaps the simplest is a linear hierarchy, where ``later'' abstractions are made from their immediate ``precursors'' by continuously merging clusters.
For instance, we cluster animals into \{fish, birds, mammals, annelids, mollusks, $\ldots$\}, and further cluster these abstracted terms into \{vertebrates, invertebrates\}.
A linear hierarchy is essentially a totally (\aka linearly) ordered set;
it has a clear notion of \emph{abstraction level}, \eg in animal taxonomy: species $\rightarrow$ genus $\rightarrow$ family $\rightarrow$ order $\rightarrow$ class $\rightarrow$ phylum $\rightarrow$ kingdom.
Notably, the famous agglomerative hierarchical clustering is an example of linear hierarchy, producing a linear sequence of coarser and coarser partitions.

\paragraph{More than linear: partial order.}
Nevertheless, an abstraction hierarchy can be more complicated than simply linear due to various clustering possibilities that are incomparable.
For instance, take the same set of polygons from Figure~\ref{fig:abstraction-nature}.
Besides color, consider two other ways of partitioning the same set (Figure~\ref{fig:abstraction-incomp}).
These three abstractions are made from three incomparable criteria (color, convexity, number of sides), and the resulting partitions are also incomparable (none of them is made from merging clusters from either of the others).
Hence, instead of a linear hierarchy, a family of abstractions forms a partially ordered set in general, and the notion of abstraction level is only relative due to incomparability.
This is one way in which this paper extends traditional (linear) hierarchical clustering.

\paragraph{More than partial order: lattice.}
Given any two abstractions as two coarse-grained views of the same set, there is a clear notion of their so-called \emph{finest common coarsening} and oppositely, their \emph{coarsest common refinement}.
We exemplify the latter in the foreground of Figure~\ref{fig:abstraction-hierarchy}, where a color-based clustering and a shape-based clustering uniquely determines a new clustering based on both color and shape.
The initial two abstractions of polygons, together with their coarsest common refinement, are part of a bigger hierarchy (Figure~\ref{fig:abstraction-hierarchy}), which is a special type of partially ordered set called a \emph{lattice}.
Intuitions developed thus far about abstraction hierarchy motivate their formal descriptions in Section~\ref{sec:abstraction-universe-as-partition-lattice}.

\begin{figure}[t]
\begin{center}
\includegraphics[width=0.95\columnwidth]{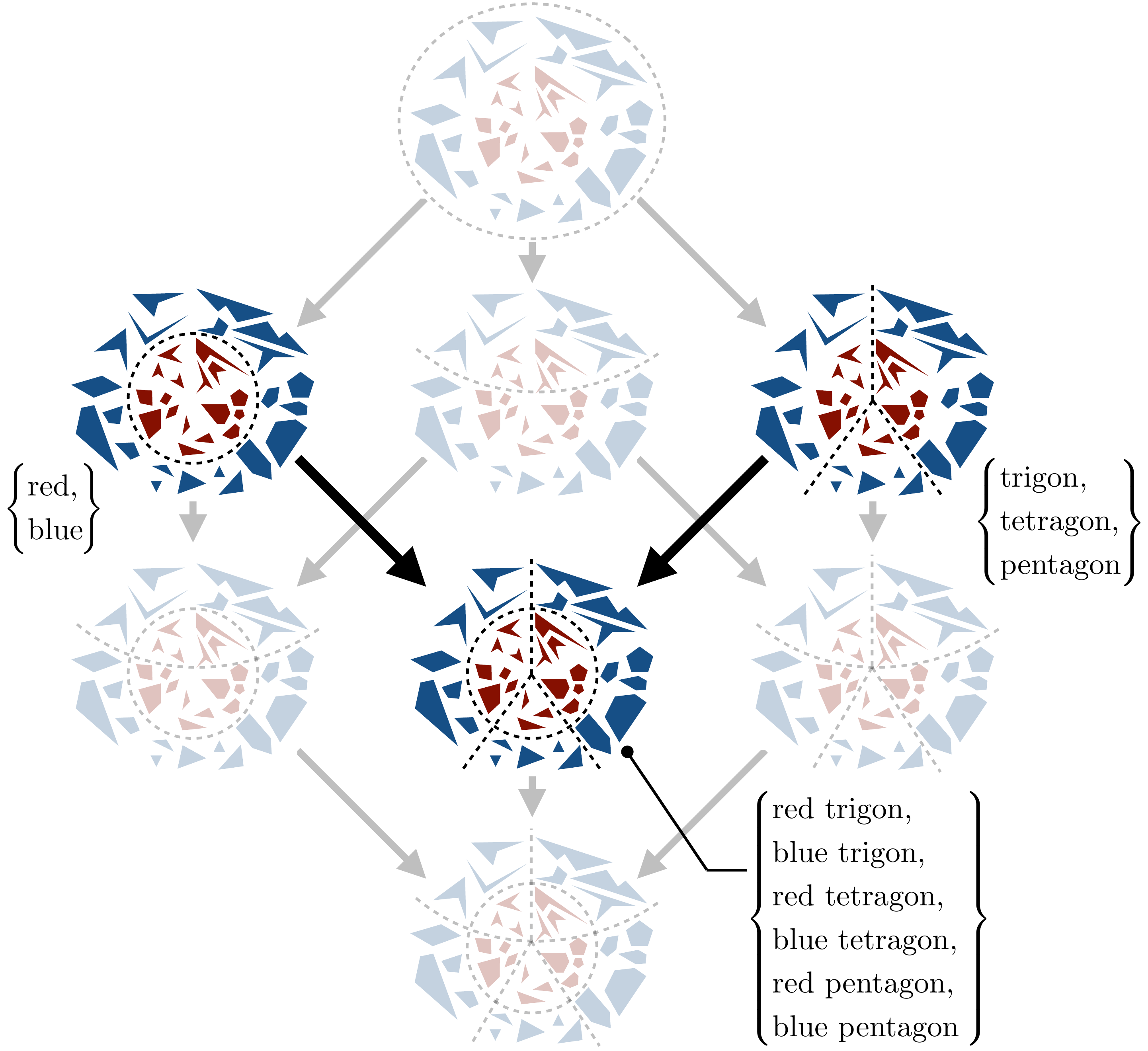}
\end{center}
\caption{Two abstractions and their coarsest common refinement (foreground) as part of a global lattice structure (foreground \& background)---a special type of partially ordered set.
This whole lattice is pictured as a directed acyclic graph, where every vertex is a distinct abstraction of the given polygons and every edge $\acal \to \acal'$ denotes $\acal$ is coarser than $\acal'$.}
\label{fig:abstraction-hierarchy}
\end{figure}

\subsection{Abstraction Mechanisms}
\label{sec:abstraction-mechanisms}

In many cases, an abstraction is made from an explicable reason---a \emph{mechanism} that drives the resulting abstraction.
For instance, color is the underlying mechanism that drives the abstraction in Figure~\ref{fig:abstraction-nature}.
An abstraction mechanism also naturally serves as a \emph{prior} with respect to its induced clustering process.
In this paper, we focus on abstractions made from clear mechanisms (rather than random clustering), and furthermore, from mechanisms that can be used as universal priors for multiple subject domains (Figure~\ref{fig:abstraction-mechanisms}).

\begin{figure}[t]
\begin{center}
\includegraphics[width=0.68\columnwidth]{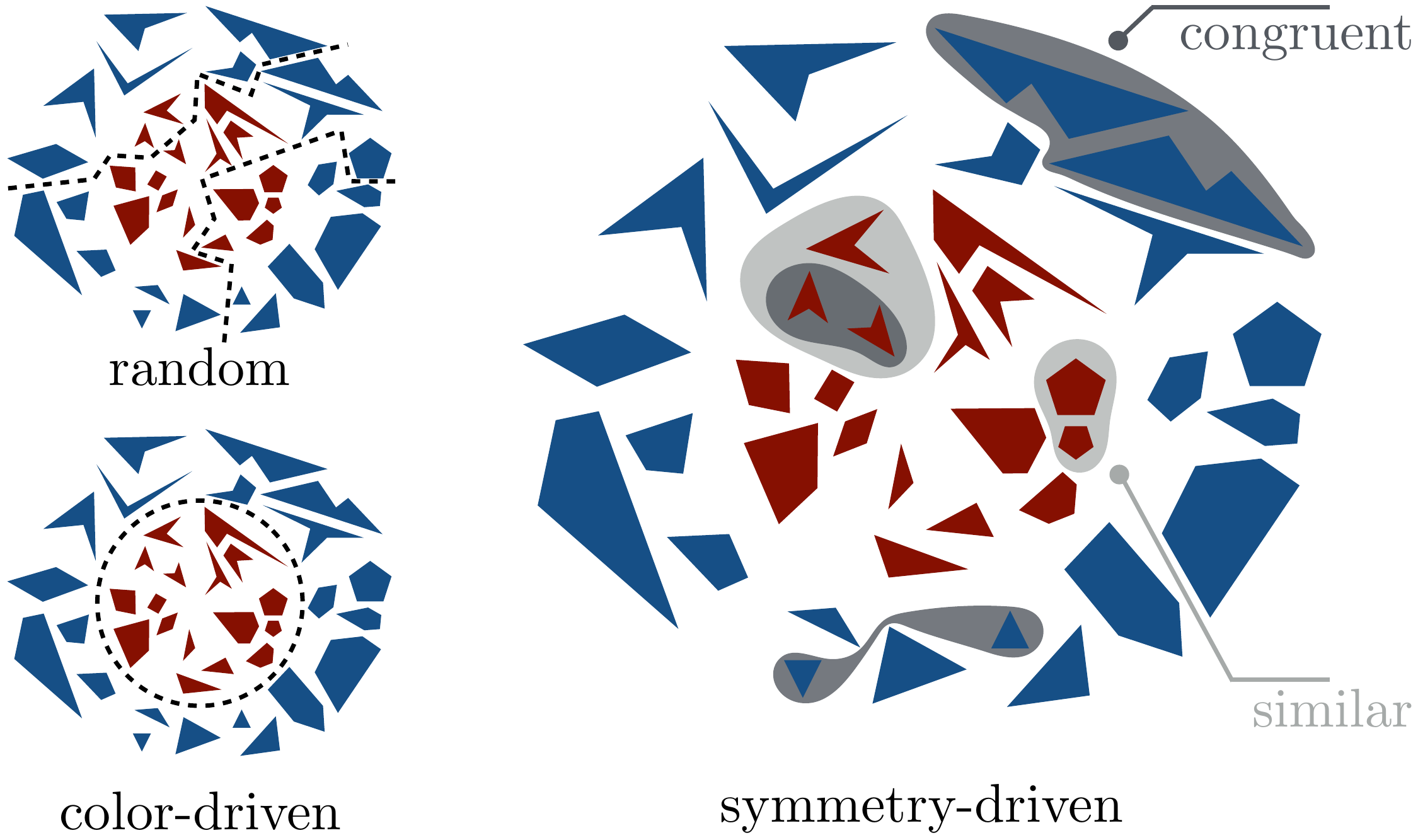}
\end{center}
\caption{Various kinds of abstraction mechanisms.}
\label{fig:abstraction-mechanisms}
\end{figure}

\paragraph{Mechanism-driven.}
In three ways, we prefer mechanism-driven abstractions to random abstractions.
First, for the sake of \emph{explainability}, the recognition of a mechanism allows us to understand the abstraction (\eg from \{men, women\}, we recognize gender is under consideration).
Second, for the sake of \emph{generalizability}, the presence of a mechanism may allow us to transfer the same mechanism to a different set (\eg from \{men, women\} to \{roosters, hens\}, \{bulls, cows\}).
Third, for the sake of \emph{computability}, restriction to a certain category of mechanisms allows us to practically generate all abstractions under the chosen category rather than to aimlessly consider all possible partitions (of a set), whose number grows faster than exponential (with respect to the size of the set).

\paragraph{Symmetry-driven.}
Mechanisms are of various kinds.
In data-driven abstractions (\ie data clustering), a mechanism is a pre-defined notion of data proximity.
In category-driven abstractions (\eg taxonomy of animals), a mechanism is usually a defining attribute of the instances themselves (ontology).
Hence, while an attribute like number of sides is a valid mechanism for abstracting polygons, it is not so when trying to abstract animals.
Rather than these two kinds of mechanisms which are often ad hoc and domain-specific, we study symmetry-driven abstractions, drawing general symmetries from nature.
Again, take the polygons in Figure~\ref{fig:abstraction-nature} as an example: we can cluster all congruent polygons together to get an abstraction admitting translation and rotation invariances; we can cluster all similar polygons together to get a coarser abstraction admitting an additional scaling invariance (Figure~\ref{fig:abstraction-mechanisms}).
These intuitions about being symmetry-driven are formalized in Section~\ref{sec:symmetry-generated-abstraction}.


\newpage
\section{Abstraction: Mathematical Formalism}
\label{sec:abstraction-math-formalism}

We formalize an abstraction process on an underlying space as a clustering problem.
In this process, elements of the space are grouped into clusters, abstracting away within-cluster variations.
The outcome is a coarse-grained abstraction space whose elements are the clusters.
Clustering is performed based on certain symmetries such that the resulting clusters are invariant with respect to the symmetries.

\subsection{Abstraction as Partition (Clustering)}
\label{sec:abstraction-as-partition}

We formalize an \emph{abstraction} of a set as a partition of the set, which is a mathematical representation of the outcome of a clustering process.
Throughout this paper, we reserve $X$ to exclusively denote a set which we make abstractions of.
The set $X$ can be as intangible as a mathematical space, \eg $\reals^n$, $\integers^n$, a general manifold;
or as concrete as a collection of items, \eg \{rat, ox, tiger, rabbit, dragon, snake, horse, sheep, monkey, rooster, dog, pig\}.

\paragraph{Preliminaries (Appendix~\ref{app:abstraction-as-partition}):}
partition of a set ($\pcal$), partition cell ($P \in \pcal$); equivalence relation on a set ($\sim$), quotient ($X{/\!\!\sim}$).

\begin{remark}
An abstraction is a partition, and vice versa.
The two terms refer to the same thing, with the only nuance being that one is used less formally, whereas the other is used in the mathematical language.
When used as a single noun, these two terms are interchangeable in this paper.
\end{remark}

\begin{remark}
A partition is not an equivalence relation.
The two terms do not refer to the same thing (one is a set, the other is a binary relation), but convey equivalent ideas since they induce each other bijectively (Appendix~\ref{app:abstraction-as-partition}).
In this paper, we use an equivalence relation to explain a partition: elements of a set $X$ are put in the same cell because they are equivalent.
Based on this reason, abstracting the set $X$ is about treating equivalent elements as the same, \ie collapsing equivalent elements in $X$ into a single entity (namely, an equivalence class or a cell) where collapsing is formalized by taking the quotient.
\end{remark}

\subsection{Abstraction Universe as Partition Lattice (Hierarchical Clustering)}
\label{sec:abstraction-universe-as-partition-lattice}

A set $X$ can have multiple partitions, provided that $|X| > 1$.
The number of all possible partitions of a set $X$ is called the \emph{Bell number} $B_{|X|}$.
Bell numbers grow extremely fast with the size of the set: starting from $B_0 = B_1 = 1$, the first few Bell numbers are:
\begin{align*}
1, 1, 2, 5, 15, 52, 203, 877, 4140, 21147, 115975, 678570, 4213597, 27644437, \ldots
\end{align*}
We use $\mathfrak{P}_X^{*}$ to denote the family of all partitions of a set $X$, so $|\mathfrak{P}_X^{*}| = B_{|X|}$.
We can compare partitions of a set in two ways.
One simple way is to compare by size: given two partitions $\pcal, \qcal$ of a set, we say that $\pcal$ is no larger than (resp.\ no smaller than) $\qcal$ if $|\pcal| \leq |\qcal|$ (resp.\ $|\pcal| \geq |\qcal|$).
Another way of comparison considers the structure of partitions via a \emph{partial order} on $\mathfrak{P}_X^{*}$.
The partial order further yields a \emph{partition lattice}, a hierarchical representation of a family of partitions.

\paragraph{Preliminaries (Appendix~\ref{app:abstraction-universe-as-partition-lattice}):}
partial order, poset; lattice, join ($\vee$), meet ($\wedge$), sublattice, join-semilattice, meet-semilattice, bounded lattice.

\begin{definition}\label{def:higher-level-abstraction}
Let $\pcal$ and $\qcal$ be two abstractions of a set $X$. We say that $\pcal$ is at a higher level than $\qcal$, denoted $\pcal \preceq \qcal$, if as partitions, $\pcal$ is coarser than $\qcal$. For ease of description, we expand the vocabulary for this definition, so the following are all equivalent:
\begin{itemize}
\setlength\itemsep{0in}
\item [1.] $\pcal \preceq \qcal$, or equivalently $\qcal \succeq \pcal$ (Figure~\ref{fig:abstraction-level}).
\item [2.] As abstractions, $\pcal$ is at a higher level than $\qcal$ (or $\pcal$ is an abstraction of $\qcal$).
\item [3.] As partitions, $\pcal$ is coarser than $\qcal$ (or $\pcal$ is a coarsening of $\qcal$).
\item [4.] As abstractions, $\qcal$ is at a lower level than $\pcal$ (or $\qcal$ is a realization of $\pcal$).
\item [5.] As partitions, $\qcal$ is finer than $\pcal$ (or $\qcal$ is a refinement of $\pcal$).
\item [6.] Any $x,x' \in X$ in the same cell in $\qcal$ are also in the same cell in $\pcal$.
\item [7.] Any $x,x' \in X$ in different cells in $\pcal$ are also in different cells in $\qcal$.
\end{itemize}
\begin{figure}[h!]
\begin{center}
\includegraphics[width=0.65\columnwidth]{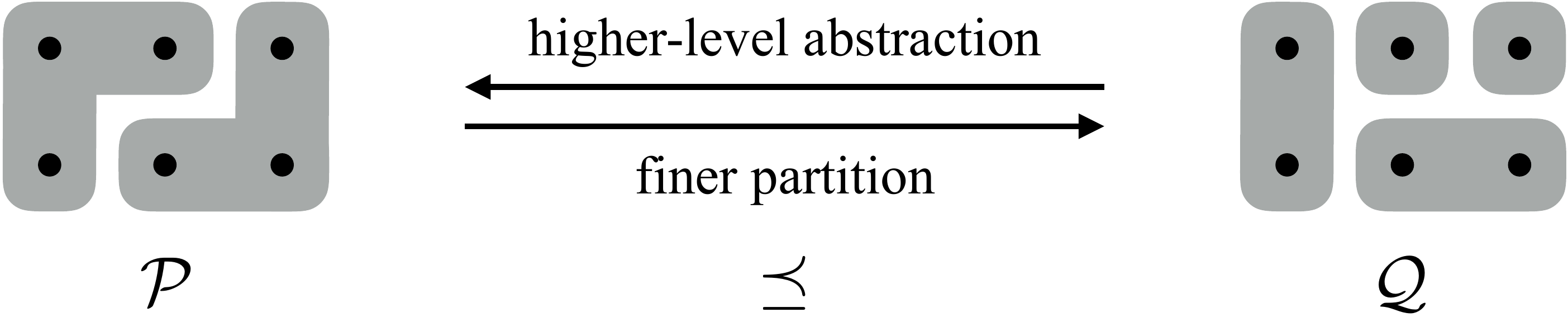}
\end{center}
\caption{The partial order $\preceq$ compares the levels of abstractions.}
\label{fig:abstraction-level}
\end{figure}
\end{definition}

It is known that the binary relation ``coarser than'' on the family $\mathfrak{P}_X^{*}$ of all partitions of a set $X$ is a partial order, so is the binary relation ``at a higher level than'' on abstractions.
Given two partitions $\pcal, \qcal$ of a set, we can have $\pcal \preceq \qcal$, $\qcal \preceq \pcal$, or they are incomparable.
Further, $(\mathfrak{P}_X^{*}, \preceq)$ is a bounded lattice, in which the greatest element is the finest partition $\{\{x\} \mid x \in X\}$ and the least element is the coarsest partition $\{X\}$.
For any pair of partitions $\pcal, \qcal \in \mathfrak{P}_X^{*}$, their join $\pcal \vee \qcal$ is the coarsest common refinement of $\pcal$ and $\qcal$; their meet $\pcal \wedge \qcal$ is the finest common coarsening of $\pcal$ and $\qcal$ (Figure~\ref{fig:partn-join-and-meet}).
\begin{figure}[h!]
\begin{center}
\includegraphics[width=0.65\columnwidth]{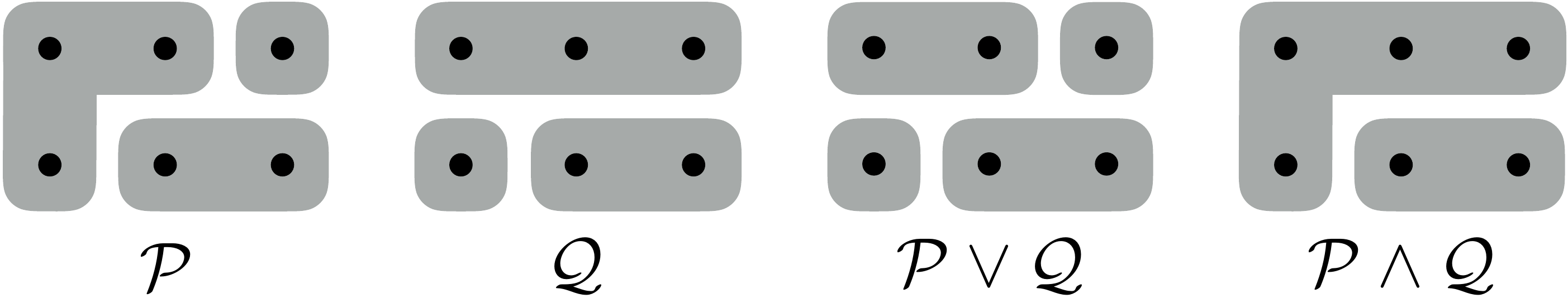}
\end{center}
\caption{Two abstractions $\pcal, \qcal$ and their join $\pcal \vee \qcal$ and meet $\pcal \wedge \qcal$.}
\label{fig:partn-join-and-meet}
\end{figure}

\begin{definition}
An abstraction universe for a set $X$ is a sublattice of $\mathfrak{P}_X^{*}$, or a partition (sub)lattice in short.
In particular, we call the partition lattice $\mathfrak{P}_X^{*}$ itself the complete abstraction universe for $X$.
An abstraction join-semiuniverse (resp.\ meet-semiuniverse) for a set $X$ is a join-semilattice (resp.\ meet-semilattice) of $\mathfrak{P}_X^{*}$.
An abstraction family for a set $X$, an even weaker notion, is simply a subset of $\mathfrak{P}_X^{*}$.
\end{definition}

If the complete abstraction universe $(\mathfrak{P}_X^{*}, \preceq)$ is finite, we can visualize its hierarchy as a directed acyclic graph where vertices denote partitions and edges denote the partial order.
The graph is constructed as follows: plot all distinct partitions of $X$ starting at the bottom with the finest partition $\{\{x\} \mid x \in X\}$, ending at the top with the coarsest partition $\{X\}$ and, roughly speaking, with coarser partitions positioned higher than finer ones.
Draw edges downwards between partitions using the rule that there will be an edge downward from $\pcal$ to $\qcal$ if $\pcal \preceq \qcal$ and there does not exist a third partition $\rcal$ such that $\pcal \preceq \rcal \preceq \qcal$.
Thus, if $\pcal \preceq \qcal$, there is a path (possibly many paths) downward from $\pcal$ to $\qcal$ passing through a chain of intermediate partitions (and a path upward from $\qcal$ to $\pcal$ if $\qcal \succeq \pcal$).
For any pair of partitions $\pcal, \qcal \in \mathfrak{P}_X^{*}$, the join $\pcal \vee \qcal$ can be read from the graph as follows: trace paths downwards from $\pcal$ and $\qcal$ respectively until a common partition $\rcal$ is reached (note that the finest partition $\{\{x\} \mid x \in X\}$ at the bottom is always the end of all downward paths in the graph, so it is guaranteed that $\rcal$ always exists).
To ensure that $\rcal = \pcal \vee \qcal$, make sure there is no $\rcal' \preceq \rcal$ (indicated by an upward path from $\rcal$ to $\rcal'$) with upward paths towards both $\pcal$ and $\qcal$ (otherwise replace $\rcal$ with $\rcal'$ and repeat the process).
Symmetrically, one can read the meet $\pcal \wedge \qcal$ from the graph.

There are limitations to this process, especially if the set $X$ is infinite.
Even for a finite set $X$ of relatively small size, the complete abstraction universe $\mathfrak{P}_X^{*}$ can be quite complicated to visualize (recall that we have to draw $|\mathfrak{P}_X^{*}| = B_{|X|}$ vertices where $B_{|X|}$ grows extremely fast with $|X|$, let alone the edges).
However, not all arbitrary partitions are of interest to us.
In the following subsections, we study symmetry-generated abstractions and abstraction universes.
So, later we can focus on certain partitions by considering certain symmetries.

\subsection{Symmetry-Generated Abstraction}
\label{sec:symmetry-generated-abstraction}

Recall that we explain an abstraction of a set by its inducing equivalence relation, where equivalent elements are treated as the same.
Instead of considering arbitrary equivalence relations or arbitrary partitions, we construct every abstraction from an explicit mechanism---a symmetry---so the resulting equivalence classes or partition cells are invariant under this symmetry.
To capture various symmetries, we consider groups and group actions.

\paragraph{Preliminaries (Appendix~\ref{app:symmetry-generated-abstraction}):}
group ($(G,*)$ or $G$), subgroup ($\leq$), trivial subgroup ($\{e\}$), subgroup generated by a set ($\langle S \rangle$), cyclic subgroup ($\langle s \rangle$); group action, $G$-action on $X$ ($\cdot: G \times X \to X$), orbit of $x \in X$ ($Gx$), set of all orbits ($X/G$).
\vspace{0.1in}

Consider a special type of group, namely the symmetric group $(S_X, \circ)$ defined over a set $X$, whose group elements are all the bijections from $X$ to $X$ and whose group operation is (function) composition.
The identity element of $S_X$ is the identity function, denoted $\id$.
A bijection from $X$ to $X$ is also called a \emph{transformation} of $X$.
Therefore, the symmetric group $S_X$ comprises all transformations of $X$, and is also called the transformation group of $X$, denoted $\transf(X)$.
We use these two terms and notations interchangeably in this paper, with a preference for $\transf(X)$ in general, while reserving $S_X$ mostly for a finite $X$.

Given a set $X$ and a subgroup $H \leq \transf(X)$, we define an $H$-action on $X$ by $h\cdot x := h(x)$ for any $h \in H, x \in X$; the \emph{orbit} of $x \in X$ under $H$ is the set $Hx := \{h(x) \mid h \in H\}$.
Orbits in $X$ under $H$ define an equivalence relation: $x \sim x'$ if and only if $x,x'$ are in the same orbit, and each orbit is an equivalence class.
Thus, the quotient $X/H = X{/\!\!\sim}$ is a partition of $X$. It is known that every cell (or orbit) in the abstraction (or quotient) $X/H$ is a minimal non-empty invariant subset of $X$ under transformations in $H$.
Therefore, we say this abstraction respects the so-called \emph{$H$-symmetry} or \emph{$H$-invariance}.

We succinctly record the above process of constructing an abstraction $X/H$ (of $X$) from a given subgroup $H \leq \transf(X)$ in the following \emph{abstraction generating chain}:
\begin{align*}
\mbox{a subgroup of $\transf(X)$}
\xrightarrow[]{\textit{group action}}
\mbox{orbits}
\xrightarrow[]{\textit{equiv. rel.}}
\mbox{a partition}
\xrightarrow[]{\textit{is}}
\mbox{an abstraction of $X$},
\end{align*}
which can be further encapsulated by the \emph{abstraction generating function} defined as follows.

\begin{definition}\label{def:abstraction-generating-function}
The abstraction generating function is the mapping $\pi: \hcal_{\transf(X)}^{*} \to \mathfrak{P}_X^{*}$ where $\hcal_{\transf(X)}^{*}$ is the collection of all subgroups of $\transf(X)$, $\mathfrak{P}_X^{*}$ is the family of all partitions of $X$, and for any $H \in \hcal_{\transf(X)}^{*}$, $\pi(H) := X/H := \{Hx \mid x \in X\}$, where $Hx := \{h(x) \mid h \in H\}$.
\end{definition}

\begin{theorem}\label{thm:abstraction-generating-function-not-injective}
The abstraction generating function $\pi: \hcal_{\transf(X)}^{*} \to \mathfrak{P}_X^{*}$ is not necessarily injective.
\end{theorem}
\begin{proof}
Let $X = \{1,2,3,4\}$ and $h = (1234),g = (1324) \in S_4 = \transf(X)$ be two transformations (also known as permutations, in the cycle notation) of $X$; consider the cyclic groups:
\begin{align*}
H = \langle h \rangle &:= \{h^n \mid n \in \integers\} = \{\id, h, h^2, h^3\} = \{\id,(1234), (13)(24), (1432)\};\\
G = \langle g \rangle &:=  \{g^n \mid n \in \integers\} = \{\id, g, g^2, g^3\} = \{\id, (1324), (12)(34), (1423)\}.
\end{align*}
It is clear that $H \neq G$ but $\pi(H) = \pi(G) = \{\{1,2,3,4\}\}$, the coarsest partition of $X$.
\end{proof}

\begin{theorem}\label{thm:abstraction-generating-function-surjective}
The abstraction generating function $\pi: \hcal_{\transf(X)}^{*} \to \mathfrak{P}_X^{*}$ is surjective.
\end{theorem}
\begin{proof}
For any $a,b \in X$, let $f_{a,b}: X \to X$ be the bijective function of the form
\begin{align*}
f_{a,b}(x) = \begin{cases}
a & x = b, \\
b & x = a, \\
x & otherwise.
\end{cases}
\end{align*}
Pick any partition $\pcal \in \mathfrak{P}_X^{*}$. For any cell $P \in \pcal$, define
\begin{align*}
S_P := \left\{f_{a,b} \mid a,b \in P, a\neq b\right\}, \quad \mbox{ and let } H := \left\langle \bigcup_{P \in \pcal}S_P \right\rangle.
\end{align*}
We claim $\pi(H) = \pcal$.
To see this, for any distinct $x,x' \in X$ that are in the same cell in $\pcal$, $f_{x,x'} \in S_P$ for some $P \in \pcal$, so $f_{x,x'} \in H$. This implies that $x$ and $x'$ are in the same orbit in $\pi(H)$, since $x' = f_{x,x'}(x)$. Therefore, $\pi(H) \preceq \pcal$.
Conversely, for any distinct $x,x' \in X$ that are in the same orbit in $\pi(H)$, there exists an $h \in H$ such that $x' = h(x)$. By definition, $h = h_k \circ \cdots \circ h_1$ for some finite integer $k > 0$ where $h_k, \ldots, h_1 \in \cup_{P \in \pcal}S_P$. Suppose $P' \in \pcal$ is the cell that $x$ is in, \ie $x \in P'$, then $h_1(x) \in P'$, since $h_1(x) \in P'$ if $h_1 \in S_{P'}$ and $h_1(x) = x$ otherwise. Likewise, we have $h_2 \circ h_1(x), h_3 \circ h_2 \circ h_1(x), \ldots, h_k \circ \cdots \circ h_1(x) \in P'$. This implies that $x' = h(x) = h_k \circ \cdots \circ h_1(x) \in P'$, \ie $x$ and $x'$ are in the same cell in $\pcal$. Therefore, $\pcal \preceq \pi(H)$.
Combining both directions yields $\pi(H) = \pcal$, so $\pi$ is surjective.
\end{proof}

\subsection{Duality: from Subgroup Lattice to Abstraction (Semi)Universe}
\label{sec:duality-subgroup-to-abstraction}

Given a subgroup of $\transf(X)$, we can generate an abstraction of $X$ via the abstraction generating function $\pi$.
Thus, given a collection of subgroups of $\transf(X)$, we can generate a family of abstractions of $X$.
Further, given a collection of subgroups of $\transf(X)$ with a hierarchy, we can generate a family of abstractions of $X$ with an induced hierarchy.
This leads us to a \emph{subgroup lattice} generating a partition (semi)lattice, where the latter is dual to the former via the abstraction generating function $\pi$.

\paragraph{Preliminaries (Appendix~\ref{app:duality_subgroup_to_abstraction}):}
the (complete) subgroup lattice for a group ($\hcal_G^{*}$, $\leq$), join ($A \vee B = \langle A\cup B \rangle$), meet ($A \wedge B = A \cap B$).
\vspace{0.1in}

We consider the subgroup lattice for $\transf(X)$, denoted $(\hcal_{\transf(X)}^{*}, \leq)$.
Similar to the complete abstraction universe $(\mathfrak{P}_X^{*}, \preceq)$, we can draw a directed acyclic graph to visualize $(\hcal_{\transf(X)}^{*}, \leq)$ if it is finite, where vertices denote subgroups and edges denote the partial order. The graph is similarly constructed by plotting all distinct subgroups of $\transf(X)$ starting at the bottom with $\{\id\}$, ending at the top with $\transf(X)$ and, roughly speaking, with larger subgroups positioned higher than smaller ones.
Draw an upward edge from $A$ to $B$ if $A \leq B$ and there are no subgroups properly between $A$ and $B$. For any pair of subgroups $A,B \in \hcal_{\transf(X)}^{*}$, the join $A \vee B$ can be read from the graph by tracing paths upwards from $A$ and $B$ respectively until a common subgroup containing both is reached, and making sure there are no smaller such subgroups; the meet $A \wedge B$ can be read from the graph in a symmetric manner.
For any subgroup $C \in \hcal_{\transf(X)}^{*}$, the subgroup sublattice $(\hcal_{C}^{*}, \leq)$ for $C$ is part of the subgroup lattice $(\hcal_{\transf(X)}^{*}, \leq)$ for $\transf(X)$, which can be read from the graph for $(\hcal_{\transf(X)}^{*}, \leq)$ by extracting the part below $C$ and above $\{\id\}$.

\begin{theorem}[Duality]\label{thm:duality-subgroup-and-partition-lattice}
Let $(\hcal_{\transf(X)}^{*}, \leq)$ be the subgroup lattice for $\transf(X)$, and $\pi$ be the abstraction generating function. Then $(\pi(\hcal_{\transf(X)}^{*}), \preceq)$ is an abstraction meet-semiuniverse for $X$. More specifically, for any $A,B \in \hcal_{\transf(X)}^{*}$, the following hold:
\begin{itemize}
\setlength\itemsep{0in}
\item [1.] partial-order reversal: if $A \leq B$, then $\pi(A) \succeq \pi(B)$;
\item [2.] strong duality: $\pi(A \vee B) = \pi(A) \wedge \pi(B)$ (Figure~\ref{fig:duality-join-and-meet}a);
\item [3.] weak duality: $\pi(A \wedge B) \succeq \pi(A) \vee \pi(B)$ (Figure~\ref{fig:duality-join-and-meet}b).
\end{itemize}
\end{theorem}
\begin{figure}[h!]
\begin{center}
\includegraphics[width=0.85\columnwidth]{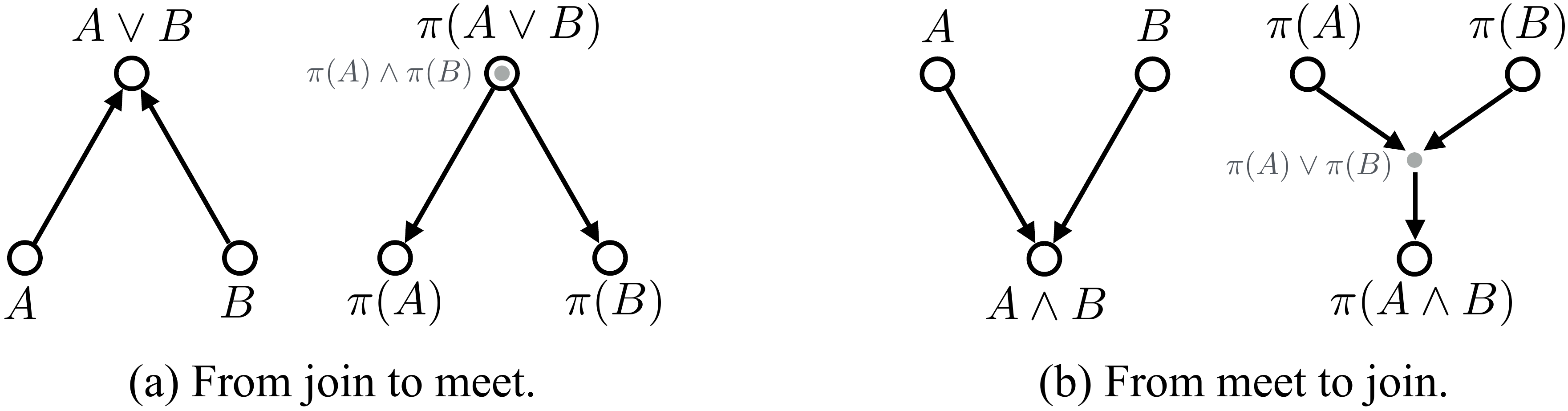}
\end{center}
\caption{Duality of join and meet between the subgroup lattice (left in each subfigure) and the partition lattice (right in each subfigure). In (a), the gray vertex denoting $\pi(A) \wedge \pi(B)$, \ie the actual meet in the partition lattice, is equal to $\pi(A \vee B)$; in (b), the gray vertex denoting $\pi(A) \vee \pi(B)$, \ie the actual join in the partition lattice, can be any vertex below $\pi(A), \pi(B)$ and above $\pi(A\wedge B)$ or even equal to these three end points.}
\label{fig:duality-join-and-meet}
\end{figure}
\begin{proof}
(Partial-order reversal)
Pick any $A,B \in \hcal_G^{*}$ and $A \leq B$.
For any $x,x' \in X$ that are in the same cell in partition $\pi(A) = X/A = \{Ax \mid x \in X\}$, $x' \in Ax = \{a(x) \mid a \in A\}$.
Since $A \leq B$, then $Ax \subseteq Bx$, which further implies that $x' \in Bx$.
So, $x$ and $x'$ are in the same cell in partition $\pi(B)$. Therefore, $\pi(A) \succeq \pi(B)$.

(Strong duality)
Pick any $A,B \in \hcal_G^{*}$. By the definition of join, $A, B \leq A \vee B$, so from what we have shown at the beginning, $\pi(A), \pi(B) \succeq \pi(A \vee B)$, \ie $\pi(A \vee B)$ is a common coarsening of $\pi(A)$ and $\pi(B)$. Since $\pi(A) \wedge \pi(B)$ is the finest common coarsening of $\pi(A)$ and $\pi(B)$, then $\pi(A \vee B) \preceq \pi(A) \wedge \pi(B)$.
Conversely, for any $x,x' \in X$ that are in the same cell in partition $\pi(A \vee B) = \pi(\langle A \cup B \rangle) = X/\langle A \cup B \rangle = \{\langle A \cup B \rangle x \mid x \in X\}$, $x$ and $x'$ must be in the same orbit under $\langle A \cup B \rangle$-action on $X$, \ie $x' \in \langle A \cup B \rangle x$ which means $x' = f_k \circ \cdots \circ f_1(x)$ for some finite integer $k$ where $f_1, \ldots, f_k \in A \cup B$ (note: the fact that $A,B$ are both subgroups ensures that $A \cup B$ is closed under inverses). This implies that $x$ and $f_1(x)$ are either in the same cell in partition $\pi(A)$ or in the same cell in partition $\pi(B)$ depending on whether $f_1 \in A$ or $f_1 \in B$, but in either event, $x$ and $f_1(x)$ must be in the same cell in any common coarsening of $\pi(A)$ and $\pi(B)$. Note that $\pi(A) \wedge \pi(B)$ is a common coarsening of $\pi(A)$ and $\pi(B)$ (regardless of the fact that it is the finest), so $x$ and $f_1(x)$ are in the same cell in partition $\pi(A) \wedge \pi(B)$. Likewise, $f_1(x)$ and $f_2 \circ f_1(x)$, $f_3 \circ f_2 \circ f_1(x)$ and $f_2 \circ f_1(x)$, $\ldots$, $f_{k-1} \circ \cdots \circ f_1(x)$ and $x'$ are all in the same cell in partition $\pi(A) \wedge \pi(B)$. Therefore, $x$ and $x'$ are in the same cell in partition $\pi(A) \wedge \pi(B)$. So, $\pi(A \vee B) \succeq \pi(A) \wedge \pi(B)$. Combining both directions yields $\pi(A \vee B) = \pi(A) \wedge \pi(B)$.

(Weak duality)
Pick any $A,B \in \hcal_G^{*}$. By the definition of meet, $A,B \geq A \wedge B$, so from what have shown at the beginning, $\pi(A), \pi(B) \preceq \pi(A \wedge B)$, \ie $\pi(A \wedge B)$ is a common refinement of $\pi(A)$ and $\pi(B)$. Since $\pi(A) \vee \pi(B)$ is the coarsest common refinement of $\pi(A)$ and $\pi(B)$, then $\pi(A \wedge B) \succeq \pi(A) \vee \pi(B)$.
We cannot obtain equality in general. For example, let $X = \integers$ and $A = \{r: \integers \to \integers \mid r(x) = k x, k \in \{-1,1\} \}$, $B = \{t: \integers \to \integers \mid t(x) = x + k, k \in \integers\}$.
It is clear that $A,B \leq \transf(X)$ and $A \wedge B = A \cap B = \{\id\}$, so $\pi(A \wedge B) = X/\{\id\} = \{\{x\} \mid x \in \integers\}$, \ie the finest partition of $\integers$.
However, $\pi(A) = \{\{x,-x\} \mid x \in \integers\}$ and $\pi(B) = \{\integers\}$, \ie the coarsest partition of $\integers$, so $\pi(A) \vee \pi(B) = \pi(A) = \{\{x,-x\} \mid x \in \integers\}$. In this example, we see that $\pi(A \wedge B) \succeq \pi(A) \vee \pi(B)$ but $\pi(A \wedge B) \neq \pi(A) \vee \pi(B)$.
\end{proof}
\begin{remark}[Practical implication]\label{remark:duality-subgroup-and-partition-lattice}
The strong duality in Theorem~\ref{thm:duality-subgroup-and-partition-lattice} suggests a quick way of computing abstractions. If one has already computed abstractions $\pi(A)$ and $\pi(B)$, then instead of computing $\pi(A\vee B)$ from $A \vee B$, one can compute the meet $\pi(A) \wedge \pi(B)$, which is generally a less expensive operation than computing $A \vee B$ and identifying all orbits in $\pi(A\vee B)$.
\end{remark}

Theorem~\ref{thm:duality-subgroup-and-partition-lattice} further allows us to build an abstraction semiuniverse with a partial hierarchy directly inherited from the hierarchy of the subgroup lattice. Nevertheless, there are cases where $\pi(A) \preceq \pi(B)$ with incomparable $A$ and $B$ since the abstraction generating function $\pi$ is not injective (Theorem~\ref{thm:abstraction-generating-function-not-injective}).
If desired, one needs additional steps to complete the hierarchy or even to complete the abstraction semiuniverse into an abstraction universe.

\subsection{More on Duality: from Conjugation to Group Action}
\label{sec:more-duality-subgroup-to-abstraction}

Partitions of a set $X$ generated from two conjugate subgroups of $\transf(X)$ can be related by a group action.
We present this relation as another duality between subgroups and abstractions, which can also simplify the computation of abstractions.

\paragraph{Preliminaries (Appendix~\ref{app:more_duality_subgroup_to_abstraction}):}
conjugate, conjugacy class.

\begin{theorem}\label{thm:group-action-layer-up}
Let $G$ be a group, $X$ be a set, and $\cdot: G \times X \to X$ be a $G$-action on $X$. Then
\begin{itemize}
\setlength\itemsep{0in}
\item [1.] for any $g \in G, Y \in 2^X$, $g \cdot Y := \{g \cdot y \mid y \in Y\} \in 2^X$, and the corresponding function $\cdot: G \times 2^X \to 2^X$ defined by $g\cdot Y$ is a $G$-action on $2^X$;
\item [2.] for any $g \in G, \pcal \in \mathfrak{P}_X^{*}$, $g \cdot \pcal := \{g \cdot P \mid P \in \pcal \} \in \mathfrak{P}_X^{*}$, and the corresponding function $\cdot: G \times \mathfrak{P}_X^{*} \to \mathfrak{P}_X^{*}$ defined by $g \cdot \pcal$ is a $G$-action on $\mathfrak{P}_X^{*}$.
\end{itemize}
\end{theorem}
\begin{proof}
See Appendix~\ref{app:group-action-layer-up}.
\end{proof}

\begin{theorem}[Duality]\label{thm:duality-conjugacy-in-subgroup-to-scaling-in-partition}
Let $X$ be a set, $\transf(X)$ be the transformation group of $X$, and $\pi$ be the abstraction generating function. Then for any $H \leq \transf(X)$ and $g \in \transf(X)$,
\begin{align*}
\pi(g \circ H \circ g^{-1}) = g \cdot \pi(H),
\end{align*}
where $\cdot$ refers to the group action defined in Statement 2 in Theorem~\ref{thm:group-action-layer-up}.
\end{theorem}
\begin{proof}
For any $Y \in \pi(g \circ H \circ g^{-1})$, $Y$ is an orbit in $X$ under $g \circ H \circ g^{-1}$, then
$
Y = (g \circ H \circ g^{-1})x
= \{(g\circ h \circ g^{-1})\cdot x \mid h \in H\}
= \{g\circ h \circ g^{-1}(x) \mid h \in H\}
= \{(g\circ h) (g^{-1}(x)) \mid h \in H\}
= \{(gh)\cdot g^{-1}(x) \mid h \in H\}
= \{g \cdot (h\cdot g^{-1}(x)) \mid h \in H\}
= \{g \cdot y \mid y \in Hg^{-1}(x)\}
= g \cdot Hg^{-1}(x)
$
for some $x \in X$. Note that in the above derivation, $g^{-1}(x) \in X$ since $g \in \transf(X)$. So, $Hg^{-1}(x)$ is the orbit of $g^{-1}(x)$ under $H$, \ie $Hg^{-1}(x) \in \pi(H)$. This implies that $Y \in g \cdot \pi(H)$. Therefore, $\pi(g \circ H \circ g^{-1}) \subseteq g \cdot \pi(H)$.

Conversely, for any $Y \in g \cdot \pi(H)$, $Y = g\cdot P$ for some $P \in \pi(H)$. Note that $P$ is an orbit in $X$ under $H$, \ie $P = Hx = \{h\cdot x \mid h \in H\}$ for some $x \in X$, then
$
Y = g \cdot P
= \{g\cdot y \mid y \in P\}
= \{g\cdot (h\cdot x) \mid h \in H\}
= \{(gh)\cdot x \mid h \in H\}
= \{g\circ h(x) \mid h \in H\}
= \{g\circ h \circ g^{-1}\circ g(x) \mid h \in H\}
= \{(g\circ h \circ g^{-1})(g(x)) \mid h \in H\}
= \{(g\circ h\circ g^{-1}) \cdot g(x) \mid h \in H\}
= (g \circ H \circ g^{-1})g(x)
$
for some $x \in X$. Note that in the above derivation, $g(x) \in X$ since $g \in \transf(X)$. Therefore, $(g \circ H \circ g^{-1})g(x)$ is the orbit of $g(x)$ under $g \circ H \circ g^{-1}$, \ie $(g \circ H \circ g^{-1})g(x) \in \pi(g \circ H \circ g^{-1})$. This implies that $Y \in \pi(g \circ H \circ g^{-1})$. So, $g \cdot \pi(H) \subseteq \pi(g \circ H \circ g^{-1})$.
\end{proof}
\begin{remark}[Practical implication]
Theorem~\ref{thm:duality-conjugacy-in-subgroup-to-scaling-in-partition} relates conjugation in the subgroup lattice $\hcal_{\transf(X)}^{*}$ to group action on the partition lattice $\mathfrak{P}_X^{*}$.
In other words, the group action on the partition lattice is dual to the conjugation in the subgroup lattice.
This duality suggests a quick way of computing abstractions.
If one has already computed abstraction $\pi(H)$, then instead of computing $\pi(g\circ H \circ g^{-1})$ from $g\circ H \circ g^{-1}$, one can compute $g\cdot \pi(H)$, which is generally a less expensive operation than computing $g\circ H \circ g^{-1}$ and identifying all orbits in $\pi(g\circ H \circ g^{-1})$.
\end{remark}

\subsection{Partial Subgroup Lattice}
\label{sec:partial-subgroup-lattice}

Theoretically, through the abstraction generating function $\pi$ and necessary hierarchy completions, we can construct the complete abstraction universe $\mathfrak{P}_X^{*}$ from the complete subgroup lattice $\hcal_{\transf(X)}^{*}$.
This is because the subgroup lattice is a larger space that ``embeds'' the partition lattice (more precisely, Theorem~\ref{thm:abstraction-generating-function-not-injective} and \ref{thm:abstraction-generating-function-surjective}).
However, as we mentioned earlier, it is not practical to even store $\mathfrak{P}_X^{*}$ for small $X$, and not all arbitrary partitions of $X$ are equally useful.
Instead of considering all subgroups of $\transf(X)$, we draw our attention to certain parts of the complete subgroup lattice $\hcal_{\transf(X)}^{*}$.
We introduce two general principles in extracting partial subgroup lattices: the top-down approach and the bottom-up approach.

\paragraph{The Top-Down Approach.}
We consider the subgroup sublattice $(\hcal_G^{*}, \leq)$ for some subgroup $G \leq \transf(X)$. If $X$ is finite, this is the part below $G$ and above $\{\id\}$ in the directed acyclic graph for the complete subgroup lattice $(\hcal_{\transf(X)}^{*}, \leq)$. As the name suggests, the top-down approach first specifies a ``top'' in $\hcal_{\transf(X)}^{*}$ (\ie a subgroup $G \leq \transf(X)$), and then extracts everything below the ``top'' (\ie the subgroup lattice $\hcal_{G}^{*}$).
The computer algebra system GAP \citep{GAP4} provides efficient algorithmic methods to construct the subgroup lattice for a given group, and even maintains several data libraries for special groups and their subgroup lattices. In general, enumerating all subgroups of a group can be computationally intensive, and therefore, is applied primarily to small groups. When computationally prohibitive, a general trick is to enumerate subgroups up to conjugacy (which is also supported by the GAP system).
Computing abstractions within the conjugacy class of any subgroup is then easy by the duality in Theorem~\ref{thm:duality-conjugacy-in-subgroup-to-scaling-in-partition}, once the abstraction generated by a representative is computed.
More details on picking a special subgroup (as the ``top'') of $\transf(X)$ are discussed in Section~\ref{sec:the-top-down-approach-sp-subgroups}.

\paragraph{The Bottom-Up Approach.}
We first pick some finite subset $S \subseteq \transf(X)$, and then generate a partial subgroup lattice for $\langle S \rangle$ by computing $\langle S' \rangle$ for every $S' \subseteq S$, starting from smaller subgroups.
As the name suggests, the bottom-up approach first constructs the trivial subgroup $\langle \emptyset \rangle = \{\id\}$, \ie the bottom vertex in the direct acyclic graph for $\hcal_{\transf(X)}^{*}$ if $X$ is finite, and then cyclic subgroups $\langle s \rangle$ for every $s \in S$.
We continue to construct larger subgroups from smaller ones by taking the join, which corresponds to gradually moving upwards in the graph for $\hcal_{\transf(X)}^{*}$ when $X$ is finite.
In general, this approach will produce at most $2^{|S|}$ subgroups for a given subset $S \subseteq \transf(X)$, and will not produce the complete subgroup sublattice $\hcal_{\langle S \rangle}^{*}$ unless $S = \langle S \rangle$.
Computing abstractions using this bottom-up approach is easy by the strong duality in Theorem~\ref{thm:duality-subgroup-and-partition-lattice}, once the abstractions generated by all cyclic subgroups are computed.
More details on this abstraction generating process and picking a generating set (as the ``bottom'') are discussed in Section~\ref{sec:the-bottom-up-approach-gen-set}.


\section{The Top-Down Approach: Special Subgroups}
\label{sec:the-top-down-approach-sp-subgroups}

We follow a top-down approach to discuss subgroup enumeration problems.
The plan is to start with the transformation group of $X = \reals^n$, and then to consider special subgroups of $\transf(\reals^n)$ and special subspaces of $\reals^n$.
To do this systematically, we derive a principle that allows us to hierarchically break the enumeration problem into smaller and smaller enumeration subproblems.
This hierarchical breakdown can guide us in restricting both the type of subgroups and the type of subspaces, so that the resulting abstraction (semi)universe fits our desiderata, and more importantly can be computed in practice.
Figure~\ref{fig:subgroups-and-spaces} presents an outline consisting of special subgroups and subspaces considered in this section as well as their hierarchies.

\begin{figure}[t]
\begin{center}
\includegraphics[width=0.92\columnwidth]{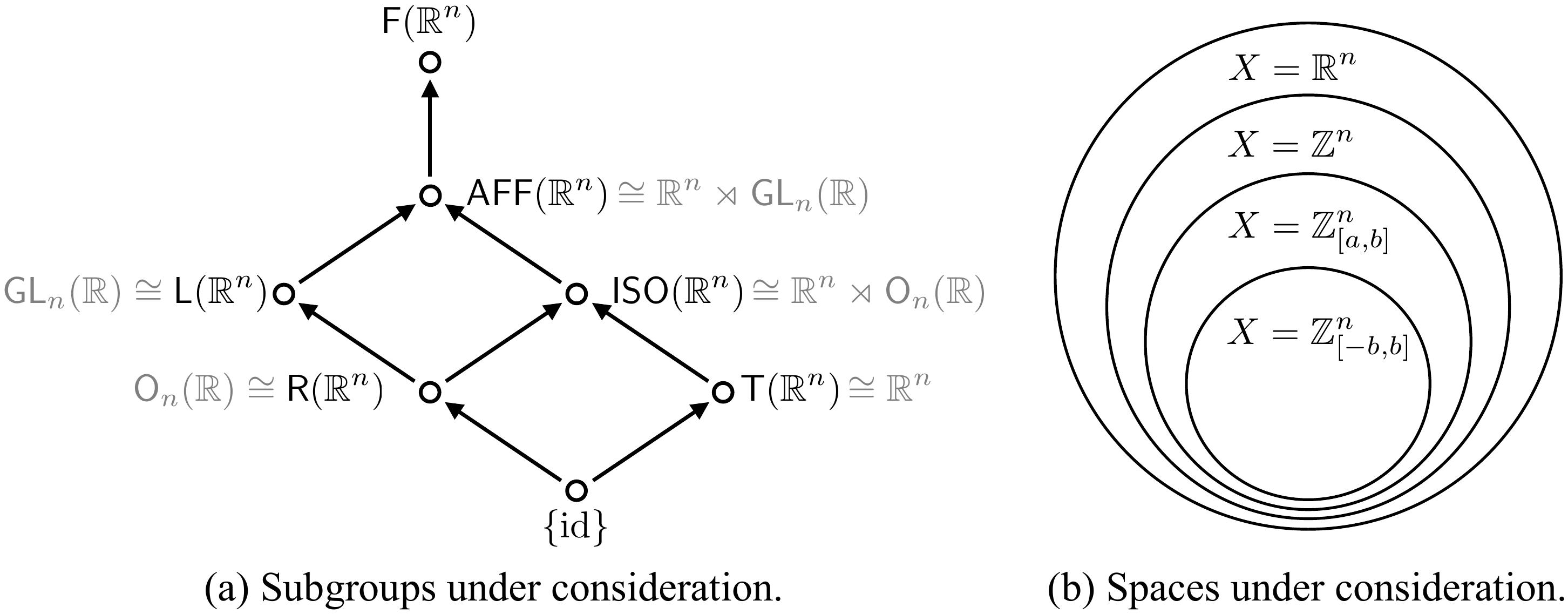}
\end{center}
\caption{Special subgroups and spaces as well as their hierarchies. (a) presents a backbone of the complete subgroup lattice $\hcal_{\transf(\reals^n)}^{*}$, including important subgroups and their breakdowns. One can check the above directed acyclic graph indeed represents a sublattice: it is closed under both join and meet. (b) presents important subspaces of $\reals^n$, where restrictions are gradually added to eventually lead to practical abstraction-construction algorithms.}
\label{fig:subgroups-and-spaces}
\end{figure}

\noindent\underline{Note:} we do not claim the originality of the content in this section.
Indeed, many parts have been studied in various contexts.
Our work is to extend existing results from specific context to a general setting.
This generalization coherently puts different pieces of context-specific knowledge under one umbrella, forming the guiding principle of the top-down approach.

\paragraph{Preliminaries (Appendix~\ref{app:the-top-down-approach-sp-subgroups}):}
group homomorphism, isomorphism ($\cong$); normalizer of a set in a group (${\rm N}_G(S) := \{g \in G \mid gSg^{-1} = S\}$), normal subgroup ($\trianglelefteq$); group decomposition, inner semi-direct product, outer semi-direct product ($\rtimes$).

\subsection{The Affine Transformation Group $\affine(\reals^n)$}
\label{sec:the-affine-transformation-group}

An \emph{affine transformation} of $\reals^n$ is a function $f_{A,u}: \reals^n \to \reals^n$ of the form
\begin{align*}
f_{A,u}(x) = Ax + u \quad \mbox{ for any } x \in \reals^n,
\end{align*}
where $A \in \genlin_n(\reals)$ is an $n\times n$ real invertible matrix and $u \in \reals^n$ is an $n$-dimensional real vector. We use $\affine(\reals^n)$ to denote the set of all affine transformations of $\reals^n$.
There are two special cases:
\begin{itemize}
\setlength\itemsep{0in}
\item [1.] A \emph{translation} of $\reals^n$ is a function $t_u: \reals^n \to \reals^n$ of the form $x \mapsto x + u$ where $u \in \reals^n$;
we use $\tra(\reals^n)$ to denote the set of all translations of $\reals^n$.
\item [2.] A \emph{linear transformation} of $\reals^n$ is a function $r_A: \reals^n \to \reals^n$ of the form $x \mapsto Ax$ where $A \in \genlin_n(\reals)$; we use $\lin(\reals^n)$ to denote the set of all linear transformations of $\reals^n$.
\end{itemize}
It is easy to check that $\tra(\reals^n), \lin(\reals^n) \leq \affine(\reals^n) \leq \transf(\reals^n)$; further, $(\tra(\reals), \circ)$ and $(\lin(\reals^n), \circ)$ are isomorphic to $(\reals^n, +)$ and $(\genlin_n(\reals), \cdot)$, respectively.
It is known that
\begin{align*}
\affine(\reals^n) = \tra(\reals^n)\circ \lin(\reals^n) \cong \tra(\reals^n) \rtimes \lin(\reals^n) \cong \reals^n \rtimes \genlin_n(\reals).
\end{align*}
So every affine transformation can be uniquely identified with a pair $(u,A) \in \reals^n \rtimes \genlin_n(\reals)$. In particular, the identity transformation is identified with $(\bm{0}, I)$, the translation group $\tra(\reals^n)$ is identified with $\{(u,I) \mid u \in \reals^n\}$, and the linear transformation group $\lin(\reals^n)$ is identified with $\{(\bm{0}, A) \mid A \in \genlin_n(\reals)\}$.
Under this identification, compositions and inverses of affine transformations become
\begin{align}\label{eqn:composition-and-inverse-of-affine-pair}
(u,A)(u',A') = (u+Au', AA') \quad \mbox{ and } \quad (u, A)^{-1} = (-A^{-1}u, A^{-1}).
\end{align}
The above identification further allows us to introduce two functions $\ell: \affine(\reals^n) \to \genlin_n(\reals)$ and $\tau: \affine(\reals^n) \to \reals^n$ to extract the linear and translation part of an affine transformation, respectively, where
\begin{align*}
\ell(f_{A,u}) = A, \  \tau(f_{A,u}) = u \quad \mbox{ for any } f_{A,u} \in \affine(\reals^n).
\end{align*}

Now we can start our journey towards a complete identification of every subgroup $H$ of $\affine(\reals^n)$.
We introduce the first foundational quantity $T := \tra(\reals^n) \cap H$, which is the set of pure translations in $H$, called the \emph{translation subgroup} of $H$.
It is easy to check that $T \trianglelefteq H$ since translations are normal in affine transformations.
Therefore, the quotient group $H/T = \{T \circ h \mid h \in H\}$ is well-defined.
The elements in $H/T$ are called \emph{cosets}. The following theorems reveal more structures of $H/T$, the second foundational quantity.

\begin{lemma}\label{lemma:linear-part-is-homomorphism}
$\ell: \affine(\reals^n) \to \genlin_n(\reals)$ is a homomorphism.
\end{lemma}
\begin{proof}
For any $f_{A,u}, f_{A',u'} \in \affine(\reals^n)$, we have $\ell(f_{A,u} \circ  f_{A',u'}) = \ell(f_{AA', Au'+u}) = AA' = \ell(f_{A,u})\ell(f_{A',u'})$, which implies that $\ell$ is a homomorphism.
\end{proof}

\begin{theorem}\label{thm:in-the-same-coset-means-same-linear-part}
Let $H \leq \affine(\reals^n)$, $T = \tra(\reals^n) \cap H$. Then $h,h' \in H$ are in the same coset in $H/T$ if and only if they have the same linear part, \ie $\ell(h) = \ell(h')$.
\end{theorem}
\begin{proof}
See Appendix~\ref{app:in-the-same-coset-means-same-linear-part}.
\end{proof}

\begin{theorem}\label{thm:in-the-same-coset-translation-part-diff}
Let $H \leq \affine(\reals^n)$, $T = \tra(\reals^n) \cap H$.
If $h,h' \in H$ are in the same coset in $H/T$, then $\tau(h')-\tau(h) \in \tau(T) := \{u \mid t_u \in T\}$.
\end{theorem}
\begin{proof}
See Appendix~\ref{app:in-the-same-coset-translation-part-diff}.
\end{proof}
\begin{remark}
Theorems~\ref{thm:in-the-same-coset-means-same-linear-part} and \ref{thm:in-the-same-coset-translation-part-diff} present two characterizations of elements in the same coset in $H/T$, respectively. The former, through the linear part, is an if-and-only-if characterization; while the latter, through the translation part, is a necessary but not sufficient characterization.
\end{remark}

\begin{theorem}\label{thm:characterize-linear-part}
Let $H \leq \affine(\reals^n)$, $T = \tra(\reals^n) \cap H$. Then $H/T \cong \ell(H)$.
\end{theorem}
\begin{proof}
It is clear that $\ell(H) \leq \genlin_n(\reals)$, since $H \leq \affine(\reals^n)$ and $\ell$ is a homomorphism (Lemma~\ref{lemma:linear-part-is-homomorphism}) which preserves subgroups.
Let $\bar{\ell}: H/T \to \ell(H)$ be the function of the form $\bar{\ell}(T \circ h) = \ell(h)$, we claim that $\bar{\ell}$ is an isomorphism. To see this, for any $T\circ h, T \circ h' \in H/T$,
\begin{align*}
\bar{\ell}((T\circ h)(T \circ h')) = \bar{\ell}(T \circ (h\circ h')) = \ell(h\circ h') = \ell(h)\ell(h') = \bar{\ell}(T \circ h) \bar{\ell}(T \circ h'),
\end{align*}
which implies $\bar{\ell}$ is a homomorphism. Further, for any $T\circ h, T \circ h' \in H/T$, if $\bar{\ell}(T\circ h) = \bar{\ell}(T\circ h')$, then $\ell(h) = \ell(h')$. By Theorem~\ref{thm:in-the-same-coset-means-same-linear-part}, this implies that $T \circ h = T \circ h'$, so $\bar{\ell}$ is injective.
Lastly, for any $A \in \ell(H)$, there exists an $h \in H$ such that $\ell(h) = A$. For this particular $h$, $T \circ h \in H/T$, and $\bar{\ell}(T\circ h) = \ell(h)=A$. This implies that $\bar{\ell}$ is surjective.
\end{proof}
\begin{remark}
Theorem~\ref{thm:characterize-linear-part} can be proved directly from the first isomorphism theorem, by recognizing $\ell|_H$ is a homomorphism whose kernel and image are $T$ and $\ell(H)$, respectively. However, the above proof explicitly gives the isomorphism $\bar{\ell}$ which is useful in the sequel.
\end{remark}

\begin{theorem}[Compatibility]\label{thm:compatibility-via-group-action}
Let $H \leq \affine(\reals^n)$, $T = \tra(\reals^n) \cap H$.
For any $A \in \ell(H)$ and $v \in \tau(T)$, we have $Av \in \tau(T)$.
Further, if we define a function $\cdot: \ell(H) \times \tau(T) \to \tau(T)$ of the form $(A,v) \mapsto Av$, then $\cdot$ is a group action of $\ell(H)$ on $\tau(T)$.
\end{theorem}
\begin{proof}
See Appendix~\ref{app:compatibility-via-group-action}.
\end{proof}

So far, we have seen that for any subgroup $H \leq \affine(\reals^n)$, its subset of pure translations $T := \tra(\reals^n) \cap H$ is a normal subgroup of $H$; $T$ is also a normal subgroup of $\tra(\reals^n)$, since $\tra(\reals^n)$ is a commutative group.
As a result, both quotient groups $H/T$ and $\tra(\reals^n)/T$ are well-defined.
We next introduce a function, called a \emph{vector system}, which connects the two quotient groups.
It turns out that vector systems comprise the last piece of information that leads to a complete identification of every subgroup of $\affine(\reals^n)$.
Note that $H/T \cong \ell(H)$ (Theorem~\ref{thm:characterize-linear-part}) and $\tra(\reals^n)/T \cong \reals^n/\tau(T)$; thus for conceptual ease (think in terms of matrices and vectors), we introduce vector systems connecting $\ell(H)$ and $\reals^n/\tau(T)$ instead.

\begin{definition}[Vector system]\label{def:vector-system}
For any $L\leq \genlin_n(\reals)$ and $V \leq \reals^n$, an $(L,V)$-vector system is a function $\xi: L \to \reals^n/V$, which in addition satisfies the following two conditions:
\begin{itemize}
\setlength\itemsep{0in}
\item [1.] compatibility condition: for any $A \in L$, $AV = \{Av \mid v\in V\} = V$;
\item [2.] cocycle condition: for any $A,A' \in L$, $\xi(AA') = \xi(A) + A\xi(A') $.
\end{itemize}
\underline{Note:} elements in $\reals^n/V$ are cosets of the form $V + u$ for $u \in \reals^n$. It is easy to check: for any two cosets in $\reals^n/V$, the sum
\begin{align*}
(V+u)+(V+u') = \{v+u+v'+u' \mid v,v' \in V\} = V + (u + u');
\end{align*}
for any $A \in L$ and any coset in $\reals^n/V$, the product
\begin{align*}
A(V+u) = \{A(v+u) \mid v \in V\} = V + Au.
\end{align*}
So, the sum and product in the cocycle condition are defined in the above sense.
\end{definition}
We use $\Xi_{L,V}$ to denote the family of all $(L,V)$-vector systems.
One can check that $\Xi_{L,V} \neq \emptyset$ if and only if $L,V$ are compatible (consider the trivial vector system $\xi_{L,V}^{0}$ given by $\xi_{L,V}^{0}(A) = V$ for all $A \in L$).
We use $\Xi^{*} := \{\Xi_{L,V} \mid L \leq \genlin_n(\reals), V \leq \reals^n \mbox{ compatible}\}$ to denote the universe of all vector systems.
\begin{remark}
The universe of all vector systems $\Xi^*$ can be parameterized by the set of compatible pairs $(L,V) \in \hcal_{\genlin_n(\reals)}^{*} \times \hcal_{\reals^n}^{*}$.
The reason is straightforward: $L$ and $V$ respectively define the domain and codomain of a function, and two functions are different if either their domains or their codomains are different.
\end{remark}

\begin{lemma}\label{lemma:vector-system-properties}
Let $L\leq \genlin_n(\reals)$, $V \leq \reals^n$, and $\xi \in \Xi_{L,V}$, then
\begin{itemize}
\setlength\itemsep{0in}
\item [1.] for the identity matrix $I \in L$, $\xi(I) = V$;
\item [2.] for any $A \in L$, $\xi(A^{-1}) = -A^{-1}\xi(A)$.
\end{itemize}
\end{lemma}
\begin{proof}
See Appendix~\ref{app:vector-system-properties}.
\end{proof}

\begin{theorem}[Affine subgroup identification]\label{thm:affine-group-identification}
Let
\begin{align*}
\Sigma := \{(L,V,\xi) \mid L \leq \genlin_n(\reals), V \leq \reals^n, \xi \in \Xi_{L,V}\},
\end{align*}
then there is a bijection between $\hcal_{\affine(\reals^n)}^{*}$ and $\Sigma$.
\end{theorem}
\begin{proof}
(Outline)
Let $\Psi: \hcal_{\affine(\reals^n)}^{*} \to \Sigma$ be the function defined by
\begin{align*}
\Psi(H) := (\ell(H),~ \tau(T),~ \xi_H) \quad \mbox{ for any } H \in \hcal_{\affine(\reals^n)}^{*},
\end{align*}
where $T := \tra(\reals^n) \cap H$, and $\xi_H: \ell(H) \to \reals^n/\tau(T)$ is given by $\xi_H(A) = \tau(\bar{\ell}^{-1}(A))$ with $\bar{\ell}: H/T \to \ell(H)$ being the isomorphism defined in the proof of Theorem~\ref{thm:characterize-linear-part}.
The plan is to first show that $\Psi$ is well-defined, and then to show that it is bijective; in particular, we will show that the inverse function
\begin{align*}
\Psi^{-1}((L,V,\xi)) = \{f_{A,u} \in \affine(\reals^n) \mid A \in L, u \in \xi(A)\} \quad \mbox{ for any } (L,V,\xi) \in \Sigma.
\end{align*}
The entire proof is divided into four parts.
We relegate the full proof to Appendix~\ref{app:affine-group-identification}.
\end{proof}
\begin{remark}\label{remark:enumerate-affine-subgroups}
The bijection $\Psi$ from $\hcal_{\affine(\reals^n)}^{*}$ to $\Sigma$ allows us to use the latter to parameterize the former.
Further, through the inverse function $\Psi^{-1}$, we can enumerate affine subgroups by enumerating triplets $(L,V,\xi) \in \Sigma$, or more specifically, by enumerating matrix subgroups of $\genlin_n(\reals)$, vector subgroups of $\reals^n$, and then vector systems for every compatible pair of a matrix subgroup and a vector subgroup.
Note that enumeration for each element in the triplet is still not practical if no restriction is imposed.
Nevertheless, we have broken the original subgroup enumeration problem into three smaller enumeration problems.
More importantly, we are now more directed in imposing restrictions on both subgroups and spaces, under which the three smaller enumerations become practical.
We will discuss these restrictions (\eg being isometric, finite, discrete, compact) in more detail in the sequel.
\end{remark}

\subsection{The Isometry Group $\iso(\reals^n)$}
\label{sec:isometries-of-rn}

One way to restrict $\Sigma := \{(L,V,\xi) \mid L \leq \genlin_n(\reals), V \leq \reals^n, \xi \in \Xi_{L,V}\}$ is to consider a special subgroup of $\genlin_n(\reals)$.
Instead of all subgroups of $\genlin_n(\reals)$, we consider only subgroups consisting of \emph{orthogonal matrices}.
This restriction gives rise to the subgroup lattice $\hcal_{\iso(\reals^n)}^{*}$ where $\iso(\reals^n)$ denotes the group of isometries of $\reals^n$.
In this subsection, we first give an overview of $\iso(\reals^n)$, and then cast $\hcal_{\iso(\reals^n)}^{*}$ in the big picture of $\hcal_{\affine(\reals^n)}^{*}$ and $\Sigma$.

An \emph{isometry} of $\reals^n$, with respect to the Euclidean distance $\dist$, is a transformation $h:\reals^n \to \reals^n$ which preserves distances: $\dist(h(x),h(x')) = \dist(x,x')$, for all $x,x' \in \reals^n$.
We use $\iso(\reals^n)$ to denote the set of all isometries of $\reals^n$, which is a subgroup of the transformation group $\transf(\reals^n)$.
So, we call $\iso(\reals^n)$ the \emph{isometry group} of $\reals^n$.

A \emph{(generalized) rotation} of $\reals^n$ is a linear transformation $r_A: \reals^n \to \reals^n$ given by $x \mapsto Ax$, for some orthogonal matrix $A \in \orthmat_n(\reals) := \{A\in \reals^{n\times n}\mid A^\top = A^{-1}\} \leq \genlin_n(\reals)$.
We use $\rot(\reals^n)$ to denote the set of all rotations of $\reals^n$, which is a subgroup of the linear transformation group $\lin(\reals^n)$.
So, we call $\rot(\reals^n)$ the \emph{rotation group} of $\reals^n$.

There are two key characterizations of $\iso(\reals^n)$.
The first one regards its components:
\begin{align*}
\iso(\reals^n) = \langle \tra(\reals^n) \cup \rot(\reals^n) \rangle \quad \mbox{ where } \tra(\reals^n) \cap \rot(\reals^n) = \{\id\}.
\end{align*}
This characterization says that $\iso(\reals^n)$ comprises exclusively translations, rotations, and their finite compositions.
Note that we can rewrite the above characterization as $\tra(\reals^n) \vee \rot(\reals^n) = \iso(\reals^n)$ and $\tra(\reals^n) \wedge \rot(\reals^n) = \{\id\}$.
This determines the positions of the four subgroups $\{\id\}$, $\tra(\reals^n)$, $\rot(\reals^n)$, and $\iso(\reals^n)$ in the subgroup lattice $(\hcal_{\transf(\reals^n)}^{*}, \leq)$, which forms a diamond shape in the direct acyclic graph in Figure~\ref{fig:subgroups-and-spaces}a.
The second characterization of $\iso(\reals^n)$ regards a unique representation for every isometry of $\reals^n$, which is done by a group decomposition of $\iso(\reals^n)$ as semi-direct products:
\begin{align*}
\iso(\reals^n) = \tra(\reals^n)\circ \rot(\reals^n) \cong \reals^n \rtimes \orthmat_n(\reals).
\end{align*}
This characterization says that every isometry of $\reals^n$ can be uniquely represented as an affine transformation $f_{A,u} \in \affine(\reals^n)$ where $A \in \orthmat_n(\reals)$ and $u \in \reals^n$. This further implies that $\iso(\reals^n)$ is a special subgroup of $\affine(\reals^n)$.

Let $\Psi: \hcal_{\affine(\reals^n)}^{*} \to \Sigma$ be the bijection defined in the proof of Theorem~\ref{thm:affine-group-identification}, and let
\begin{align*}
\Sigma' := \{(L,V,\xi) \mid L \leq \orthmat_n(\reals), V \leq \reals^n, \xi \in \Xi_{L,V}\} \subseteq \Sigma.
\end{align*}
One can check: $\Psi^{-1}(\Sigma') = \hcal_{\iso(\reals^n)}^{*}$.
This means $\Psi|_{\hcal_{\iso(\reals^n)}^{*}}: \hcal_{\iso(\reals^n)}^{*} \to \Sigma'$ is well-defined and bijective.
Therefore, the subgroups of $\iso(\reals^n)$ can be enumerated by the triplets in $\Sigma'$ in a similar manner as in Remark~\ref{remark:enumerate-affine-subgroups}.
The only difference is that we now enumerate subgroups of $\orthmat_n(\reals)$ instead of the entire $\genlin_n(\reals)$.

Note that restricting to subgroups of $\orthmat_n(\reals)$ does not really make the enumeration problem practical. However, there are many ways of imposing additional restrictions on $\iso(\reals^n)$ to eventually achieve practical enumerations.
We want to point out that there is no universal way of constraining the infinite enumeration problem into a practical one: the design of restrictions is most effective if it is consistent with the underlying topic domain.
So, for instance, one can start with his/her intuition to try out some restrictions whose effectivenesses can be verified via a subsequent learning process (cf.\ Section~\ref{sec:discussion-to-info-lattice-and-learning}).
In the next subsection, we give two examples to illustrate some of the existing design choices that have been made in two different domains.

\subsection{Special subgroups of $\iso(\reals^n)$ used in Chemistry and Music}
\label{sec:special-subgroups-chemistry-and-music}

From two examples, we show how additional restrictions can be imposed to yield a finite collection of subgroups of $\iso(\reals^n)$, capturing different parts of the infinite subgroup lattice $\hcal_{\iso(\reals^n)}^{*}$.
The two examples are from two different topic domains: one is from chemistry (more precisely, crystallography), the other is from music.
The ways of adding restrictions in these two examples are quite different: one introduces conjugacy relations to obtain a finite collection of subgroup types; the other restricts the space to be discrete or even finite.

\subsubsection{Crystallographic Space Groups}
\label{sec:space-groups}

In crystallography, symmetry is used to characterize crystals, to identify repeating parts of molecules, and to simplify both data collection and subsequent calculations.
Further, the symmetry of physical properties of a crystal such as thermal conductivity and optical activity has a strong connection with the symmetry of the crystal.
So, a thorough knowledge of symmetry is crucial to a crystallographer.
A complete set of symmetry classes is captured by a collection of 230 unique 3-dimensional space groups.
However, space groups represent a special type of subgroups of $\iso(\reals^n)$ which can be defined in general for any dimension.

We give a short review of known results from crystallography, and then identify space groups in the parametrization set $\Sigma$ that we derived earlier.
A \emph{crystallographic space group} or \emph{space group} $\Gamma$ is a discrete (with respect to the subset topology) and cocompact (\ie the abstraction space $\pi(\Gamma) := \reals^n/\Gamma$ is compact with respect to the quotient topology) subgroup of $\iso(\reals^n)$.
So, if the underlying topic domain indeed considers only compact abstractions, space groups are good candidates. A major reason is that for a given dimension, there exist only finitely many space groups (up to isomorphism or affine conjugacy) by Bieberbach's second and third theorems \citep{Bieberbach1911,Charlap2012}.

Bieberbach's first theorem \citep{Bieberbach1911,Charlap2012} gives an equivalent characterization of space groups: a subgroup $\Gamma$ of $\iso(\reals^n)$ is a space group if $T: = \tra(\reals^n) \cap \Gamma$ is isomorphic to $\integers^n$ and $\tau(T)$ spans $\reals^n$. In particular, for a space group $\Gamma$ in standard form, we have $\ell(\Gamma) \leq \orthmat_n(\integers)$, $\tau(T) = \integers^n$ \citep{EickS2006}.
Therefore, we can use
\begin{align*}
\Sigma''_{cryst} := \{(L,V,\xi) \mid L \leq \orthmat_n(\integers), V = \integers^n, \xi \in \Xi_{L,V}\} \subseteq \Sigma' \subseteq \Sigma
\end{align*}
to parameterize the set of all space groups in standard form.
We will soon (in Section~\ref{sec:isometries-of-zn}) see that $|\orthmat_n(\integers)| = n!2^n$ which is finite.
For every $L \leq \orthmat_n(\integers)$, the enumeration of vector systems $\xi \in \Xi_{L,\integers^n}$ is also made feasible in \citet{Zassenhaus1948} by identifying orbits in $H^1(L,\reals^n/\integers^n)$ under the group action of ${\rm N}_{\genlin_n(\integers)}(L)$ on $H^1(L,\reals^n/\integers^n)$, where $H^1(L,\reals^n/\integers^n)$ is the first cohomology group of $L$ with values in $\reals^n/\integers^n$ and ${\rm N}_{\genlin_n(\integers)}(L)$ is the integral normalizer of $L$.
We refer interested readers to the original Zassenhaus algorithm \citep{Zassenhaus1948} and the GAP package CrystCat \citep{FelschG2000} for more details on the algorithmic implementation of space groups.

\subsubsection{Isometries of $\integers^n$ in Music}
\label{sec:isometries-of-zn}

Another example of obtaining a finite collection of subgroups of $\iso(\reals^n)$ comes from computational music theory.
This is an extension to our earlier work on building an automatic music theorist \citep{YuVGK2016, YuV2017, YuLV2017}.
In this example, we impose restrictions on the space, focusing on discrete subsets of $\reals^n$ that represent music pitches from equal temperament.
Restrictions on the space further result in restrictions on the subgroups under consideration, namely only those subgroups that \emph{stabilize} the restricted subsets of $\reals^n$.
We start our discussion on isometries of $\integers^n$, while further restrictions for a finite discrete subspace such as $\integers_{[a,b]}^n$ or $\integers_{[-b,b]}^n$ (Figure~\ref{fig:subgroups-and-spaces}b) will be presented in Section~\ref{sec:restrict-to-finite-subspaces}.
%
%
We first introduce a few definitions regarding the space $\integers^n$ in parallel with their counterparts regarding $\reals^n$, and then establish their equivalences under \emph{restricted setwise stabilizers}.


\begin{definition}
An isometry of $\integers^n$, with respect to the Euclidean distance $\dist$ (or more precisely $\dist|_{\integers^n}$) is a function $h': \integers^n \to \integers^n$ which preserves distances: $\dist(h'(x),h'(x')) = \dist(x,x')$, for all $x,x' \in \integers^n$. We use $\iso(\integers^n)$ to denote the set of all isometries of $\integers^n$.
\end{definition}

\begin{definition}
A translation of $\integers^n$ is a function $t'_u: \integers^n \to \integers^n$ of the form $x \mapsto x + u$, where $u \in \integers^n$. We use $\tra(\integers^n)$ to denote the set of all translations of $\integers^n$.
\end{definition}


\begin{definition}
A (generalized) rotation of $\integers^n$ is a function $r'_A: \integers^n \to \integers^n$ of the form $x \mapsto Ax$, where $A \in \orthmat_n(\integers) := \{A \in \integers^{n\times n}\mid A^\top = A^{-1}\}$. We use $\rot(\integers^n)$ to denote the set of all rotations of $\integers^n$.
\end{definition}
It is easy to check that $(\tra(\integers^n), \circ)$ is isomorphic to $(\integers^n, +)$, and $(\rot(\integers^n), \circ)$ is isomorphic to $(\orthmat_n(\integers), \cdot)$; further, $\tra(\integers^n), \rot(\integers^n) \leq \transf(\integers^n)$, and $\tra(\integers^n), \rot(\integers^n) \subseteq \iso(\integers^n)$, so translations and rotations of $\integers^n$ are transformations and are also isometries. However, we do not know yet whether $(\iso(\integers^n), \circ)$ is a group or whether $\iso(\integers^n) \subseteq \transf(\integers^n)$. It turns out that the results are indeed positive, \ie $\iso(\integers^n) \leq \transf(\integers^n)$, but we need more steps to see this.

\begin{definition}
Let $G \leq \transf(X)$, $Y \subseteq X$, and $G_Y := \{g \in G \mid g(Y) = Y\}$ be the setwise stabilizer of $Y$ under $G$.
The restricted setwise stabilizer of $Y$ under $G$ is the set
\begin{align*}
G_Y|_Y := \{g|_Y\mid g \in G_Y\},
\end{align*}
where $g|_Y: Y \to Y$ is the (surjective) restriction of the function $g$ to $Y$.
\end{definition}

\begin{theorem}\label{thm:fy-fxyy}
For any $Y \subseteq X$, $\transf(Y) = \rstab{\transf}{X}{Y}$.
\end{theorem}
\begin{proof}
See Appendix~\ref{app:fy-fxyy}.
\end{proof}
\begin{corollary}\label{coro:fz-frzz}
$\transf(\integers^n) = \transf(\reals^n)_{\integers^n}|_{\integers^n}$.
\end{corollary}

\begin{theorem}\label{thm:trz-trrzz}
$\tra(\integers^n) = \tra(\reals^n)_{\integers^n}|_{\integers^n}$, and $\rot(\integers^n) = \rot(\reals^n)_{\integers^n}|_{\integers^n}$.
\end{theorem}
\begin{proof}
See Appendix~\ref{app:trz-trrzz}.
\end{proof}

\begin{theorem}\label{thm:isoz-isorzz}
$\iso(\integers^n) = \iso(\reals^n)_{\integers^n}|_{\integers^n}$.
\end{theorem}
\begin{proof}
See Appendix~\ref{app:isoz-isorzz}.
\end{proof}

\begin{remark}
Through restricted setwise stabilizers, Corollary~\ref{coro:fz-frzz} as well as Theorems~\ref{thm:trz-trrzz} and \ref{thm:isoz-isorzz} collectively verify that transformations, translations, rotations, and isometries of $\integers^n$ are precisely those transformations, translations, rotations, and isometries of $\reals^n$ that stabilize $\integers^n$, respectively.
In particular, it is now clear that $(\iso(\integers^n), \circ)$ is indeed a group, and moreover $\tra(\integers^n), \rot(\integers^n) \leq \iso(\integers^n) \leq \transf(\integers^n)$.
\end{remark}

The parallels between translations, rotations, isometries of $\integers^n$ and their counterparts of $\reals^n$ yield the two characterizations of $\iso(\integers^n)$ which are parallel to the those of $\iso(\reals^n)$:
\begin{align*}
\iso(\integers^n) &= \langle \tra(\integers^n) \cup \rot(\integers^n) \rangle \quad \mbox{ where } \tra(\integers^n) \cap \rot(\integers^n) = \{\id\}; \\
\iso(\integers^n) &= \tra(\integers^n) \circ \rot(\integers^n) \cong \integers^n \rtimes \orthmat_n(\integers).
\end{align*}
This further yields the parametrization of $\hcal_{\iso(\integers^n)}^{*}$ by
\begin{align*}
\Sigma''_{isozn} := \{(L, V, \xi) \mid L \leq \orthmat_n(\integers), V \leq \integers^n, \xi \in \Xi''_{L,V}\} \subseteq \Sigma' \subseteq \Sigma,
\end{align*}
where $\Xi''_{L,V} := \{\xi \in \Xi_{L,V} \mid \xi(L) \subseteq \integers^n/V\} \subseteq \Xi_{L,V}$.
Note that $\iso(\integers^n)$ still have infinitely many subgroups, since the choices for $V$ and $\xi$ are still unlimited.
Next we will show how to enumerate a finite subset from $\hcal_{\iso(\integers^n)}^{*}$ when considering the music domain.

The space of music pitches from equal temperament can be denoted by $\integers$.
Every adjacent pitch is separated by a half-step (or semi-tone) denoted by the integer $1$, which is also the distance between every adjacent keys (regardless of black or white) in a piano keyboard.
While the absolute integer assigned to each music pitch is not essential, in the standard MIDI convention, C$_4$ (the middle C) is $60$, C$\sharp _4$ is $61$, and so forth.
Therefore, the space $\integers^n$ represents the space of chords consisting of $n$ pitches. For instance, $\integers^3$ denotes the space of trichords, $\integers^4$ denotes the space of tetrachords, and so forth.
Known music transformations of fixed-size chords \citep{Tymoczko2010, Lewin2010} can be summarized as a subset of the following parametrization set
\begin{align*}
\Sigma''_{music} := \{(L, V, \xi) \mid L \leq \orthmat_n(\integers), V \in \hcal_{\integers^n}^{M}, \xi = \xi^{0}_{L,V} \in \Xi''_{L,V}\} \subseteq \Sigma''_{isozn},
\end{align*}
where $\hcal_{\integers^n}^{M} := \{\langle \bm{1} \rangle, (12\integers)^n, \langle \bm{1} \rangle \vee (12\integers)^n\}$ is a finite collection of music translation subgroups including music transpositions, octave shifts, and their combinations;
$\xi_{L,V}^{0}$ is the trivial vector system given by $\xi_{L,V}^{0}(A) = V$ for any $A \in L$ requiring the inclusion of all rotations to include music permutations and inversions.
Together with the fact that $\orthmat_n(\integers)$ is finite, the enumeration of each element in the triplet $(L,V,\xi)$ is finite, yielding a finite $\Sigma''_{music}$.

It is important to recognize that the significance of the parametrization set $\Sigma''_{music}$ is not limited to recover known music-theoretic concepts but to complete existing knowledge by forming a music ``closure'' $\Sigma''_{music}$.
Such a ``closure'' can be further fine-tuned to be either more efficient (\eg by removing uninteresting rotation subgroups) or more expressive (\eg by adding more translation subgroups).

\subsection{Section Summary}

In this section, we first moved down from the full transformation group of $\reals^n$---the top vertex in the subgroup lattice $(\hcal_{\transf(\reals^n)}^{*}, \leq)$---to the affine group of $\reals^n$.
Focusing on $\affine(\reals^n)$, we derived a complete identification of its subgroups by constructing a parametrization set $\Sigma$ and a bijection $\Psi: \hcal_{\affine(\reals^n)}^{*} \to \Sigma$.
So, every subgroup of $\affine(\reals^n)$ bijectively corresponds to a unique triplet in $\Sigma$.
Towards the goal of a finite collection of affine subgroups, we further moved down in the subgroup lattice $(\hcal_{\transf(\reals^n)}^{*}, \leq)$ from the affine group of $\reals^n$ to the isometry group of $\reals^n$.
Focusing on $\iso(\reals^n)$, we identified the parametrization of $\hcal_{\iso(\reals^n)}^{*}$ by a subset $\Sigma' \subseteq \Sigma$.
From there, we made a dichotomy in our top-down path, and presented two examples to obtain two collections of subgroups used in two different topic domains.
One is a finite collection of space groups (in standard form and up to affine conjugacy) used in crystallography, which is parameterized by $\Sigma''_{cryst} \subseteq \Sigma'$;
the other is a finite completion of existing music concepts, which is parameterized by $\Sigma''_{music} \subseteq \Sigma''_{isozn} \subseteq \Sigma'$. A complete roadmap that we have gone through is summarized in Figure~\ref{fig:top-down-roadmap}.

\begin{figure}[t]
\begin{center}
\includegraphics[width=0.85\columnwidth]{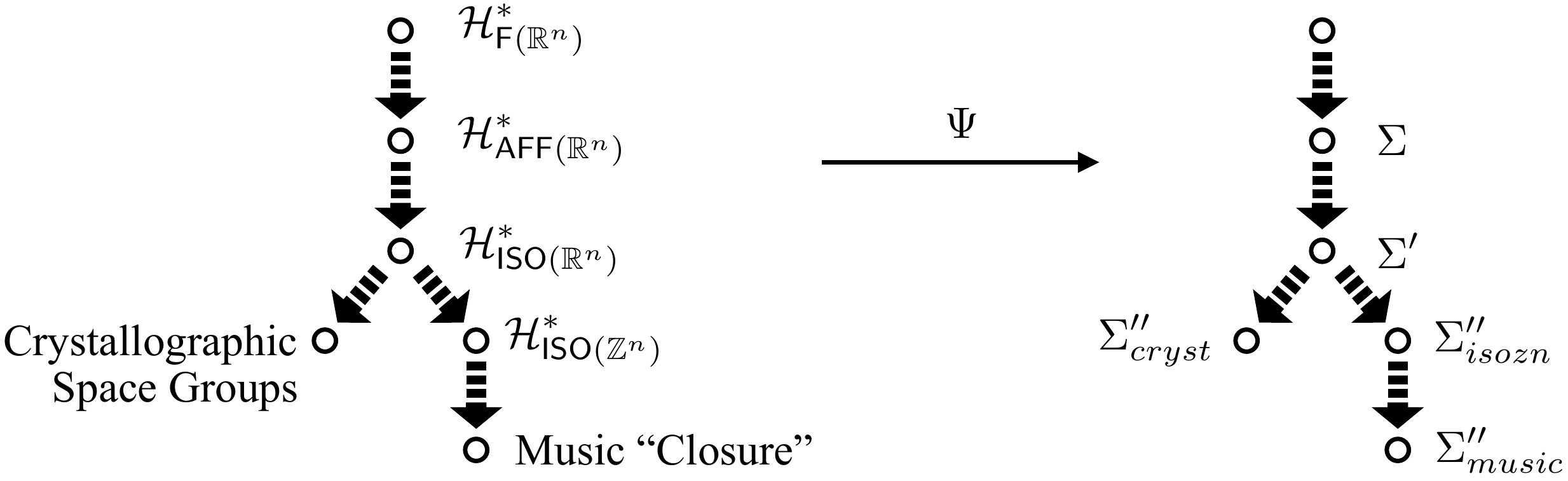}
\end{center}
\caption{Roadmap of the top-down paths in terms of collection of subgroups (left) and the corresponding parametrization paths (right) under the parametrization map $\Psi$.}
\label{fig:top-down-roadmap}
\end{figure}

We finally reiterate that the selection of top-down paths is one's design choice.
Whenever necessary, one should make his/her own decision on creating a new branch or even trying out several branches along major downward paths.
The top-down path with two branches introduced in this section serve for illustration purposes.


\section{The Bottom-Up Approach: Generating Set}
\label{sec:the-bottom-up-approach-gen-set}

We follow a bottom-up approach to extract a partial subgroup lattice $\hcal_{\langle S \rangle}$ from a generating set $S$.
This is done by an induction procedure which first extracts cyclic subgroups $\{\langle s \rangle \mid s \in S\}$ as base cases, and then inductively extracts other subgroups via the join of the extracted ones.
The resulting collection of subgroups $\hcal_{\langle S \rangle}$ is generally not the complete subgroup lattice $\hcal_{\langle S \rangle}^{*}$ since some of its subgroups are missing.
The dual of this induction procedure gives a mirrored induction algorithm that computes the corresponding abstraction semiuniverse in an efficient way.
The missing subgroups can be made up by adding more generators, but this hinders the efficiency.
At the end of this section, we will discuss the trade-off between expressiveness and efficiency when designing a generating set in practice.

\subsection{From Generating Set to Subgroup (Semi)Lattice}
\label{sec:from-gen-set-to-subgroup-semilattice}

Let $S \subseteq \transf(X)$ be a finite subset consisting of transformations of a set $X$.
We construct a collection $\hcal_{\langle S \rangle}$ consisting of subgroups of $\langle S \rangle$ where every subgroup is generated by a subset of $S$.
To succinctly record this process and concatenate it with the abstraction generating chain, we introduce the following one-step \emph{subgroup generating chain}:
\begin{align*}
\mbox{a subset of $\transf(X)$}
\xrightarrow[]{\textit{group generation}}
\mbox{a subgroup of $\transf(X)$},
\end{align*}
which can be further encapsulated by the \emph{subgroup generating function} defined as follows.

\begin{definition}\label{def:subgroup-generating-function}
The subgroup generating function is the mapping $\pi': 2^{\transf(X)} \to \hcal_{\transf(X)}^{*}$ where $2^{\transf(X)}$ is the power set of $\transf(X)$, $\hcal_{\transf(X)}^{*}$ is the collection of all subgroups of $\transf(X)$, and for any $S \in 2^{\transf(X)}$, $\pi'(S) := \langle S \rangle := \{s_k \circ \cdots \circ s_1 \mid s_i \in S \cup S^{-1}, i = 1, \ldots, k, k \in \integers_{\geq 0}\}$ where $S^{-1} := \{s^{-1} \mid s \in S\}$.
By convention, $s_k \circ \cdots \circ s_1 = \id$ for $k = 0$, and $\pi'(\emptyset) = \langle \emptyset \rangle = \{\id\}$.
\end{definition}

\begin{remark}
The subgroup generating function in Definition~\ref{def:subgroup-generating-function} is nothing but generating a subgroup from its given generating set.
However, we can now write the procedure at the beginning of this subsection succinctly as $\hcal_{\langle S \rangle} = \pi'(2^S)$ for any finite subset $S \subseteq \transf(X)$;
further, the subgroup generating chain and the abstraction generating chain can now be concatenated, which is denoted by the composition $\Pi := \pi \circ \pi'$.
\end{remark}

Like the abstraction generating function $\pi$, the subgroup generating function $\pi'$ is not necessarily injective, since a generating set of a group is generally not unique; $\pi'$ is surjective, since every subgroup per se is also its own generating set.
The following theorem captures the structure of $\pi'(2^S)$ for a finite subset $S \subseteq \transf(X)$.

\begin{theorem}\label{thm:subsets-generated-subgroups-is-join-semilattice}
Let $S \subseteq \transf(X)$ be a finite subset, and $\pi'$ be the subgroup generating function.
Then $(\pi'(2^S), \leq)$ is a join-semilattice, but not necessarily a meet-semilattice.
In particular,
\begin{align*}
\pi'(A\cup B) = \pi'(A) \vee \pi'(B) \quad \mbox{ for any } A,B \subseteq S.
\end{align*}
\end{theorem}
\begin{proof}
For any $A,B \subseteq S$, we have
\begin{align*}
\pi'(A\cup B) = \langle A\cup B \rangle = \langle \langle A \rangle \cup \langle B \rangle \rangle = \langle \pi'(A) \cup \pi'(B) \rangle = \pi'(A) \vee \pi'(B).
\end{align*}
Then for any $\pi'(A), \pi'(B) \in \pi'(2^S)$ where $A,B \subseteq S$, the join $\pi'(A) \vee \pi'(B)  = \pi'(A \cup B) \in \pi'(2^S)$, since $A\cup B \subseteq S$. So, $(\pi'(2^S), \leq)$ is a join-semilattice.

We give an example in which $(\pi'(2^S), \leq)$ is not a meet-semilattice. Let $X = \reals^n$ and $S = \{t_{\bm{e_1}}, t_{\bm{e_2}}, t_{(3/2)\cdot\bm{1}}\}$ be a set consisting of three translations where $\bm{e_1} = (1,0), \bm{e_2} = (0,1), \bm{1} = (1,1)$.
Further, let $A = \{t_{\bm{e_1}}, t_{\bm{e_2}}\}$ and $B = \{t_{(3/2)\cdot\bm{1}}\}$. The meet $\pi'(A) \wedge \pi'(B) = \langle A \rangle \cap \langle B \rangle = \langle t_{3\cdot \bm{1}} \rangle \not\in \pi'(2^S)$.
\end{proof}

\begin{remark}
Although the collection of subgroups generated by the subgroup generating function $\pi'$ is not a lattice in general, it is sufficient that it is a join-semilattice.
This is because the family of abstractions generated by the abstraction generating function $\pi$ is a meet-semiuniverse (recall the strong and week dualities in Theorem~\ref{thm:duality-subgroup-and-partition-lattice}).
As a result, the closedness of $\pi'(2^S)$ under join is carried over through the strong duality to preserve the closedness of $\pi(\pi'(2^S))$ under meet. This preservation of closednesses under join and meet has a significant practical implication: it directly yields an induction algorithm that implements $\Pi(2^S) := \pi \circ \pi'(2^S)$ from a finite subset $S$.
\end{remark}

\subsection{An Induction Algorithm}
\label{sec:an-induction-algorithm}

We describe an algorithmic implementation of $\Pi(2^S)$ from a finite generating set $S$ for a finite space $X$.
More specifically, the algorithmic problem here is as follows.
\begin{equation}
\label{eqn:induction-algo-problem}
\begin{aligned}
\mbox{Inputs:} \quad & \mbox{1) a finite set $X$ (to be abstracted)}, \\
& \mbox{2) a finite generating set $S \subseteq \transf(X)$}; \\
& \mbox{the underlying group action is: $\langle S \rangle$ acting on $X$}. \\
\mbox{Output:} \quad & \mbox{the abstraction semiuniverse (for $X$)} \\
& \qquad \Pi(2^S) = \pi \circ \pi'(2^S) = \{X/\langle S' \rangle \mid S' \subseteq S\}.
\end{aligned}
\end{equation}
By literally following the definition $\Pi = \pi \circ \pi'$, a naive algorithm solves Problem~\eqref{eqn:induction-algo-problem} straightforwardly in two steps.
It first considers the subgroup join-semilattice $\pi'(2^S)$, then computes the abstraction meet-semiuniverse $\pi(\pi'(2^S))$.
However, as mentioned in Remark~\ref{remark:duality-subgroup-and-partition-lattice}, computing every abstraction of $X$ by identifying orbits from a subgroup action can be expensive.
In this subsection, we first present a naive two-step implementation, and then introduce an induction algorithm that solves Problem~\eqref{eqn:induction-algo-problem} efficiently in an indirect way.

\paragraph{A Naive Two-Step Implementation.}
Step One: consider $\pi'(2^S) = \{\langle S' \rangle \mid S' \subseteq S\}$.
This is straightforward since we can simply enumerate (possibly with duplication) every subgroup in $\pi'(2^S)$ by indexing its generating set $S' \subseteq S$.
Step Two: consider $\pi(\pi'(2^S))$, \ie for every $S' \subseteq S$ and its corresponding $\pi'(S') = \langle S' \rangle$, we compute $\pi(\pi'(S')) = X/\langle S' \rangle$ by identifying the set of orbits $\{\langle S' \rangle x\mid x \in X\}$.
More specifically, as a subroutine, for every pair $x,x' \in X$, we need to check whether or not they are in the same orbit---known as the \emph{orbit identification} problem.
The number of checks needed is $O(|X|^2)$ which can be computationally expensive if $|X|$ is large.
However, what really makes this naive approach fail is that most checks may not finish in finite time.
To get a rough sense, take $S' = \{s_1,s_2\}$ as an example, without leveraging additional properties, a brute-force check may be endless: ``$x' = s_1(x)$?'', ``$x' = s_1^{-1}(x)$?'', ``$x' = s_2(x)$?'', ``$x' = s_1\circ s_2 \circ s_2 \circ s_1^{-1}(x)$?'', and so forth.

\paragraph{An Induction Algorithm.}
Instead, we give an algorithm based on induction on $|S'|$ for all nonempty subsets $S' \subseteq S$.
(Note: $\Pi(\emptyset) = X/\langle \emptyset \rangle$ is simply the finest partition of $X$.)

\noindent \underline{Base case:} compute $\Pi(S')$ for $|S'|=1$ (Algorithm~\ref{algo:base-partn}) as orbits under a cyclic subgroup:
\begin{align}\label{eqn:base-case}
\Pi(S') = \{\langle S' \rangle x\mid x \in X\}.
\end{align}
\underline{Induction step:} compute $\Pi(S')$ for $|S'|>1$ (Algorithm~\ref{algo:meet}) as the meet of two partitions:
\begin{align}\label{eqn:induction-step}
\Pi(S') = \Pi(S'') \wedge \Pi(S''') \quad \mbox{ for any } S'',S''' \subset S' \mbox{ and } S'' \cup S''' = S'.
\end{align}
In the base case, every partition, called a \emph{base partition}, is explicitly computed from orbits.
Yet orbit identification is feasible for one generator, say $s$, since two points $x,x' \in X$ are in the same orbit if and only if $x' = s^n (x)$ for some finite $n \in \integers$.
We can do even better than the quadratic-time orbit identification: Algorithm~\ref{algo:base-partn} uses \emph{orbit tracing} instead, which is linear in $|X|$.
In the induction step, the meet operation bypasses the endless brute-force checks in orbit identification.
Its correctness can be proved by leveraging Theorem~\ref{thm:subsets-generated-subgroups-is-join-semilattice} and the strong duality in Theorem~\ref{thm:duality-subgroup-and-partition-lattice}, or more explicitly,

\begin{minipage}[t]{0.45\textwidth}
\begin{algorithm}[H]
\KwIn{a generator $s$ and a set $X$}
\KwOut{the base partition $\Pi(\{s\})$}
\vspace{0.05in}
\SetKwFunction{FBasePartn}{BasePartn}
\SetKwProg{Fn}{Function}{:}{}
\Fn{\FBasePartn{s}}{
initialize label id: $l = 0$\;
\For{each point $x \in X$}
{
  \If{$x$ is not labeled}{
    initialize a new orbit:\\
    \quad $O = \{x\}$\;
    transform:\\
    \quad $y = s(x)$\;
    \While{$y \in X$ and $y \neq x$ and $y$ is not labeled}{
      enlarge the orbit:\\
      \quad $O = O \cup \{y\}$\;
      transform:\\
      \quad $y = s(y)$\;
    }
    \If{$y \notin X$ or $y = x$}{
      create a new label:\\
      \quad $l = l + 1$\;
    }
    \If{$y$ is labeled}{
      use $y$'s label:\\
      \quad $l = y$'s label\;
    }
    label every point in the orbit $O$ by $l$\;
  }
}
\KwRet the partition according to the labels\;
}
\vspace{0.1in}
\caption{Computing base partitions by tracing orbits: $O(|X|)$.}
\label{algo:base-partn}
\end{algorithm}
\vspace{0.1in}
\end{minipage}
~
\begin{minipage}[t]{0.47\textwidth}
\begin{algorithm}[H]
\KwIn{two partitions $\pcal$ and $\qcal$ of a set $X$}
\KwOut{the meet $\pcal \wedge \qcal$, \ie finest common coarsening of $\pcal$ and $\qcal$}
\vspace{0.05in}
\SetKwFunction{FMeet}{Meet}
\SetKwProg{Fn}{Function}{:}{}
\Fn{\FMeet{$\pcal$, $\qcal$}}{
\For{each cell $Q \in \qcal$}
{
  initialize a new cell:\\
  \quad $P_{merge} = \emptyset$\;
  \For{each cell $P \in \pcal$}
  {
  \If{$P \cap Q \neq \emptyset$}{
    merge:
    \quad $P_{merge} = P_{merge} \cup P$\;
    remove:\\
    \quad $\pcal = \pcal \backslash \{P\}$\;
    }
  }
  insert:\\
  \quad $\pcal = \pcal \cup \{P_{merge}\}$\;
}
\KwRet $\pcal$\;
}
\vspace{0.1in}
\caption{Computing partitions generated from more than one generators inductively by taking the meet of two partitions computed earlier: $O(|\pcal||\qcal|)$. Normally, all base partitions should be already computed and cached before running the induction steps.}
\label{algo:meet}
\end{algorithm}
\end{minipage}

\begin{align*}
\Pi(S') &= \pi \circ \pi'(S'' \cup S''') \\
&= \pi(\pi'(S'')\vee \pi'(S''')) \\
&= \pi\circ \pi'(S'') \wedge \pi\circ\pi'(S''') \\
&= \Pi(S'') \wedge \Pi(S''').
\end{align*}
The above proof holds for any pair $S'', S''' \subseteq S'$ as long as $S'' \cup S''' = S'$.
However, different choices of $(S'',S''')$ can yield different run time which will be discussed later.

\paragraph{Extended Use in a General Setting.}
The induction algorithm, consisting of several runs of Algorithm~\ref{algo:base-partn} followed by several runs of Algorithm~\ref{algo:meet}, works for any finitely generated group $\langle S \rangle$ acting on any finite set $X$, \ie Problem~\eqref{eqn:induction-algo-problem} in general.
But we may do it more generally.
The way Algorithm~\ref{algo:base-partn} is currently written---especially the simple check of whether $y \in X$---allows the induction algorithm to be run in a more \emph{general setting},
where $\langle S \rangle$ acts on a possibly infinite set $\tilde{X}$, and our input $X$ is a subset of the larger ambient space $\tilde{X}$.
Compared to Problem~\eqref{eqn:induction-algo-problem}, this more general setting is more precisely stated as follows.
\begin{equation}
\label{eqn:induction-algo-problem-general}
\begin{aligned}
\mbox{Inputs:} \quad & \mbox{1) a possibly infinite set $\tilde{X}$ (to be abstracted)}, \\
& \mbox{2) a finite generating set $S \subseteq \transf(\tilde{X})$}, \\
& \mbox{3) a finte subset $X \subseteq \tilde{X}$}; \\
& \mbox{the underlying group action is: $\langle S \rangle$ acting on $\tilde{X}$ (not $X$)}. \\
\mbox{Output:} \quad & \mbox{the abstraction semiuniverse (for $\tilde{X}$ restricted to $X$)} \\
& \qquad \Pi(2^S) = \pi \circ \pi'(2^S) = \{X/\langle S' \rangle \mid S' \subseteq S\}, \\
& \mbox{where $X/\langle S' \rangle := \{\langle S'\rangle x \cap X\mid x\in X\}$.}
\end{aligned}
\end{equation}
Here we are using a non-standard yet unambigiuous notation $X/\langle S' \rangle$ to mean the partition of $X$ obtained from restricting $\tilde{X}/\langle S' \rangle$ to $X$, since the group action is on $\tilde{X}$ instead of $X$.
Problem~\eqref{eqn:induction-algo-problem-general} is more general since it includes Problem~\eqref{eqn:induction-algo-problem} as a special case by setting $X = \tilde{X}$ for a finite $\tilde{X}$.
So now, we can do computational abstraction on an infinite input space, and presents the result on whatever finite subspace is asked for.
However, we need to be careful here.
Even though the algorithm will run, it is not always accurate in the general setting (it is accurate for Problem~\eqref{eqn:induction-algo-problem} only): both Algorithm~\ref{algo:base-partn} and \ref{algo:meet} may only give an approximation for Problem~\eqref{eqn:induction-algo-problem-general}.
A full discussion on this general setting, including where the approximation might occur and how to correct it, will be detailed in Section~\ref{sec:restrict-to-finite-subspaces}.

\begin{figure}[t]
\begin{center}
\includegraphics[width=0.63\columnwidth]{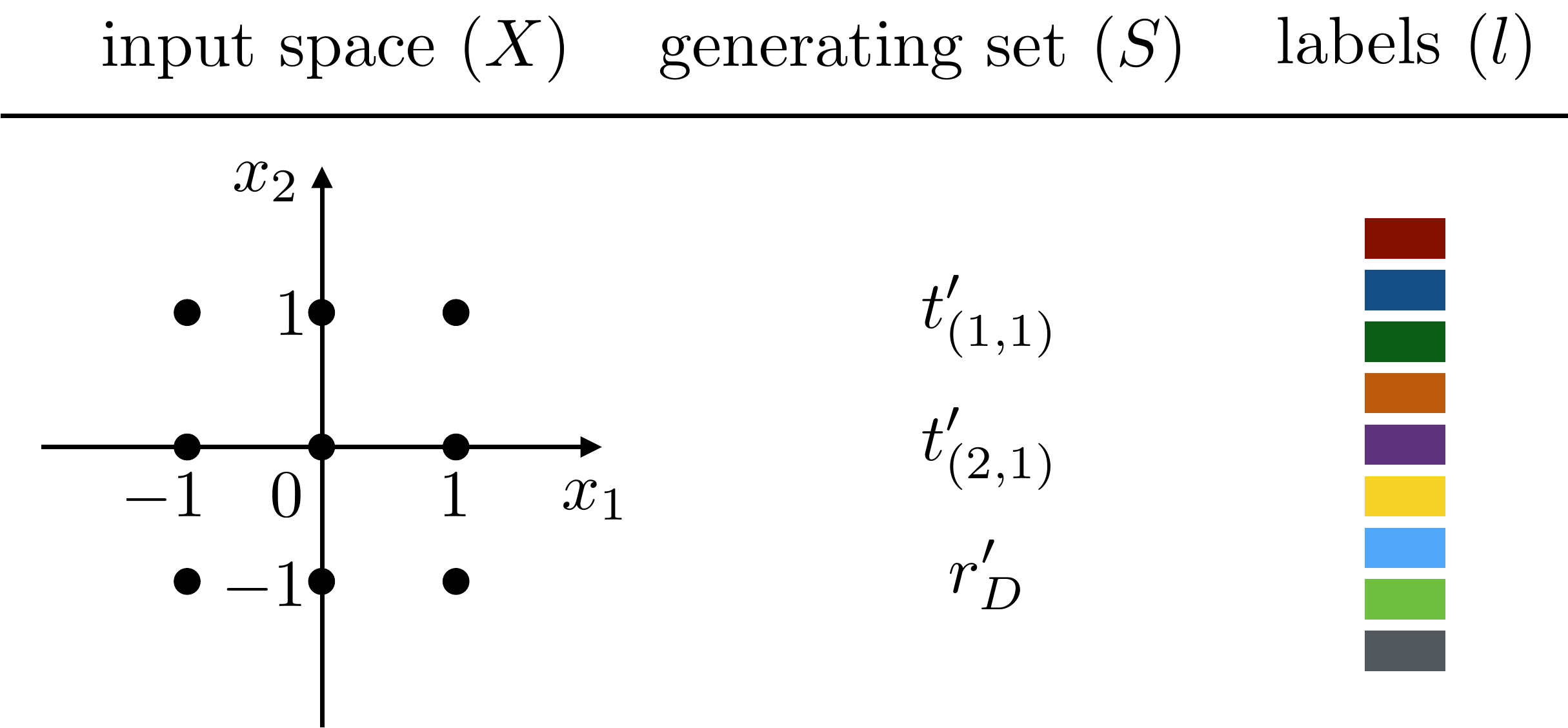}
\end{center}
\caption{A toy example on which the induction algorithm is to be run.}
\label{fig:toy-example-for-induction-algo}
\end{figure}

\paragraph{Run on an Toy Example.}
We walk through a complete run of the above induction algorithm on a toy example (Figure~\ref{fig:toy-example-for-induction-algo}).
In order to hit all cases and show the algorithm's full functionality, we draw the example from the general setting formalized in Problem~\eqref{eqn:induction-algo-problem-general}:
\begin{align*}
X &= \{x \in \integers^2 \mid \|x\|_\infty \leq 1\} \subseteq \integers^2, \\
S &= \{t'_{(1,1)}, t'_{(2,1)}, r'_{D}\} \subseteq \transf(\integers^2) \quad \mbox{ where } D = \begin{bmatrix}0&1\\1&0\end{bmatrix}.
\end{align*}
So here, $X = \{(-1,-1),(0,-1),(1,-1),(-1,0),(0,0),(1,0),(-1,1),(0,1),(1,1)\}$ is a nine-element set;
$S$ contains three generators, where $t'_{(1,1)}, t'_{(2,1)}$ are two translations by vectors $(1,1), (2,1)$, respectively, and $r'_{D}$ is the reflection about the diagonal line ($x_2 = x_1$).
More explicitly, for any $x = (x_1,x_2) \in X$,
$t'_{(1,1)}(x) = (x_1+1,x_2+1)$, $t'_{(2,1)}(x) = (x_1+2,x_2+1)$, and $r'_{D}(x) = (x_2,x_1)$.
The desired output for this toy example is the following abstraction semiuniverse consisting of $2^{|S|} = 2^3 = 8$ abstractions:
\begin{align*}
\Pi(2^S) &= \{~X/\langle\emptyset\rangle, \\
& \hspace{0.26in}~
X/\langle t'_{(1,1)} \rangle, ~~X/\langle t'_{(2,1)} \rangle, ~~X/\langle r'_{D} \rangle, \\
& \hspace{0.26in}~
X/\langle\{t'_{(1,1)}, t'_{(2,1)}\}\rangle, ~~X/\langle\{t'_{(2,1)}, r'_{D}\}\rangle, ~~X/\langle\{r'_{D}, t'_{(1,1)}\}\rangle, \\
& \hspace{0.26in}~
X/\langle\{t'_{(1,1)}, t'_{(2,1)}, r'_{D}\}\rangle~\}.
\end{align*}

\vspace{0.1in}
\noindent \underline{Global procedure.}
In a full run of the induction procedure, we first run Algorithm~\ref{algo:base-partn} three independent times to compute the three base partitions $X/\langle t'_{(1,1)} \rangle$, $X/\langle t'_{(2,1)} \rangle$, $X/\langle r'_{D} \rangle$;
we then run Algorithm~\ref{algo:meet} three independent times to compute the next three partitions $X/\langle\{t'_{(1,1)}, t'_{(2,1)}\}\rangle$, $X/\langle\{t'_{(2,1)}, r'_{D}\}\rangle$, $X/\langle\{r'_{D}, t'_{(1,1)}\}\rangle$, each of which is computed from its corresponding base partitions, \eg $X/\langle\{t'_{(1,1)}, t'_{(2,1)}\}\rangle$ is computed from $X/\langle t'_{(1,1)} \rangle$ and $X/\langle t'_{(2,1)} \rangle$.
Lastly, we run Algorithm~\ref{algo:meet} one more time to compute $X/\langle\{t'_{(1,1)}, t'_{(2,1)}, r'_{D}\}\rangle$ from earlier computed partitions, \eg from $X/\langle\{t'_{(1,1)}, t'_{(2,1)}\}\rangle$ and $X/\langle r'_{D} \rangle$ (among many other choices).
Note that $X/\langle \emptyset \rangle := \{\{x\}\mid x\in X\}$ is the finest partition of $X$, which does not require any algorithmic run.
Next, we will zoom into this global process, and walk the readers through each individual run of Algorithm~\ref{algo:base-partn} and Algorithm~\ref{algo:meet} one at a time.

\begin{figure}[t]
\begin{center}
\includegraphics[width=0.8\columnwidth]{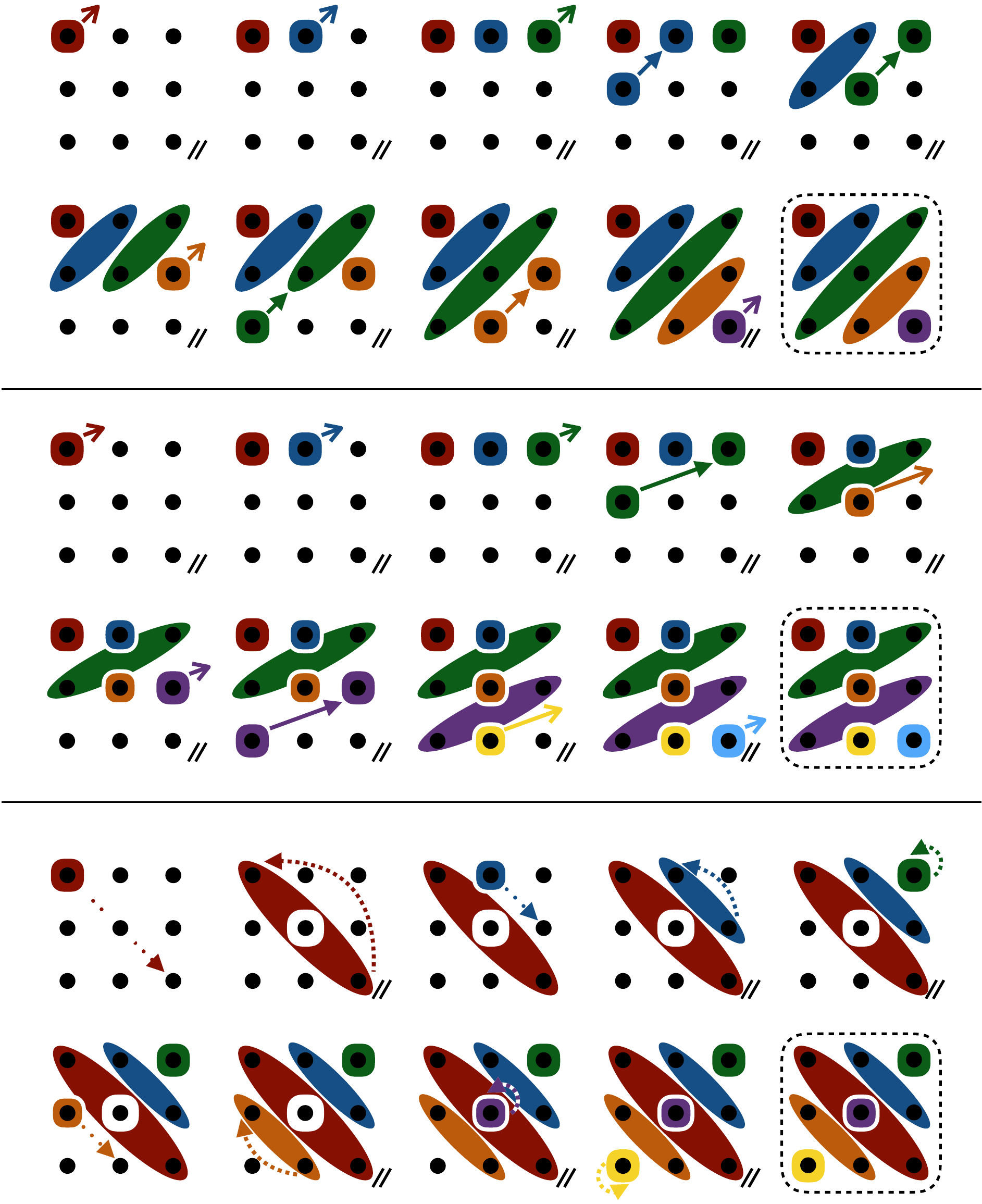}
\end{center}
\caption{Three independent runs of Algorithm~\ref{algo:base-partn} for generators $t'_{(1,1)}$ (top), $t'_{(2,1)}$ (middle), and $r'_{D}$ (bottom), respectively. Every color denotes the label of a partition cell; every double slash (/\!/) marks the end of a for-loop iteration; every arrow indicates a transformation (see more details in the text for the four types of arrows). Final outputs are in dashed boxes.}
\label{fig:toy-algo1}
\end{figure}

\vspace{0.1in}
\noindent \underline{Algorithm~\ref{algo:base-partn}.}
The three independent runs of this algorithm are illustrated in Figure~\ref{fig:toy-algo1}.
The sequential coloring of the nine dots in the input set $X$ denotes the labeling process; the double slash (/\!/) denotes the end of an iteration in the for-loop (so, the first two runs cost nine iterations each, and the third run costs six iterations).
Both the coloring sequence and the double slashes mark the linear progress of the algorithm.
Closer attention should be paid to the arrows: every arrow, regardless of its type, denotes an action of applying the transformation (specified by the single generator).
However, the result of the transformation includes four cases (corresponding to the four conditions in Algorithm~\ref{algo:base-partn}, namely $y\notin X$, $y=x$, $y~is~labeled$, and otherwise).
Therefore, we adopt four different types of arrows to denote these four cases, respectively.

\vspace{-0.05in}
\begin{itemize}
\setlength\itemsep{0in}
\item [1.] \textit{An open arrow} indicates that the transformed point jumps out of the input set (see condition $y \notin X$).
There is only one type of open arrow (arrowhead is $>$); contrarily, there are three types of triangle arrows (arrowhead is a filled triangle).
\item [2.] \emph{A solid triangle arrow} indicates that the transformed point is an already labeled point (see condition $y~is~labeled$).
\item [3.] \emph{A dashed triangle arrow} indicates that the transformed point coincides with an earlier point in the transformation sequence (see condition $y = x$).
\item [4.] \emph{A dotted triangle arrow} indicates that the transformed point does not fall under any of the above conditions (see condition of the while-loop).
\end{itemize}
\vspace{-0.05in}
The final output of each run of Algorithm~\ref{algo:base-partn}---a partition of $X$ whose cells are marked by colors---is shown in a dashed box in Figure~\ref{fig:toy-algo1}.

\begin{figure}[t]
\begin{center}
\includegraphics[width=0.98\columnwidth]{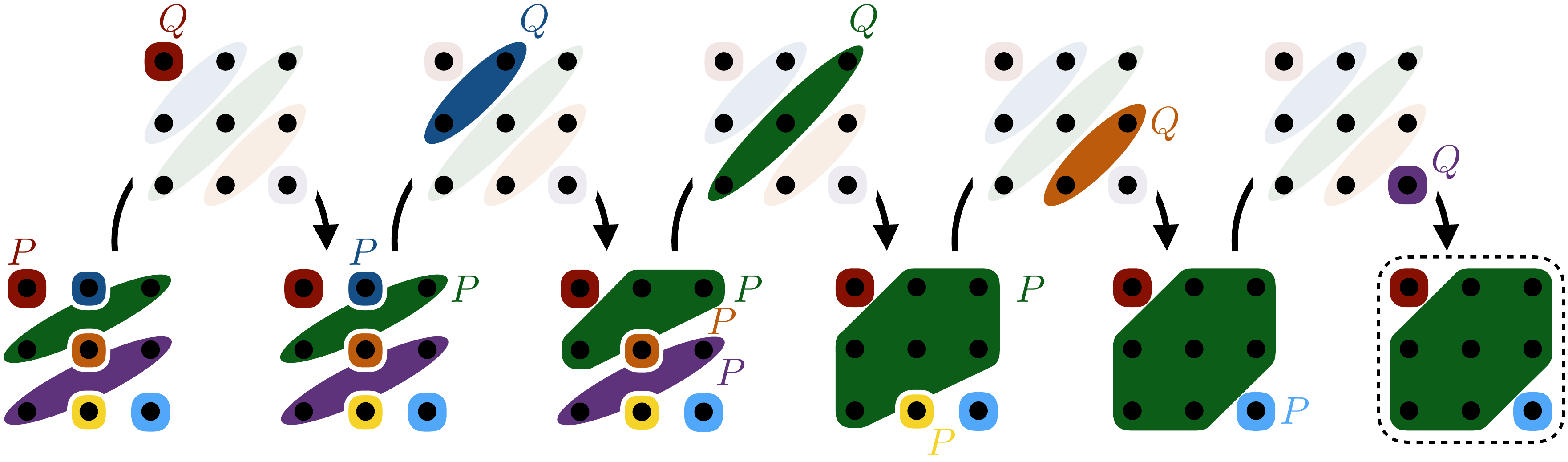}
\end{center}
\caption{One run of Algorithm~\ref{algo:meet} from input partitions $\pcal = X/\langle t'_{(2,1)} \rangle$ and $\qcal = X/\langle t'_{(1,1)} \rangle$. The outer for-loop over $Q \in \qcal$ is shown in the top row; the inner for-loop over $P \in \pcal$ is shown in the bottom row.
The marked $P$ cells are the ones that intersect with the following $Q$ cell.
During the entire process, $\qcal$ remains unchanged, whereas $\pcal$ changes in place from one of the algorithm inputs to the algorithm output.
The final output is in the dashed box.}
\label{fig:toy-algo2}
\end{figure}

\noindent \underline{Algorithm~\ref{algo:meet}.}
One run of this algorithm is illustrated in Figure~\ref{fig:toy-algo2} (other runs are similar).
On this run, the algorithm takes as inputs two partitions $\pcal = X/\langle t'_{(2,1)} \rangle$, $\qcal = X/\langle t'_{(1,1)} \rangle$, and computes their meet $\pcal \wedge \qcal$ for our desired output $X/\langle\{t'_{(1,1)}, t'_{(2,1)}\}\rangle$.
This is a pretty standard algorithm that computes the meet of two partitions in general.
The idea is to take one of the input partitions, say $\qcal$, as a fixed reference partition, and iterates its cells in the outer loop;
we then gradually change the other partition $\pcal$ \emph{in place} from an input partition to the output partition.
The changes are made in each inner loop: for each $Q \in \qcal$, we collect all cells in the current $\pcal$ that intersect with $Q$, and merge them into one cell.
Take the third outer iteration in Figure~\ref{fig:toy-algo2} as an example where $Q = \{(-1,-1),(0,0),(1,1)\}$ (green):
three out of the six cells in $\pcal$---namely $\{(-1,0),(0,1),(1,1)\}$ (green), $\{(0,0)\}$ (orange), and $\{(-1,-1),(1,0)\}$ (purple)---intersect with $Q$, so they are merged into one single cell (the green cell in the following $\pcal$).
The output of this run of Algorithm~\ref{algo:meet}---again a partition of $X$ whose cells are marked by colors---is shown in a dashed box in Figure~\ref{fig:toy-algo2}.
After all runs of Algorithm~\ref{algo:meet} finish, the final output of the entire induction algorithm is shown in Figure~\ref{fig:toy-algo-output-hierarchy}.

\begin{figure}[t]
\begin{center}
\includegraphics[width=0.8\columnwidth]{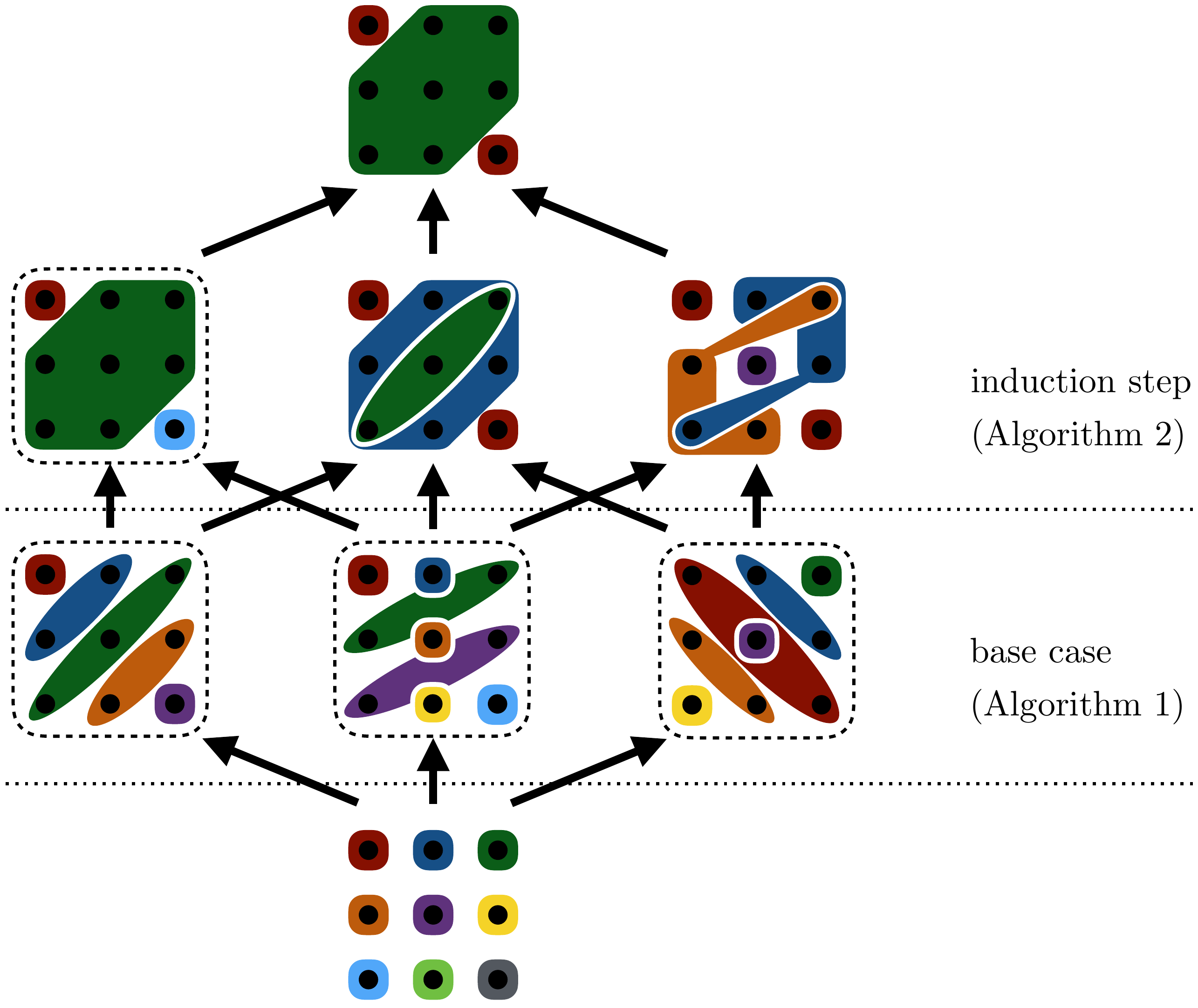}
\end{center}
\caption{The final output of the induction algorithm run on the toy example in Figure~\ref{fig:toy-example-for-induction-algo}. Every directed edge ($\acal \to \acal'$) denotes the ``finer than'' relation ($\acal \succeq \acal'$), since the input partitions of Algorithm~\ref{algo:meet} are both finer than the output partition. Partitions in the dashed boxes are the same partitions in the dashed boxes from Figures~\ref{fig:toy-algo1} and \ref{fig:toy-algo2}.}
\label{fig:toy-algo-output-hierarchy}
\end{figure}

\vspace{0.1in}
\noindent \underline{Caution: correctness.}
It is important to notice that this toy example falls under the general setting formalized in Problem~\eqref{eqn:induction-algo-problem-general}.
The underlying group action here is $\langle S \rangle$ acting on $\integers^n$ (not $X$) and the input space $X$ is a finite subset of $\integers^n$.
So, not all output partitions in Figure~\ref{fig:toy-algo-output-hierarchy} are correct.
It is an easy check that while all base partitions are correct, the leftmost partition in the second row from the top in Figure~\ref{fig:toy-algo-output-hierarchy} is only an approximation of the desired partition $X/\langle\{t_{(1,1)},t_{(2,1)}\}\rangle = \{X\}$ (since $(1,1)$ and $(2,1)$ form a basis of $\integers^n$).
However, every incorrect partition in Figure~\ref{fig:toy-algo-output-hierarchy} can be corrected by running the additional ``Expand-and-Restrict'' technique right after each run of Algorithm~\ref{algo:meet}.
See Section~\ref{sec:restrict-to-finite-subspaces} for more details, and particularly Section~\ref{sec:expand-and-restrict} for the rectification technique.

\paragraph{Computational Complexity.}
Recall in the general setting formalized in Problem~\eqref{eqn:induction-algo-problem-general}, each of the desired abstractions of $X$ is the partition $X/\langle S' \rangle = \{\langle S'\rangle x \cap X\mid x\in X\} = (\tilde{X}/\langle S' \rangle)|_X$, \ie the partition $\tilde{X}/\langle S' \rangle$ restricted to the subspace $X$ (see the precise definition in Definition~\ref{def:an-abstraction-of-a-subspace}).
The crux of this problem is \emph{orbit identification}: given any $x,x' \in X$, check whether there exists an $s \in \langle S' \rangle$ such that $x' = s\cdot x$.
In some special cases, there might be smart ways of conducting these checks in a magnitude lower than $O(|X|^2)$ times (like in our \emph{orbit tracing}).
But things may get tough in general.
So, before analyzing the computational complexity of our induction algorithm, we first discuss the complexity of the problem.

\vspace{0.1in}
\noindent \underline{Complexity of the Problem.}
We analyze the complexity of the problem---particularly, the problem of orbit identification---in the worst case.
Here, the worst case is among all possible $S$ (which implies all possible $S' \subseteq S$) and all possible $\tilde{X}$ (which implies all possible $X \subseteq \tilde{X}$).
Unfortunately, we can show that in the worst case, this problem is \emph{unsolvable}: it is even harder than the famous \emph{word problem} for groups, and the word problem was shown to be unsolvable \citep{Novikov1955, Boone1958, Britton1958}.
Our problem is harder because it includes the word problem as a special case.
To see this, let $G = \langle U | R \rangle$ be a group presentation, and consider the group action of $\langle\langle R \rangle\rangle$ acting on $F(U)$, where $\langle\langle R \rangle\rangle$ is the normally generated subgroup by $R$, $F(U)$ is the free group over $U$, and the group action is the group operation of $F(U)$ restricted to $\langle\langle R \rangle\rangle \times F(U)$.
The orbit identification problem for this group action is: given any $w,w' \in F(U)$, decide whether there exists an $r \in \langle\langle R \rangle\rangle$ such that $w' = r\cdot w$.
This is equivalent to decide when the two words $w,w'$ (elements of a free group are called \emph{words}) represent the same element in $G = F(U)/\langle\langle R \rangle\rangle$, or equivalently, when $w'\cdot w^{-1} \in \langle\langle R \rangle\rangle$, which is precisely the statement of the word problem.
A simple example of $G=\langle U|R \rangle$ with unsolvable word problem can be found in \citet{Collins1986}.
Although our problem is unsolvable in the worst case, it is solvable in many special cases.
For example, our induction algorithm can solve all instances of Problem~\eqref{eqn:induction-algo-problem} exactly; when further equipped with the ``Expand-and-Restrict'' technique (Section~\ref{sec:expand-and-restrict}), our algorithm can also solve many instances of Problem~\eqref{eqn:induction-algo-problem-general} either exactly or approximately.

\vspace{0.1in}
\noindent \underline{Complexity of the Induction Algorithm.}
Our induction algorithm always runs Algorithm~\ref{algo:base-partn} $|S|$ times for the base case, then runs Algorithm~\ref{algo:meet} in the induction step $2^{|S|}-|S|-1$ times.
In any run of Algorithm~\ref{algo:base-partn}, every point in the input space $X$ will be transformed exactly once and be labeled exactly once.
Thus, the complexity of Algorithm~\ref{algo:base-partn} is always $O(|X|)$ in all runs, yielding a total complexity of $O(|S||X|)$ for the base case.
Regarding the complexity of Algorithm~\ref{algo:meet}, we count the total number of merge steps.
The outer for-loop has exactly $|\qcal|$ iterations, and the number of merge steps in every outer iteration is precisely the number of the corresponding inner iterations.
The number of iterations in each inner for-loop is the number of cells in the current $\pcal$ that intersect with the current $Q \in \qcal$, which unfortunately varies from case to case.
However, given that the size of $\pcal$ is monotonically decreasing, we can upper bound this number by the initial $|\pcal|$.
So, the total number of merge steps is upper bounded by $|\pcal||\qcal|$; the complexity of Algorithm~\ref{algo:meet} is upper bounded by $O(|\pcal||\qcal|)$.
Unlike Algorithm~\ref{algo:base-partn} which has a fixed complexity for every individual run, the complexity of Algorithm~\ref{algo:meet} varies from run to run according to the sizes of each run's input partitions.
In case a rough estimation is useful, we can further upper bound $O(|\pcal||\qcal|)$ consistently by $O(|X|^2)$ for $\pcal$, $\qcal$ in all runs of Algorithm~\ref{algo:meet}.
Therefore, the total complexity of the entire induction algorithm is very loosely upper bounded by
\begin{align*}
O(|S||X|)+O(2^{|S|}|X|^2).
\end{align*}
We make two comments about this upper bound.
First, the exponential growth with respect to $|S|$ is not only inevitable but also necessary.
This is because our induction algorithm is expected to output a (semi)universe of $O(2^{|S|})$ abstractions, and in general we want a large universe so that the abstractions therein are as expressive as possible.
Second, while the first term in the upper bound is tight, the second term with a quadratic dependence on $|X|$ is very loose.
This suggests that we can do better, and sometimes much better, than $|X|^2$ in practice.
It is worth noting that how loose the bound is depends on the generators in $S$, the input space $X$, and where in the hierarchy Algorithm~\ref{algo:meet} is run.
While the first two are given (so we do not have a choice), we often have a choice for $\pcal$ and $\qcal$ when running Algorithm~\ref{algo:meet}, since in Equation~\eqref{eqn:induction-step} in the induction step, we can often have many choices for selecting $S''$ and $S'''$ especially when $S'$ is large.
Therefore, a clear strategy is to pick $\pcal = \Pi(S'')$ and $\qcal = \Pi(S''')$ among the candidate set $\{(S'',S''')|S''\subset S', S''' \subset S', S''\cup S''' = S'\}$ such that the product $|\pcal||\qcal|$ is minimized.

\vspace{0.1in}
In summary, the essence of the induction step is to do orbit identification indirectly.
However, we cannot get around orbit identification, but always have to do it in some form.
Therefore, it is important to notice that while our induction algorithm runs on all cases (including the worst case), this does \emph{not} contradict the fact that the problem in the worst case is unsolvable.
This is because the output from running the induction algorithm alone is not always accurate in the general setting (\ie Problem~\eqref{eqn:induction-algo-problem-general}), especially when considering a group action on an infinite set whose orbits are further to be restricted to a finite subset.
As will be detailed in Section~\ref{sec:restrict-to-finite-subspaces}, we need the Expand-and-Restrict technique to correct errors.
In the worst case and when approximations are unacceptable, we need an infinite expand, which is computationally infeasible and agrees with the unsolvability of the problem.

\subsection{Finding a Generating Set of $\iso(\integers^n)$}
\label{sec:find-a-gen-set}

We give an example of finding a finite generating set.
The key idea is based on recursive group decompositions.
In light of storing abstractions of a set $X$ in digital computers, we consider the discrete space $X = \integers^n (\subseteq \reals^n)$.
Further, we restrict our attention to generators that are isometries of $\integers^n$, since $\iso(\integers^n)$ is finitely generated.
We show this by explicitly finding a finite generating set of $\iso(\integers^n)$.

Recall that (in Section~\ref{sec:isometries-of-zn}) we presented one of the characterizations of $\iso(\integers^n)$ as
\begin{align}\label{eqn:generator-decomposition-of-iso-zn}
\iso(\integers^n) = \langle \tra(\integers^n) \cup \rot(\integers^n) \rangle \quad \mbox{ where } \tra(\integers^n) \cap \rot(\integers^n) = \{\id\}.
\end{align}
We start from this characterization, and seek a generating set of $\tra(\integers^n)$ and a generating set of $\rot(\integers^n)$.
Finding generators of $\tra(\integers^n)$ is easy: $\tra(\integers^n) = \langle t'_{\bm{e_1}} \cup \cdots \cup t'_{\bm{e_n}} \rangle$.
However, finding generators of $\rot(\integers^n)$ requires more structural inspections.
The strategy is to first study the matrix group $\orthmat_n(\integers)$ which is isomorphic to $\rot(\integers^n)$, and then transfer results to $\rot(\integers^n)$.
Interestingly, $\orthmat_n(\integers)$ has a decomposition similar to what $\iso(\integers^n)$ has in Equation~\eqref{eqn:generator-decomposition-of-iso-zn}.
By definition, $\orthmat_n(\integers)$ consists of all orthogonal matrices with integer entries.
For any $A \in \orthmat_n(\integers)$, the orthogonality and integer-entry constraints restrict every column vector of $A$ to be a unique standard basis vector or its negation.
This will lead to the decomposition of $\rot(\integers^n)$.

\paragraph{Notations.}
$\bm{1} = (1,\ldots, 1) \in \reals^n$ is the all-ones vector;
$\bm{e_1}, \ldots, \bm{e_n}$ are the standard basis vectors of $\reals^n$ where $\bm{e_i} \in \{0,1\}^n$ has a $1$ in the $i$th coordinate and $0$s elsewhere;
$\bm{\nu_1}, \ldots, \bm{\nu_n}$ are the so-called unit negation vectors of $\reals^n$ where $\bm{\nu_i} \in \{-1,1\}^n$ has a $-1$ in the $i$th coordinate and $1$s elsewhere.

\begin{definition}[Permutation]\label{def:permutation}
A {permutation matrix} is a matrix obtained by permuting the rows of an identity matrix; we denote the set of all $n \times n$ permutation matrices by $\perm_n$.
A {permutation of an index set} is a bijection $\sigma: \{1, \ldots, n\} \to \{1, \ldots, n\}$; the set of all permutations of the size-$n$ index set is known as the symmetric group $S_n$.
A {permutation of (integer-valued) vectors} is a rotation $r'_P: \integers^n \to \integers^n$ for some $P \in \perm_n$; we denote the set of all permutations of $n$-dimensional vectors by $\rot_\perm(\integers^n) \subseteq \rot(\integers^n)$.
\end{definition}

\begin{definition}[Negation]\label{def:negation}
A {(partial) negation matrix} is a diagonal matrix whose diagonal entries are drawn from $\{-1,1\}$; we denote the set of all $n \times n$ negation matrices by $\nega_n$.
A {(partial) negation of (integer-valued) vectors} is a rotation $r'_N: \integers^n \to \integers^n$ for some $N \in \nega_n$; we denote the set of all negations of $n$-dimensional vectors by $\rot_\nega(\integers^n) \subseteq \rot(\integers^n)$.
\end{definition}

\begin{remark}
Under Definitions~\ref{def:permutation} and \ref{def:negation}, one can verify that a permutation (of vectors) maps $x$ to $Px$ by permuting $x$'s coordinates according to $P \in \perm_n$;
likewise, a negation (of vectors) maps $x$ to $Nx$ by negating $x$'s coordinates according to $N \in \nega_n$.
\end{remark}

\begin{theorem}\label{thm:permutation-and-negation-isomorphism}
We have the following characterizations of permutations and negations:
\begin{align*}
(\rot_\perm(\integers^n), \circ) \cong (\perm_n, \cdot) \cong (S_n, \circ) \quad \mbox{ and } \quad (\rot_\nega(\integers^n), \circ) \cong (\nega_n, \cdot).
\end{align*}
In particular, these imply that $|\rot_\perm(\integers^n)| = |\perm_n| = |S_n| = n!$ and $|\rot_\nega(\integers^n)| = |\nega_n| = 2^n$.
\end{theorem}

\begin{proof}
It is an exercise to check that all entities in the theorem are indeed groups.

Let $\phi_\perm: \rot_\perm(\integers^n) \to \perm_n$ be the function given by $\phi_\perm(r'_P) = P$, for any $r'_P \in \rot_\perm(\integers^n)$. For any $r'_P, r'_Q \in \rot_\perm(\integers^n)$, if $\phi_\perm(r'_P) = \phi_\perm(r'_Q)$, \ie $P = Q$, then $r'_P = r'_Q$, so $\phi_\perm$ is injective.
For any $P \in \perm_n$, $r'_P \in \rot_\perm(\integers^n)$ and $\phi_\perm(r'_P) = P$, so $\phi_\perm$ is surjective.
Further, for any $r'_P, r'_Q \in \rot_\perm(\integers^n)$, $\phi_\perm(r'_P \circ r'_Q) = \phi_\perm(r'_{P \cdot Q}) = P \cdot Q = \phi_\perm(r'_P) \cdot \phi_\perm(r'_Q)$, so $\phi_\perm$ is a homomorphism.
Now we see that $\phi_\perm$ is an isomorphism. So, $(\rot_\perm(\integers^n), \circ) \cong (\perm_n, \cdot)$.

Let $\phi_S: S_n \to \perm_n$ be the function given by $\sigma \mapsto P^\sigma$, where $P^\sigma$ is an $n \times n$ permutation matrix obtained by permuting the rows of the identity matrix according to $\sigma$, \ie
\begin{align*}
P^\sigma_{ij} = \begin{cases}
1 & i = \sigma(j) \\
0 & i \neq \sigma(j)
\end{cases}
\quad \mbox{ for any } i, j \in \{1, \ldots, n\}.
\end{align*}
For any $\sigma, \mu \in S_n$, if $\phi_S(\sigma) = \phi_S(\mu)$, \ie $P^\sigma = P^\mu$, then $\sigma(j) = \mu(j)$ for all $j \in \{1, \ldots, n\}$, \ie $\sigma = \mu$, so $\phi_S$ is injective.
For any $P \in \perm_n$, let $\sigma: \{1, \ldots, n\} \to \{1, \ldots, n\}$ be the function given by $\sigma(j) \in \{i|P_{ij} = 1\}$, which is well-defined since $\{i|P_{ij} = 1\}$ is a singleton for all $j \in \{1, \ldots, n\}$ given that $P \in \perm_n$ is a permutation matrix. It is clear that $\sigma \in S_n$, and $\phi_S(\sigma) = P$. So, $\phi_S$ is surjective.
Further, for any $\sigma, \mu \in S_n$, $\phi_S(\sigma \circ \mu) = P^{\sigma \circ \mu} = P^{\sigma} \cdot P^{\mu} = \phi_S(\sigma) \cdot \phi_S(\mu)$ where the second equality holds because for all $i,j \in \{1, \ldots, n\}$,
\begin{align*}
(P^{\sigma} \cdot P^{\mu})_{ij} = \sum_{k=1}^n P^\sigma_{ik}\cdot P^\mu_{kj} = P^\sigma_{i\mu(j)} \cdot 1 = \begin{cases}
1 & i = \sigma \circ \mu(j) \\
0 & i \neq \sigma \circ \mu(j)
\end{cases} = P^{\sigma \circ \mu}_{ij},
\end{align*}
so $\phi_S$ is a homomorphism.
Now we see that $\phi_S$ is an isomorphism. So, $(S_n, \circ) \cong (\perm_n, \cdot)$.

Let $\phi_\nega: \rot_\nega(\integers^n) \to \nega_n$ be the function given by $\phi_\nega(r'_N) = N$, for any $r'_N \in \rot_\nega(\integers^n)$. For any $r'_N, r'_M \in \rot_\nega(\integers^n)$, if $\phi_\nega(r'_N) = \phi_\nega(r'_M)$, \ie $N = M$, then $r'_N = r'_M$, so $\phi_\nega$ is injective.
For any $N \in \nega_n$, $r'_N \in \rot_\nega(\integers^n)$ and $\phi_\nega(r'_N) = N$, so $\phi_\nega$ is surjective.
Further, for any $r'_N, r'_M \in \rot_\nega(\integers^n)$, $\phi_\nega(r'_N \circ r'_M) = \phi_\nega(r'_{N \cdot M}) = N \cdot M = \phi_\nega(r'_N) \cdot \phi_\nega(r'_M)$, so $\phi_\nega$ is a homomorphism.
Now we see that $\phi_\nega$ is an isomorphism. So, $(\rot_\nega(\integers^n), \circ) \cong (\nega_n, \cdot)$.
\end{proof}

\begin{theorem}
We have the following characterization of $\orthmat_n(\integers)$:
\begin{align*}
\orthmat_n(\integers) =  \langle \nega_n \cup \perm_n \rangle \quad \mbox{ where } \nega_n \cap \perm_n = \{I\}.
\end{align*}
\end{theorem}

\begin{proof}
We first show that $\nega_n, \perm_n \leq \orthmat_n(\integers)$.
$(\orthmat_n(\integers), \cdot)$ is a group since matrix multiplication $\cdot$ is associative, $I \in \orthmat_n(\integers)$ is the identity element, and for any $A \in \orthmat_n(\integers)$, $A^\top \in \orthmat_n(\integers)$ is its inverse. Pick any $N \in \nega_n$, then $N \in \integers^{n\times n}$ and $N^\top N = NN = I$, so $N \in \orthmat_n(\integers)$, which implies that $\nega_n \subseteq \orthmat_n(\integers)$. Pick any $P \in \perm_n$, then $P = [\bm{e_{\sigma(1)}}, \cdots, \bm{e_{\sigma(n)}}] \in \integers^{n \times n}$ for some $\sigma \in S_n$ and $(P^\top P)_{ij} = \bm{e_{\sigma(i)}^\top}\bm{e_{\sigma(j)}} = \delta_{ij}$, \ie $P^\top P = I$, so $P \in \orthmat_n(\integers)$, which implies that $\perm_n \subseteq \orthmat_n(\integers)$.
Now we perform subgroup tests to show that $\nega_n, \perm_n \leq \orthmat_n(\integers)$. First, we check that 1) $I \in \nega_n$, 2) for any $N,N' \in \nega_n$, $NN' \in \nega_n$, 3) for any $N \in \nega_n$, $N^{-1} = N \in \nega_n$; therefore, $\nega_n \leq \orthmat_n(\integers)$. Second, we check that 1) $I \in \perm_n$, 2) for any $P,P' \in \perm_n$, $PP' \in \perm_n$, 3) for any $P \in \perm_n$, $P^{-1} = P^\top \in \perm_n$; therefore, $\perm_n \leq \orthmat_n(\integers)$.

Now we show that $\nega_n \cap \perm_n = \{I\}$. Pick any $N \in \nega_n \backslash \{I\}$ and any $P \in \perm_n$. It is clear that $N \neq P$ since $N$ has at least one $-1$ entries while $P$ has no $-1$ entries. This implies that $(\nega_n\backslash \{I\})\cap \perm_n = \emptyset$. Further, $I \in \nega_n \cap \perm_n$. Therefore, $\nega_n \cap \perm_n = \{I\}$.

Lastly we show that $\orthmat_n(\integers) =  \langle \nega_n \cup \perm_n \rangle$. It is clear that $\langle \nega_n \cup \perm_n \rangle \subseteq \orthmat_n(\integers)$, since $\nega_n, \perm_n \leq \orthmat_n(\integers)$.
Conversely, pick any $A = [\bm{a_1}, \cdots, \bm{a_n}] \in \orthmat_n(\integers)$ where $\bm{a_i}$ denotes the $i$th column of $A$. By definition, $A^\top A = I$, so $\langle \bm{a_i}, \bm{a_i} \rangle = \|\bm{a_i}\|_2^2 = 1$ for $i \in \{1, \ldots, n\}$, and $\langle \bm{a_i}, \bm{a_j} \rangle = 0$ for $i,j \in \{1, \ldots, n\}$ and $i \neq j$. On the one hand, given $A \in \integers^{n \times n}$, the unit-norm property $\|\bm{a_i}\|_2^2 = 1$ implies that $\bm{a_i}$ is a standard basis vector or its negation, \ie $\bm{a_i} = \pm\bm{e_k}$ for some $k \in \{1, \ldots, n\}$. On the other hand, for $i \neq j$, the orthogonality property $\langle \bm{a_i}, \bm{a_j} \rangle = 0$ implies that for some $k \neq k'$, $\bm{a_i} = \pm\bm{e_k}$ and $\bm{a_j} = \pm\bm{e_{k'}}$. Thus, there exist some vector $\alpha = (\alpha_1, \ldots, \alpha_n) \in \{1,-1\}^n$ and some permutation $\sigma \in S_n$ such that $A = [\alpha_1\bm{e_{\sigma(1)}}, \cdots, \alpha_n\bm{e_{\sigma(n)}}] = \diag(\alpha)[\bm{e_{\sigma(1)}}, \cdots, \bm{e_{\sigma(n)}}] = NP$, where $N = \diag(\alpha) \in \nega_n$ and $P = [\bm{e_{\sigma(1)}}, \cdots, \bm{e_{\sigma(n)}}] \in \perm_n$. This implies $A \in \langle \nega_n \cup \perm_n \rangle$. So, $\orthmat_n(\integers) \subseteq \langle \nega_n \cup \perm_n \rangle$.
\end{proof}

\begin{corollary}\label{coro:generator-decomposition-of-rot-zn}
We have the following characterization of $\rot(\integers^n)$:
\begin{align}\label{eqn:generator-decomposition-of-rot-zn}
\rot(\integers^n) = \langle \rot_\nega(\integers^n) \cup \rot_\perm(\integers^n) \quad \mbox{ where } \rot_\nega(\integers^n) \cap \rot_\perm(\integers^n) = \{\id\}.
\end{align}
\end{corollary}

\begin{remark}
The decomposition of the rotation group $\rot(\integers^n)$ in Equation~\eqref{eqn:generator-decomposition-of-rot-zn} has a similar form compared to the decomposition of the isometry group $\iso(\integers^n)$ in Equation~\eqref{eqn:generator-decomposition-of-iso-zn}.
One can show that $\rot(\integers^n)$ has a second characterization that is similar to the second characterization of $\iso(\integers^n)$, where $\rot(\integers^n)$ can also be decomposed as semi-direct products:
\begin{align*}
\rot(\integers^n) = \rot_\nega(\integers^n) \circ \rot_\perm(\integers^n) \cong \nega_n \rtimes \perm_n.
\end{align*}
However, this characterization is not used in this paper, so we omit its proof.
\end{remark}
By Corollary~\ref{coro:generator-decomposition-of-rot-zn}, finding generators of $\rot(\integers^n)$ breaks down into finding those of $\rot_\nega(\integers^n)$ and $\rot_\perm(\integers^n)$, respectively.
First, from unit negation vectors, we can find generators for negations:
\begin{align*}
\rot_\nega(\integers^n) = \langle \{r'_{\diag(\bm{\nu_1})}, \ldots, r'_{\diag(\bm{\nu_n})}\} \rangle.
\end{align*}
Second, from the fact that the symmetric group is generated by $2$-cycles of the form $(i,i+1)$: $S_n = \langle \{(1,2), \ldots, (n-1,n) \} \rangle$ \citep{Conrad2016}, we can find generators for permutations:
\begin{align*}
\rot_\perm(\integers^n) = \langle \{ r'_{P^{(1,2)}}, \ldots, r'_{P^{(n-1,n)}} \} \rangle,
\end{align*}
where $P^{(i,i+1)} \in \perm_n$ is obtained by swapping the $i$th and $(i+1)$th rows of $I$.
Finally,
\begin{align}
\notag
\iso(\integers^n) &= \left\langle \tra(\integers^n) \cup \rot(\integers^n) \right\rangle \\
\notag
&= \left\langle \tra(\integers^n) \cup \rot_\nega(\integers^n) \cup \rot_\perm(\integers^n)\right\rangle \\
\label{eqn:final-decomp-of-iso-zn}
&= \left\langle \tra_0 \cup \rot_{\nega 0} \cup \rot_{\perm 0} \right\rangle,
\end{align}
where $\tra_0 := \{t'_{\bm{e_i}}\}_{i=1}^n$, $\rot_{\nega 0} := \{r'_{\diag(\bm{\nu_i})}\}_{i=1}^n$, and $\rot_{\perm 0} := \{r'_{P^{(i,i+1)}}\}_{i=1}^{n-1}$.
Here, we performed recursive group decompositions to yield the generating set $\tra_0 \cup \rot_{\nega 0} \cup \rot_{\perm 0}$ with finite size $n+n+(n-1) = 3n-1$.
This verifies that $\iso(\integers^n)$ is indeed finitely generated.

\subsection{Trade-off: Minimality or Diversity (Efficiency or Expressiveness)}
\label{sec:trade-off-minimality-or-diversity}

A generating set of a group is not unique.
There are two extremes when considering the size of a generating set.
One considers the \emph{largest} generating set of a group which is the group itself;
the other considers a \emph{minimal} generating set which is not unique either.

\begin{definition}\label{def:minimal-generating-set}
Let $G$ be a group, $S \subseteq G$, and $\langle S \rangle$ be the subgroup of $G$ generated by $S$. We say that $S$ is a minimal generating set (of $\langle S \rangle$) if for any $s \in S$, $\langle S \backslash \{s\} \rangle \neq \langle S \rangle$.
\end{definition}

\begin{theorem}\label{thm:minimal-generating-set}
Let $G$ be a group and $S \subseteq G$, then $S$ is a minimal generating set if and only if for any $s \in S$, $s \notin \langle S \backslash \{s\} \rangle$.
\end{theorem}
\begin{proof}
See Appendix~\ref{app:minimal-generating-set}.
\end{proof}

Considering $\iso(\integers^n) = \langle \tra_0 \cup \rot_{\nega 0} \cup \rot_{\perm 0} \rangle$, it is easy to check that $\tra_0$, $\rot_{\nega 0}$, and $\rot_{\perm 0}$ are minimal individually;
whereas their union is not.
Nevertheless, it is not hard to show that $S^\star := \{t'_{\bm{e_1}}, r'_{\diag(\bm{\nu_1})}\} \cup \rot_{\perm 0}$ is a minimal generating set of $\iso(\integers^n)$ with size $n+1$.

\vspace{0.1in}
There is a trade-off between minimality and diversity, which further leads to the trade-off between efficiency and expressiveness.
Again we use $\iso(\integers^n)$ as an example.
From one extreme, a minimal generating set is most efficient in the following sense: $S\subseteq \iso(\integers^n)$ is a minimal generating set (of $\langle S \rangle$) if and only if $\pi'|_{2^S}$ is injective, \ie for any $S',S'' \subseteq S$, if $S' \neq S''$, then $\pi'(S') = \langle S' \rangle \neq \langle S'' \rangle = \pi'(S'')$ (an easy check).
Therefore, whenever $S$ is not minimal, there are duplicates in the generated subgroups, and thus duplicates in the subsequent abstraction generations.
Every occurrence of a duplicate is wasted computation since it does not produce a new abstraction in the end.
Intuitively, if a generating set is further away from being minimal, then more duplicates tend to occur and the abstraction generating process is less efficient.
To the other extreme, the largest generating set is most expressive in the following sense: if $S = \iso(\integers^n)$, then $\pi'(2^S) = \hcal_{\iso(\integers^n)}^{*}$, \ie the collection of all subgroups of $\iso(\integers^n)$; and in general, for any $S \subset S_+ \subseteq \iso(\integers^n)$, the monotonicity property $\pi'(2^S) \subset \pi'(2^{S_{+}})$ holds (an easy check).
However, the largest generating set is also the least efficient not only because it has the largest number of duplicates, but in this case, it is infinite.
Thus, to respect the trade-off between efficiency and expressiveness, we need to find a balance between the two extremes.

Our plan is to start from a minimal generating set $S^\star$ of $\iso(\integers^n)$ and then gradually enlarge it by adding the so-called \emph{derived generators}.
In other words, we aim for a filtration: $S^\star \subseteq S^\star_{+1} \subseteq S^\star_{+2} \subseteq S^\star_{+3} \subseteq \cdots$  such that the corresponding collections of subgroups satisfy
\begin{align*}
\pi'(2^{S^\star}) \subseteq \pi'(2^{S_{+1}^\star}) \subseteq \pi'(2^{S_{+2}^\star}) \subseteq \pi'(2^{S_{+3}^\star})\subseteq \cdots \quad \mbox{ and } \quad \bigcup_{m=1}^\infty \pi'(2^{S_{+m}^\star}) = \hcal_{\iso(\integers^n)}^{*}.
\end{align*}

\begin{definition}\label{def:derived-generators}
Let $S^\star$ be a minimal generating set of $\iso(\integers^n)$, and define
\begin{align*}
S^\star_{+m} := \{s_k^{\alpha_k}\circ \cdots \circ s_1^{\alpha_1} \mid k \in \integers_{\geq 0}, s_k, \ldots, s_1 \in S^\star, \alpha_k, \ldots, \alpha_1 \in \integers, \textstyle\sum_{i=1}^k|\alpha_i| \leq m\}.
\end{align*}
A {derived generator} of length $m$ is an $s \in S^\star_{+m}\backslash S^\star_{+(m-1)}$.
\end{definition}
\begin{remark}
In Definition~\ref{def:derived-generators}, $S^\star_{+m}$ is  the ``ball'' with center $\id$ and radius $m$ in the Cayley graph of $S^\star \cup (S^\star)^{-1}$.
It is an easy check that $S^\star \cup (S^\star)^{-1} = S^\star_{+1} \subseteq S^\star_{+2} \subseteq S^\star_{+3} \subseteq \cdots$. 
\end{remark}

Note that $\cup_{m=1}^{\infty}S^\star_{+m} = \langle S^\star \rangle = \iso(\integers^n)$, since the growing ``ball'' will eventually cover the whole Cayley graph.
Therefore, $\cup_{m=1}^\infty \pi'(2^{S_{+m}^\star}) = \hcal_{\iso(\integers^n)}^{*}$.
This suggests we gradually add derived generators of increasing length to $S^\star$, and approximate $\hcal_{\iso(\integers^n)}^{*}$ by $\pi'(2^{S_{+m}^\star})$ for some large $m$.
Without any prior preference, one must go through this full procedure to grow the ball $S^\star_{+m}$ from radius $m = 1$.
Although computationally intense, it is incremental.
More importantly, this is a \emph{one-time} procedure, but the resulting abstraction (semi)universe is \emph{universal}: computed abstractions can be used in different topic domains.

However, just like biological perception systems which have innate preference for certain stimuli, having prior preference for certain derived generators can make the abstraction (semi)universe grow more efficiently.
As an illustrative example and a design choice, we start from the minimal generating set $S^\star := \{t'_{\bm{e_1}}, r'_{\diag(\bm{\nu_1})}\} \cup \rot_{\perm 0}$, and prioritize three types of derived generators.
First, we add back the basis generators $\{t'_{\bm{e_i}}\}_{i=2}^n$ and $\{r'_{\diag(\bm{\nu_i})}\}_{i=2}^n$, so we get back the generating set $\tra_0 \cup \rot_{\nega 0} \cup \rot_{\perm 0}$ from Equation~\eqref{eqn:final-decomp-of-iso-zn}.
This restores the complete sets of translations, negations, and permutations---the three independent pillars generating $\iso(\integers^n)$.
The other two types of derived generators, called \emph{circulators} and $\emph{synchronizers}$, are inspired by biologically innate preference for \emph{periodicity} and \emph{synchronization} \citep{VanrullenZI2014,VonPD1992,Tymoczko2010}.

\begin{definition}\label{def:circulator-and-synchronizer}
Let $S = \{s_1, \ldots, s_k\}$ be a minimal generating set.
A circulator of $S$ with period $\alpha$ is: $s^\alpha$ for some $s \in S$ and $\alpha \in \integers_{>0}$.
(Consider group action on $X$ and any $x \in X$: the orbit $\langle s^\alpha \rangle x$ consists of periodic points from $\langle s \rangle x$.)
If $\langle S \rangle$ is Abelian, the synchronizer of $S$ is: $s_k\circ \cdots \circ s_1$.
\end{definition}

We denote the set of all circulators of $\tra_0$ with a fixed period $\alpha$ by $\tra_0^{\alpha} := \{t'_{\alpha{\bm{e_i}}}\}_{i=1}^n$.
Inspecting circulators of $\rot_{\nega 0}$ and $\rot_{\perm 0}$ does not yield new generators, since for any $s \in \rot_{\nega 0} \cup \rot_{\perm 0}$, $s^2 = \id$.
The synchronizers of $\tra_0$ and $\rot_{\nega 0}$ are $t'_{\bm{1}}$ and $r'_{-I}$ ($\langle \rot_{\perm 0} \rangle$ is not Abelian). Adding these circulators and synchronizers to $\tra_0 \cup \rot_{\nega 0} \cup \rot_{\perm 0}$ yields the generating set:
\begin{align}\label{eqn:final-s-star-plus}
S^\star_+ := \tra_0 \cup \tra_0^2 \cup \cdots \cup \tra_0^\tau \cup \{t'_{\bm{1}}\} \cup \rot_{\nega 0} \cup \{r'_{-I}\} \cup \rot_{\perm 0},
\end{align}
where $\tau$ denotes an upper bound on period exploration.
Note that $|S^\star_+| = \tau n + 1 + n + 1 + (n - 1) = (\tau+2)n + 1$. In light of real applications, we can use this generating set to generate an abstraction semiuniverse for automatic music concept learning.


\section{Implementing the General Setting: Restriction to Finite Subspaces}
\label{sec:restrict-to-finite-subspaces}

Computers have to work with finite spaces for finite execution time.
If the underlying space $X$ is finite, then there is no issue, and in particular, our induction algorithm in Section~\ref{sec:an-induction-algorithm} can solve any instances of Problem~\eqref{eqn:induction-algo-problem} exactly, regarding any finitely generated subgroup acting on any finite set.
If $X$ is infinite (\eg $\integers^n$), we can still mathematically think of any abstraction of $X$.
However, to compute an abstraction (or just to present it to people) in practice, we have to consider a finite subspace of $X$, \eg Problem~\eqref{eqn:induction-algo-problem-general}.
In this generating setting, we must be careful about both what an abstraction of a subspace means and what potential problems might occur.

\begin{definition}\label{def:an-abstraction-of-a-subspace}
Let $X$ be a set and $\pcal$ be an abstraction of $X$. For any $Y \subseteq X$, the restriction of $\pcal$ to $Y$ is an abstraction of $Y$ given by $\pcal|_Y := \{P \cap Y \mid P \in \pcal\} \backslash \{\emptyset\}$.
\end{definition}
\begin{remark}
Unless otherwise stated, the term ``an abstraction of a subspace'' means an abstraction of the ambient space restricted to that subspace. Under this definition, we need extra caution when computing an abstraction of a subspace.
\end{remark}

Let $X$ be a set, and $H \leq \transf(X)$ be a subgroup of the transformation group of $X$.
For any $Y \subseteq X$, according to Definition~\ref{def:an-abstraction-of-a-subspace}, the correct way of generating the abstraction of the subspace $Y$ from $H$ is:
\begin{align*}
\pi(H)|_Y = (X/H)|_Y = \{Hx \cap Y \mid x \in X\} \backslash \{\emptyset\}.
\end{align*}
A risky way of computing the abstraction of the subspace $Y$ is by thinking only in $Y$ while forgetting the ambient space $X$.
The risk here is to get a partition of $Y$, denoted $\rcal_Y^H$, which is strictly finer than $\pi(H)|_Y = \rcal_X^H|_Y$.
In other words, there are possibly cells in $\rcal_Y^H$ that should be merged but are not if they are connected via points outside the subspace $Y$.
For instance, consider $X = \integers^2$, $Y = \{(0,0),(1,0),(0,1),(1,1)\} \subseteq X$, and the subgroup $H = \langle \{ t'_{\bm{1}}, r'_{-I} \} \rangle \leq \iso(\integers^2) \leq \transf(\integers^2)$.
Let $\rcal_Y^H$ be the abstraction of $Y$ obtained by running the induction algorithm on $Y$ (instead of $X$) in the bottom-up approach.
One can check:
\begin{align*}
\rcal_Y^H &= \{~\{(0,0),(1,1)\}, ~\{(0,1)\}, ~\{(1,0)\}~\}; \\
\pi(H)|_Y &= \{~\{(0,0),(1,1)\}, ~\{(0,1),(1,0)\}~\}.
\end{align*}
The two points $(1,0)$ and $(0,1)$ should be in one cell since $(1,0) \xmapsto{r'_{-I}} (-1,0) \xmapsto{t'_{\bm{1}}} (0,1)$, but are not in $\rcal_Y^H$ since the via-point $(-1,0) \not\in Y$.
In general, the risk is present if we compute an abstraction of a subspace $Y$ from other abstractions of $Y$ or from orbit tracing.

However, for computational reasons, we want to forget the ambient space $X$!
In particular, the risky way is the only practical way if $X$ is infinite and $Y$ is finite, since it is not realistic to identify all orbits in an infinite space.
This suggests that we take the risk to generate $\rcal_Y^H$ as the first step, and rectify the result in a second step to merge cells that are missed in the first step. As a result, we introduce a technique called ``expand-and-restrict''.

\subsection{Expand-and-Restrict}
\label{sec:expand-and-restrict}

``Expand-and-restrict'' is an empirical technique which first expands the subspace and then restricts it back, \ie to compute $\rcal_{Y_+}^H|_Y$ for some finite subspace $Y_+$ such that $Y \subset Y_+ \subset X$.
The expansion $Y_+$ takes more via-points into consideration, so it helps merge cells that are missed in $\rcal_Y^H$.
In practice, we carry out this technique gradually in a sequential manner, which is similar to what we did in enlarging a minimal generating set (cf.\ Section~\ref{sec:trade-off-minimality-or-diversity}).
Given an infinite space $X$ and a finite subspace $Y \subset X$, we first construct a filtration
\begin{align*}
Y = Y_{+0} \subset Y_{+1} \subset Y_{+2} \subset \cdots \subset X \quad \mbox{ where $Y_{+k}$ is finite $\forall k \in \integers_{\geq 0}$ and } \bigcup_{k=0}^\infty Y_{+k} = X.
\end{align*}
We then start a search process for a good expansion $Y_{+k}$.
More specifically, we iteratively compute $\rcal_{Y_{+k}}^H|_Y$ for expansion factors $k=0, 1, 2, \ldots$ until the results reach a consensus among consecutive iterations.
To determine a consensus, theoretically, we need to find the smallest $k$ such that $\rcal_{Y_{+k}}^H|_Y = \rcal_{Y_{+k'}}^H|_Y$ for all $k' > k$, which requires an endless search.
In practice, we stop the search whenever $\rcal_{Y_{+k}}^H|_Y = \rcal_{Y_{+(k+1)}}^H|_Y = \cdots = \rcal_{Y_{+(k+\Delta k)}}^H|_Y$ for some positive integer $\Delta k$.
We call this an \emph{early stop}, whose resulting abstraction $\rcal_{Y_{+k}}^H|_Y$ is an empirical approximation of the true abstraction $\pi(H)|_Y$.
Note that without early stopping, we will have the correct result $\pi(H)|_Y = \rcal_X^H|_Y$ in the limit of this infinite search process.
Therefore, even in cases where the space $X$ is finite, if $X$ is much larger than the subspace $Y$, this empirical search can be more efficient than computing $\pi(H)|_Y$ directly, since earlier search iterations will be extremely cheap and if an early stop happens early there is a win.

\subsection{An Implementation Example}
\label{sec:an-implementation-example}

We give an example to illustrate some implementation details on generating abstractions of a finite subspace.
In this example, we consider finite subspaces of $X = \integers^n$ to be the centered hypercubes of the form $Y = \hcube{-b}{b}$ where $\integers_{[-b,b]} := \integers \cap [-b,b]$ and $b>0$ is finite.
To construct an abstraction semiuniverse for such a finite hypercube, we adopt the bottom-up approach and pick the generating set to be $S^\star_+$ defined in Equation~\eqref{eqn:final-s-star-plus}.
Taking $S^\star_+$ and $\hcube{-b}{b}$ as inputs, we run the induction algorithm, where both Algorithm~\ref{algo:base-partn} (for base cases) and Algorithm~\ref{algo:meet} (for the induction step) are run on the finite subspace $\hcube{-b}{b}$ instead of the infinite space $\integers^n$.
This is the first step which gives abstractions $\rcal_\hcube{-b}{b}^{\langle S \rangle}$ for $S \subseteq S^\star_+$.

As mentioned earlier, for every $S \subseteq S^\star_+$, the correct abstraction should be $\rcal_{\integers^n}^{\langle S \rangle}|_{\hcube{-b}{b}}$ which is generally not equal to $\rcal_\hcube{-b}{b}^{\langle S \rangle}$.
So, we run the ``expand-and-restrict'' technique as the second step.
We first construct a filtration: let $Y_{+k} = \hcube{-b-k}{b+k}$ be a finite expansion of $Y = \hcube{-b}{b}$, then it is clear that $Y = Y_{+0} \subset Y_{+1} \subset Y_{+2} \subset \cdots \subset X$ and $\cup_{k=0}^\infty Y_{+k} = X = \integers^n$.
We then start the empirical search process and set $\Delta k = 1$ (the most greedy search).
This means we will stop the search whenever $\rcal_{Y_{+k}}^{\langle S \rangle}|_Y = \rcal_{Y_{+(k+1)}}^{\langle S \rangle}|_Y$, and return the abstraction $\rcal_{Y_{+k}}^{\langle S \rangle}|_Y = \rcal_{\hcube{-b-k}{b+k}}^{\langle S \rangle}|_{\hcube{-b}{b}}$ as the final result to approximate $\rcal_{\integers^n}^{\langle S \rangle}|_{\hcube{-b}{b}} = \Pi(S)|_{\hcube{-b}{b}}$.

There are three additional implementation tricks that are special to this example.
The first trick applies to cases where the subspace $\hcube{-b}{b}$ is large, \ie a large $b$.
In this case, every search iteration in the ``expand-and-restrict'' technique is expensive and gets more expensive as the search goes.
However, for the generating set $S^\star_+$ specifically, it is typical to have $b \gg \tau$ so as to reveal strong periodic patterns.
Thus, we run the entire two-step abstraction generating process for $\hcube{-\tau}{\tau}$ instead of $\hcube{-b}{b}$, pretending $\hcube{-\tau}{\tau}$ is the subspace that we want to abstract.
This yields a much faster abstraction process since $\hcube{-\tau}{\tau}$ is much smaller than $\hcube{-b}{b}$.
The result is an abstraction $\rcal_{\hcube{-\tau-k}{\tau+k}}^{\langle S \rangle}|_{\hcube{-\tau}{\tau}}$ for some expansion factor $k$.
We reuse this same $k$ and compute $\rcal_{\hcube{-b-k}{b+k}}^{\langle S \rangle}|_{\hcube{-b}{b}}$ as the final result, which is the only expensive computation.
Note that this trick adds an additional empirical approximation, assuming that the same expansion factor $k$ works for both small and large subspaces.
While we have not yet found a theoretical guarantee for this assumption, this trick works well in practice, and provides huge computational savings.

\underline{Note:} for some generating subsets $S \subseteq S^\star_+$, we can prove (so no approximations) that the expansion factor $k = 0$ (no need to expand) or $1$.
Although this provides theoretical guarantees in certain cases, the tricks used in the current proofs are case-by-case depending on the chosen generators.
Thus, before we find a universal way of proving things, we prefer empirical strategies---like the above search process---which work universally in any event.

The second trick considers the subspace to be any general hypercube in $\integers^n$, which is not necessarily square or centered.
The trick here is simply to find a minimum centered square hypercube containing the subspace.
If the ambient space $X$ happens to be ``spatially stationary''---the absolute location of each element in the space is not important but only their relative position matters (\eg the space of music pitches)---then we find a minimum square hypercube containing the subspace and center it via a translation.
Centering is very important and specific to the chosen generating set $S^\star_+$.
This is because $S^\star_+$ contains only pure translations and pure rotations;
and centering square hypercubes makes pure rotations safe: no rotation maps a point in $\hcube{-b}{b}$ outside (one can check that for any $r'_A \in \rot(\integers^n)$, $r'_A(\hcube{-b}{b}) = \hcube{-b}{b}$).
In practice, centering dramatically decreases the number of miss-merged cells, and makes it safe to choose small $\Delta k$ for early stopping.
This explains why we only consider subspaces of the form $\hcube{-b}{b}$ in the first place, and boldly choose $\Delta k = 1$.

The third trick considers a quick-and-dirty pruning of duplicates in generating a family of partitions, leaving room for larger-period explorations.
Without this trick, to generate the partition family $\Pi(2^{S^\star_+})$, we need $|2^{S^\star_+}| = 2^{(\tau+2)n+1} = O(2^\tau)$ computations, which hinders exploration on period $\tau$.
However, $S^\star_+$ is not minimal, so $|\Pi(2^{S^\star_+})| < |2^{S^\star_+}|$, suggesting many computations are not needed since they yield the same abstraction.
We focus on circulators, where we exclude computations on those $S \in 2^{S^\star_+}$ containing multiple periods.
This reduces the number of computations to $(2^{3n+1}-2^{2n+1})\tau+2^{2n+1} = O(\tau)$.

A real run on the subspace $\integers^4_{[-12,12]}$ and $\tau = 4$ computes $31232$ partitions, during which all search processes in ``expand-and-restrict'' end in at most three iterations.
This means in this experiment we only need to expand the subspace by $k = 0$ or $1$ for all abstractions in the family.

Lastly, we briefly mention the task of completing a global hierarchy on an abstraction family $\mathfrak{P}_Y$.
A brute-force algorithm makes $O(|\mathfrak{P}_Y|^2)$ comparisons, determining the relation ($\preceq$ or incomparable) for every unordered pair of partitions $\pcal,\qcal \in \mathfrak{P}_Y$.
Locally, we run a subroutine {\tt GetRelation(P,Q)} implemented via the contingency table \citep{HubertA1985} whenever we want to query a pair of partitions.
Globally, we use two properties to reduce the number of calls to {\tt GetRelation(P,Q)}: 1) transitivity: for any $\pcal, \pcal', \pcal'' \in \mathfrak{P}_Y$, $\pcal \preceq \pcal'$ and $\pcal' \preceq \pcal''$ implies $\pcal \preceq \pcal''$; 2) dualities in Theorem~\ref{thm:duality-subgroup-and-partition-lattice}.
The final output of our abstraction process is a directed acyclic graph of the abstraction (semi)universe $\Pi(2^{S^\star_+})$.
Similar to the first trick above, in practice it suffices to complete the hierarchy for smaller subspaces like $\hcube{-\tau}{\tau}$, assuming the same hierarchy holds for the actual subspace under consideration.


\section{Discussion: Information Lattice and Learning}
\label{sec:discussion-to-info-lattice-and-learning}

In his 1953 work, Claude E. Shannon attempted to unravel the nature of information beyond just quantifying its amount \citep{Shannon1953}.
In the specific context of communication problems, he coined the term \emph{information element} to denote the nature of information which is invariant under ``(language) translations'' or different encoding-decoding schemes.
He further introduced a partial order between a pair of information elements, eventually yielding a lattice of information elements, or \emph{information lattice} in short.

In this section, we first present a brief overview of Shannon's original work and then cast the information lattice in our abstraction-generation framework without needing to introduce information-theoretic functionals.
Our abstraction-generation framework not only generalizes Shannon's information lattice, but more importantly presents a generating chain that brings learning into the picture.
This eventually opens the opportunity for data-driven concept learning which aims to discover human-interpretable rules from sensory data.
After the theoretical connections, we present a real implementation of an information lattice from a music application.
In this application, we build an automatic music theorist and pedagogue that self-learns music composition rules and provides people with personalized music composition lessons \citep{YuVGK2016,YuV2017,YuLV2017}.

\subsection{Theoretical Generalization: a Separation of Clustering and Statistics}
\label{sec:theoretic-generalization}

We present an overview of Shannon's original work and a follow-up work by \citet{LiC2011} which formalizes Shannon's idea in a more principled way.

\paragraph{Nature of information:}
\begin{center}
information element \quad or \quad random variable?
\end{center}
An \emph{information element} is an equivalence class of random variables (of a common sample space) with respect to the ``being-informationally-equivalent'' relation, where two random variables are informationally equivalent if they induce the same $\sigma$-algebra (of the sample space).
Under this definition, the notion of an information element---essentially a probability space---is more abstract than that of a random variable: an information element can be realized by different random variables.
The relationship between different but informationally-equivalent random variables and their corresponding information element is analogous to the relationship between different translations (say, English, French, or a code) of a message and the actual content of the message.
Since different but faithful translations are viewed as different ways of describing the same information, the information itself is then regarded as the equivalence class of all translations or ways of describing the same information.
Therefore, the notion of information element reveals the nature of information.

\paragraph{Group-theoretic interpretation:}
\begin{align}\label{eqn:chain-info-partition-subgroup}
\mbox{information lattice $\rightarrow$ partition lattice $\rightarrow$ subgroup lattice ($\rightarrow$ interpretation).}
\end{align}
An \emph{information lattice} is a lattice of information elements, where the partial order is defined by $x \leq y \iff H(x|y) = 0$ where $H$ denotes the information entropy.
The join of two information elements $x\vee y = x+y$ is called the total information of $x$ and $y$;
the meet of two information elements $x \wedge y = xy$ is called the common information of $x$ and $y$.
By definition, every information element can be uniquely determined by its induced $\sigma$-algebra.
Also, it is known that every $\sigma$-algebra of a countable sample space can be uniquely determined by its generating (via union operation) sample-space-partition.
Thus, an information lattice has a one-to-one correspondence to a partition lattice.
Further, given a partition of a sample space, \citet{LiC2011} constructed a unique permutation subgroup whose group action on the sample space produces orbits that coincide with the given partition.
Therefore, under this specific construction (which also validates our Theorem~\ref{thm:abstraction-generating-function-surjective}), any partition lattice has a one-to-one correspondence to the constructed subgroup lattice \citep[see][General Isomorphism Theorem]{LiC2011}.
This yields the above Chain~\eqref{eqn:chain-info-partition-subgroup} which further achieves group-theoretic interpretations of various information-theoretic results, bringing together information theory and group theory \citep[chap.~16]{ChanY2002, Yeung2008}.
\vspace{0.1in}

Now we cast the above results into our framework, and reveal the key differences.

\paragraph{Nature of abstraction:}
\begin{center}
clustering \quad or \quad classification?
\end{center}
Generalizing Shannon's insight on the nature of information essentially reveals the difference between clustering and classification in machine learning.
We can similarly define an equivalence relation on the set of all classifications where two classifications are equivalent if they yield the same set of classes and only differ by class labels.
For instance, given a set of animals, classifying them into \{fish, amphibians, reptiles, birds, mammals\} is equivalent to classifying them into \{poisson, amphibiens, reptiles, des oiseaux, mammif\`{e}res\}, where the different class labels are only English and French translations of the same set of animal classes.
So, clustering to classification is analogous to information element to random variable; and it is clustering rather than classification that captures the nature of abstraction.
This again explains why abstraction is formalized as a clustering problem in this paper.

\begin{table}[t]
\centering
\begin{tabular}{l|ll}
    & {\small{\textit{Partition lattice}}} & {\small \textit{Information lattice}} \\
    \hline
    element & partition $(\pcal)$; & information element $(x)$; \\
     & clustering $(X,\pcal)$; & probability space $(X, \Sigma, P)$; \\
     & equiv. class of classifications & equiv. class of random variables \\
    partial order&  $\pcal \preceq \qcal$ & $x\leq y \iff H(x|y) = 0$ \\
    join & $\pcal \vee \qcal$ & $x+y$ \\
    meet & $\pcal \wedge \qcal$ & $xy$ \\
    metric & undefined & $\rho(x,y) = H(x|y) + H(y|x)$ \vspace{0.1in}
  \end{tabular}
\caption{Partition lattice and information lattice: the main difference comes from the fact that a partition lattice is not coupled with a measure; whereas an information lattice is coupled with a probability measure, so both the partial order and the metric can be defined in terms of entropies.}
\label{tab:partn-lattice-and-info-lattice}
\end{table}

\paragraph{Partition lattice and information lattice:}
We summarize major connections between a partition lattice and an information lattice in Table~\ref{tab:partn-lattice-and-info-lattice}.
The differences are rooted in the \emph{separation of clustering from statistics}, so roughly speaking, a partition lattice---which is measure-free---can be thought as an information lattice without probability measure.
Therefore, abstraction is a more general concept than information, which is not specific to communication problems, and in particular, is not attached to any stochastic processes or information-theoretic functionals such as entropy.

\paragraph{Group-theoretic learning:}
\begin{align}\label{eqn:chain-subgroup-partition-info}
\mbox{subgroup lattice $\rightarrow$ partition lattice $\rightarrow$ information lattice ($\rightarrow$ learning)}.
\end{align}
The separation of clustering and statistics is important since it opens the opportunity for \emph{interpretable statistical learning}, where \emph{interpretability} is achieved by the explicit construction of a partition lattice (symmetry-generated hierarchical clustering), and \emph{learning} is achieved by subsequent statistical inference on this lattice.
This is more precisely presented in Chain~\eqref{eqn:chain-subgroup-partition-info} aiming for learning, which at first glance, is merely a reverse process of Chain~\eqref{eqn:chain-info-partition-subgroup} aiming for re-interpretation.
However, the subgroup lattices in both chains are in stark contrast: the subgroups considered in Chain~\eqref{eqn:chain-subgroup-partition-info} are based on certain symmetries---the underlying mechanism of abstraction---whereas the subgroups considered in Chain~\eqref{eqn:chain-info-partition-subgroup} are merely (isomorphic) re-statements of the given partitions.
In other words, among possibly many subgroups (recall Theorem~\ref{thm:abstraction-generating-function-not-injective} and \ref{thm:abstraction-generating-function-surjective}: $\pi$ is surjective but not necessarily injective) that generate the same partition, we only pick the one that explains to us the types of symmetries under consideration.
The preservation of interpretable symmetries through Chain~\eqref{eqn:chain-subgroup-partition-info} makes the subsequent learning transparent.
Therefore, when abstraction does meet statistics, it will yield interpretable machine learning and knowledge discovery, which is beyond simply a re-interpretation of known results.

\subsection{A Real Application: Automatic Concept Learning in Music}
\label{sec:real-music-application}

We present a music application called MUS-ROVER from our earlier and ongoing work \citep{YuVGK2016,YuV2017,YuLV2017}, from which we show how automatic concept learning is achieved in a real information lattice.
In MUS-ROVER, each learned music concept is embedded in a \emph{probabilistic rule} (for music composition), which takes the form of an abstraction of music chords and its attaching probabilistic pattern.
Therefore, every rule is a music realization of an information element, and rules are extracted during the learning process in a symmetry-generated information lattice.

\paragraph{Computational Music Abstraction.}
Our group-theoretic approach to computational abstraction in this paper is concerned with the non-statistical part of an information lattice, \ie a symmetry-driven partition lattice consisting of various music abstractions and their hierarchy.
It is important to mention that these music abstractions are not borrowed from existing music definitions, but purely generated from symmetries.
However, if the output partition lattice is expressive enough, many abstractions in the lattice may coincide with human-invented music-theoretic terms, while the rest may suggest new abstraction possibilities that music theorists might have overlooked in the past.
Next, we walk through the main results in this paper in the specific context of computational music abstraction.

\begin{figure}[t]
\begin{center}
\includegraphics[width=0.92\columnwidth]{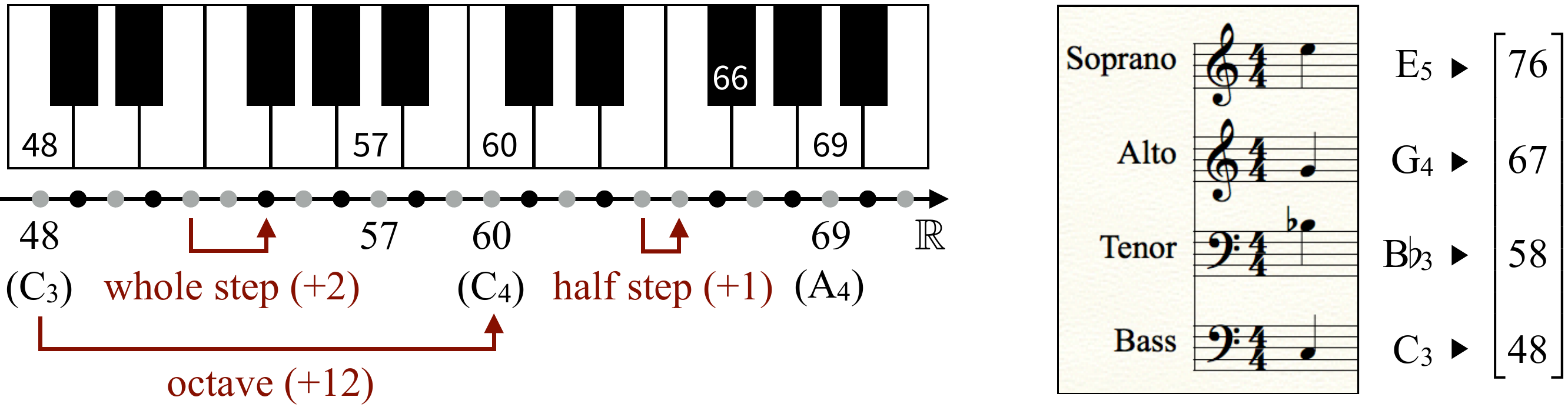}
\end{center}
\caption{The MIDI representation of music pitches (left) and a music chord (right).}
\label{fig:music-pitch-chord}
\end{figure}

\vspace{0.1in}
\noindent
\underline{Music Raw Representation.}
We represent a music \emph{pitch}---the highness of a music note---by a MIDI number $m$, which indicates the fundamental frequency $f$ of that pitch but in the following logarithmic scale: $m = 69+12\log_2(f/440)$.
As a frequency can take any positive value, MIDI numbers are generally real numbers.
Yet in our music application, we work with sheet music, and for notes on a music staff, their MIDI numbers are always integers.
For instance, the middle C (C$_4$) is $60$, and the tuning pitch (A$_4$, $f = 440 {\rm Hz}$) is $69$.
Within the finite range $[21, 108]$, consecutive MIDI integers bijectively map to the $88$ consecutive keys on a piano (regardless of being black or white), so two MIDI numbers that are $1$, $2$, $12$ apart form a half step, a whole step, and an octave, respectively (Figure~\ref{fig:music-pitch-chord}: left).

Building from music pitches, we define a music \emph{chord} by any collection of pitches (that are expected to be heard as a whole), which is naturally represented by a vector of MIDI numbers.
More specifically, an $n$-chord is a chord consisting of $n$ pitches, so the space of all $n$-chords is $\reals^n$.
In our music application, we work with the sheet music of Bach's four-voice chorales in particular, so the chord space in this specific context is $\integers^4$ (Figure~\ref{fig:music-pitch-chord}: right).
This infinite space is our input space $X (= \integers^4)$, of which we make abstractions.

\vspace{0.1in}
\noindent
\underline{Music Abstractions.}
We study isometries of chords, \ie $\iso(\integers^4)$ in our specific context of four-voice chorales from sheet music.
We take $S = S^\star_+$ in Equation~\eqref{eqn:final-s-star-plus} as our generating set, and set $\tau = 12$, \ie considering no more than an octave when exploring periodicity of circulators (a design choice).
Thus, the group action under consideration is $\iso(\integers^4) = \langle S \rangle$ acting on $X = \integers^4$, and the abstraction semiuniverse from this generating set is
\begin{align*}
\Pi(2^S) = \pi \circ \pi'(2^S) = \{\integers^4/\langle S' \rangle \mid S' \subseteq S\}.
\end{align*}
We give examples of $S' \subseteq S$ whose induced partitions $\Pi(S') = \integers^4/\langle S'\rangle$ subsume many human-invented music terms and concepts.

First, consider $S'_1 = \tra_0^{12} := \{t'_{12{\bm{e_1}}}, t'_{12{\bm{e_2}}}, t'_{12{\bm{e_3}}}, t'_{12{\bm{e_4}}}\}$.
The induced abstraction $\Pi(S'_1)$ contains $12^4$ cells, each of which represents the concept of a \emph{pitch class vector}, \ie a vector of pitch classes.
In music, a \emph{pitch class} is a set of all pitches that are octave(s) apart, \eg the pitch class C is the set $\{\text{C}_n|n \in \integers\}$, treating a high C, a middle C, a low C indistinguishably as the pitch class C.
Therefore, two chords $(\text{B}_5, \text{G}_4, \text{E}_4, \text{C}_3)$ and $(\text{B}_2, \text{G}_3, \text{E}_4, \text{C}_5)$ are in the same partition cell as the pitch class vector $(\text{B}, \text{G}, \text{E}, \text{C})$, whereas the chord $(\text{E}_5, \text{C}_4, \text{A}_3, \text{F}_3)$ is in a different cell as the pitch class vector $(\text{E}, \text{C}, \text{A}, \text{F})$ and notably, the chord $(\text{G}_4, \text{E}_4, \text{C}_4, \text{B}_3)$ is in another different cell as the pitch class vector $(\text{G}, \text{E}, \text{C}, \text{B})$.
The point here is that, under this abstraction, the concept of a pitch class vector does not distinguish the register, but as a vector, distinguishes both the order and the multiplicity of its pitch class components.

Second, consider $S'_2 = \tra_0^{12} \cup \rot_{\perm 0}$, where $\rot_{\perm 0} := \{r'_{P^{(1,2)}}, \ldots, r'_{P^{(n-1,n)}}\}$ is the generating set of the subgroup comprising permutations (of vector coordinates).
Compared to $\Pi(S'_1)$ in the first example, the induced abstraction $\Pi(S'_2)$ here is a coarser partition.
In particular, $\Pi(S'_2)$ can be obtained by merging cells in $\Pi(S'_1)$ that represent pitch class vectors like $(\text{B}, \text{G}, \text{E}, \text{C})$ and $(\text{G}, \text{E}, \text{C}, \text{B})$ but not $(\text{E}, \text{C}, \text{A}, \text{F})$.
Therefore, every cell in $\Pi(S'_2)$ represents the concept of a \emph{pitch class multi-set}, \ie a multi-set of pitch classes.
Many of these pitch class multi-sets correspond to known music-theoretic terms.
For instance,
one cell in $\Pi(S'_2)$ comprises exclusively all C major seventh chords, musically denoted $\text{CM}7 := \{\text{C},\text{E},\text{G},\text{B}\}$;
another different cell comprises exclusively all C minor seventh chords, musically denoted $\text{Cm}7 := \{\text{C},\text{E}\flat,\text{G},\text{B}\flat\}$;
a third different cell comprises exclusively all F major seventh chords, musically denoted $\text{FM}7 := \{\text{F},\text{A},\text{C},\text{E}\}$; so forth.

Third, consider $S'_3 = \tra_0^{12} \cup \rot_{\perm 0} \cup \{t'_{\bm{1}}\}$, where the added translation $t'_{\bm{1}}$ is the single generator that generates all music transpositions (which form a cyclic subgroup).
So, the abstraction $\Pi(S'_3)$ is an even coarser partition, compared to the two partitions $\Pi(S'_1)$ and $\Pi(S'_2)$ in the first and second examples.
In particular, $\Pi(S'_3)$ can be obtained by merging cells in $\Pi(S'_2)$ that represent pitch class multi-sets like $\{\text{C},\text{E},\text{G},\text{B}\}$ and $\{\text{F},\text{A},\text{C},\text{E}\}$ but not $\{\text{C},\text{E}\flat,\text{G},\text{B}\flat\}$.
Therefore, every cell in $\Pi(S'_3)$ represents the concept of a transposition invariant multi-set of pitch classes.
As examples from this abstraction, one cell in $\Pi(S'_3)$ comprises exclusively all major seventh chords, musically denoted $\text{M}7$; another different cell comprises exclusively all minor seventh chords, musically denoted $\text{m}7$.

\begin{figure}[t]
\begin{center}
\includegraphics[width=0.98\columnwidth]{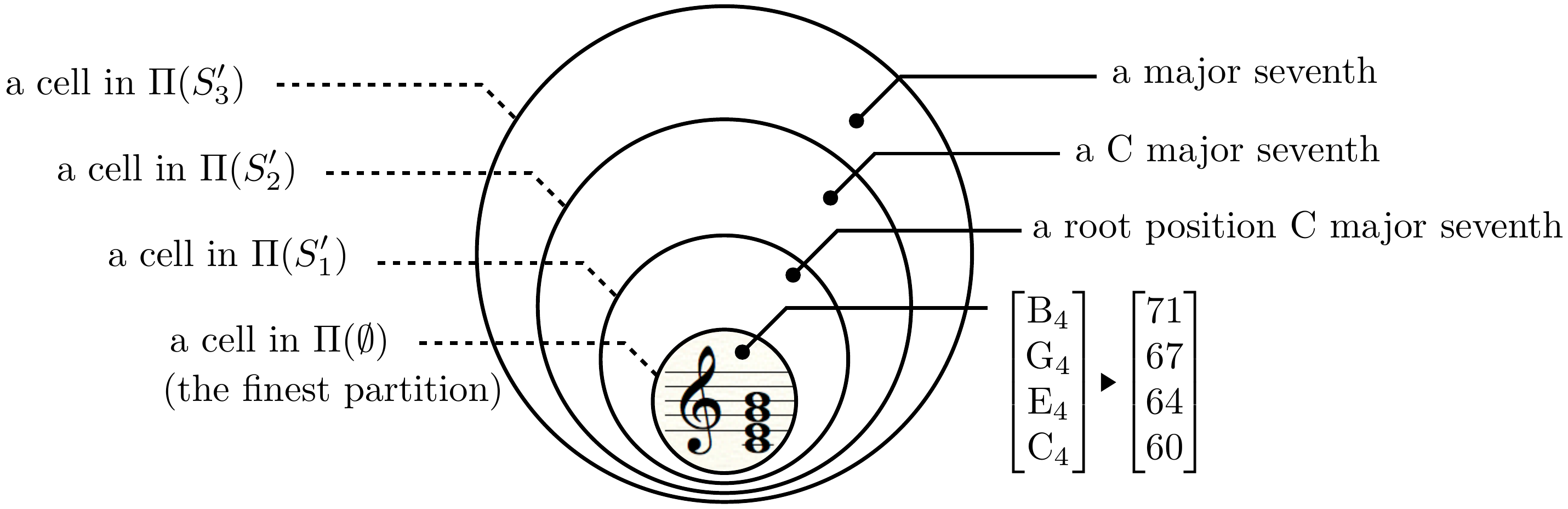}
\end{center}
\caption{The MIDI representation of music pitches (left) and a music chord (right).}
\label{fig:music-concept-hierarchy}
\end{figure}

To reveal part of the hierarchy, the three examples are picked in a way ($S'_1 \subseteq S'_2 \subseteq S'_3$) such that they clearly form a sequence of coarser and coarser partitions of our chord space $X = \integers^4$ ($\Pi(S'_1) \succeq \Pi(S'_2) \succeq \Pi(S'_3)$).
This implies that when we further zoom in and inspect things at the level of cells, we can trace out a nested sequence of cells, each of which is from a different partition in the coarsening sequence.
The nested cells can be used as sequential descriptions of an element $x \in X$, \ie a chord in our music example, from more specific descriptions to more general ones.
For instance, considering a specific chord $(\text{B}_4, \text{G}_4, \text{E}_4, \text{C}_4)$ which is $(71,67,64,60) \in \integers^4$ and looking at it through a series of nested cells from the coarsening sequence $\Pi(S'_1) \succeq \Pi(S'_2) \succeq \Pi(S'_3)$, we can say that it is a special type of root position C major seventh chord, and more generally, a C major seventh chord, and even more generally, a major seventh chord (Figure~\ref{fig:music-concept-hierarchy}).
During this sequential process, we are incrementally forgetting the register, the inversion, and the root of a chord one after another, moving toward higher and higher-level music abstractions.
On a last note, there are certainly partitions that are incomparable with this line of partitions.

Finally when coming down to real implementations, note that we cannot compute partitions of the infinite chord space $X = \integers^n$;
we have to consider a finite subspace $Y \subseteq X$.
It is natural to take a hypercube $Y = \integers^4_{[a,b]}$, where the range $[a,b]$ typically denotes the range of a music instrument (from the lowest note $a$ to the highest note $b$).
For instance, in our case regarding chorales, $[a,b]$ denotes the finite human vocal range.
Now we can run the induction algorithm (Section~\ref{sec:an-induction-algorithm}) equipped with the ``Expand-and-Restrict'' technique (Section~\ref{sec:restrict-to-finite-subspaces}) for our music application.
In particular, the chord space $X = \integers^4$ is ``spatially stationary'', so all implementation tricks introduced in Section~\ref{sec:an-implementation-example} are applicable.

It is worth noting that in each of the above three abstractions $\Pi(S'_1)$, $\Pi(S'_2)$, and $\Pi(S'_3)$, there are many cells representing musically unnamed chord types, as well as cells that are traditionally unnamed but were later named in music from more recent eras, \eg in secundal, quartal harmonies.
It is even more important to note that the value of our symmetry-driven abstractions is not just to reproduce known music terms/concepts with a group-theoretic reinterpretation, but to reveal all other traditionally less concerned concepts which yet fall under the same symmetry category.
The ability to present a bigger abstraction ``map'' from various symmetry families/subfamilies can systematically suggest new composition possibilities and guidelines for new music, and the ability to computationally build it allows faster explorations and experimentations in music composition through algorithms.

\paragraph{Learning Music Information Lattice.}
\begin{figure}[t]
\begin{center}
\includegraphics[width=0.98\columnwidth]{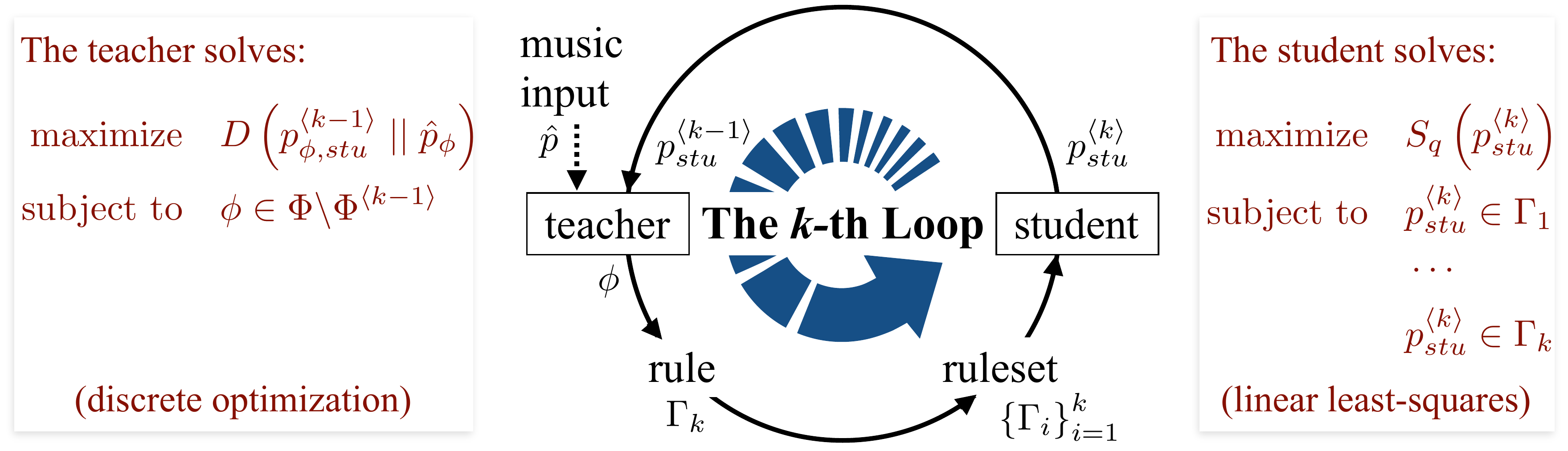}
\end{center}
\caption{MUS-ROVER's self-learning loop (the $k$th iteration). The teacher (discriminator) takes as inputs the student's latest style $p_{stu}^{\langle k-1 \rangle}$ and the input style $\hat{p}$, and identifies an abstraction $\phi$ under which the two styles manifest the largest statistical gap $D(\cdot ||\cdot)$. The identified abstraction is then made into a rule (a constraint set $\Gamma_k$), and augments the ruleset $\{\Gamma_i\}_{i=1}^k$. The student (generator) takes as input the augmented ruleset to update its writing style into $p_{stu}^{\langle k \rangle}$, meanwhile favors novelty, \ie more possibilities, by maximizing the Tsallis entropy $S_q$ subject to the rule constraints. In short, the teacher extracts rules while the student applies rules; both perform their tasks by solving optimization problems.}
\label{fig:loop}
\end{figure}

Built on top of computational music abstraction, MUS-ROVER is an automatic music theorist/pedagogue, which self-learns music composition rules from symbolic music datasets and provides personalized composition lessons.
Notably, as detailed in our earlier work \citep{YuVGK2016, YuV2017, YuLV2017}, the learning is robust to small amounts of random noise in the symbolic datasets due to the way probabilistic rules are defined therein and due to the continuity of information functionals such as entropy.
If there are systematic distortions in the dataset rather than random noise, however, these distortions may well be learned as rules.

Rules are learned from a ``teacher $\rightleftharpoons$ student'' model (Figure~\ref{fig:loop}).
This model is realized by a self-learning loop between a \emph{discriminative}
component (teacher) and a \emph{generative} component (student), where both entities cooperate to iterate through the rule-learning process.
The student starts as a \emph{tabula rasa} that picks pitches uniformly at random to form chords and chord progressions.
In each iteration, the teacher compares the student's writing style (represented by a probabilistic model) with the input style (represented by empirical statistics) to identify one music abstraction (represented by a partition of the chord space from our aforementioned computational music abstraction) that best reveals the gap between the two styles; and then make it a rule for the student.
Consequently, the student becomes less and less random by obeying more and more rules, and thus, approaches the input style.
From its rule-learning process on a dataset consisting of Bach's chorales, MUS-ROVER successfully recovered many known rules such as ``parallel perfect octaves/fifths are rare'' and ``tritones are often resolved either inward or outward'';
MUS-ROVER also suggests new probabilistic patterns on new music abstractions such as a temporarily unresolved chord progression behaving like a ``harmonic'' version of escape tone or changing tone.

In MUS-ROVER's self-learning loop, both the teacher and the student perform their tasks, namely rule extraction and rule realization respectively, by solving optimization problems.
For the student, the rule realization problem is about finding the most random probability distribution over the chord space (\ie maximizing novelty) as long as it satisfies all the probabilistic rules.
This is formalized as the following optimization problem:
\begin{align*}
\mbox{maximize }  \quad & S_q(x) \\
\mbox{subject to } \quad & x \in \Gamma_k, \quad k = 1, \ldots, K,
\end{align*}
where the optimization variable $x$ denotes the probability distribution over the chord space (note: we pre-specify a finite range of the pitches under consideration, \eg piano range, vocal range, so the chord space is finite and $x$ is a vector), the objective $S_q(x)$ is the Tsallis entropy of $x$ measuring the randomness of $x$, and the constraint sets $\Gamma_1, \ldots, \Gamma_K$ denote $K$ rules learned so far.
We mention two facts here.
First, in the limit as $q \to 1$, $S_q(x) \to H(x)$ which is the Shannon entropy.
Second, $x \in \Gamma_k$ is more explicitly represented as a linear equation $A^kx = y^k$ where the pair $(A^k,y^k)$ denotes the $k$th rule.
More specifically, $A^k$ is a boolean matrix (called a partition matrix) which stores the full information of a partition: $A^k_{ij} = 1$ if and only if the $j$th chord belongs to the $i$th partition cell; $y^k$ is a probability distribution over the partition cells.
We slightly overload the notations and let $\bm{x}$ and $\bm{y^k}$ be the information elements that represent the probability space $(X, \Sigma(I), x)$ and $(X, \Sigma(A^k), y^k)$, respectively.
In this notation, the sample space $X$ is the chord space and, for a partition matrix $P$, $\Sigma(P)$ denotes the $\sigma$-algebra generated by the partition represented by $P$ (so $\Sigma(I)$ denotes the $\sigma$-algebra generated by the finest partition).
Under this setting, the equality constraint $A^kx = y^k$ becomes $H(\bm{y^k}|\bm{x}) = 0$, \ie $\bm{y^k}$ is an abstraction of $\bm{x}$ as information elements or $\bm{y^k} \leq \bm{x}$ by Shannon's definition.
Therefore, the student's optimization problem for rule realization can be rewritten as follows:
\begin{align*}
\mbox{maximize } \quad & H(\bm{x}) \\
\mbox{subject to } \quad & H(\bm{y^k}|\bm{x}) = 0, \quad k = 1, \ldots, K,
\end{align*}
which describes an inference problem in an information lattice.
In words, this means that we want to find an information lattice for the student (used as its mental model) such that it agrees on all $K$ abstractions from the information lattice for the dataset and meanwhile achieves the largest randomness in the chord space.
As a result, a good student memorizes high-level principles---rules in terms of high-level abstractions and their statistical patterns---in music composition rather than the actual pieces.
Indeed the student is encouraged to be as creative as possible as long as the high-level principles are satisfied.

Since learning an information lattice requires both the explicit construction of an abstraction (semi)universe and statistical inference from data, the learning paradigm in MUS-ROVER differs from pure rule-based systems or pure data-driven methods in artificial intelligence, creating a middle ground between the two extremes.
Developmental cognitive scientists say that this resembles the way babies learn from both empirical experience and biological instinct \citep{Hutson2018}.
In this example, the constructed abstraction semiuniverse resembles biological instincts on perceiving sound, and the statistical inference in this semiuniverse resembles experiential learning driven by the instincts.
Therefore, the entire concept learning process is transparent, in contrast with black-box algorithms.
In particular, when we run MUS-ROVER on Bach's four-voice chorales, we can actually visualize the learning process in the information lattice as a process that mimics Bach's mental activities during chorale composition (Figure~\ref{fig:bach_brain_ngram_14_loop10}).

\begin{figure}[t]
\begin{center}
\includegraphics[width=0.95\columnwidth]{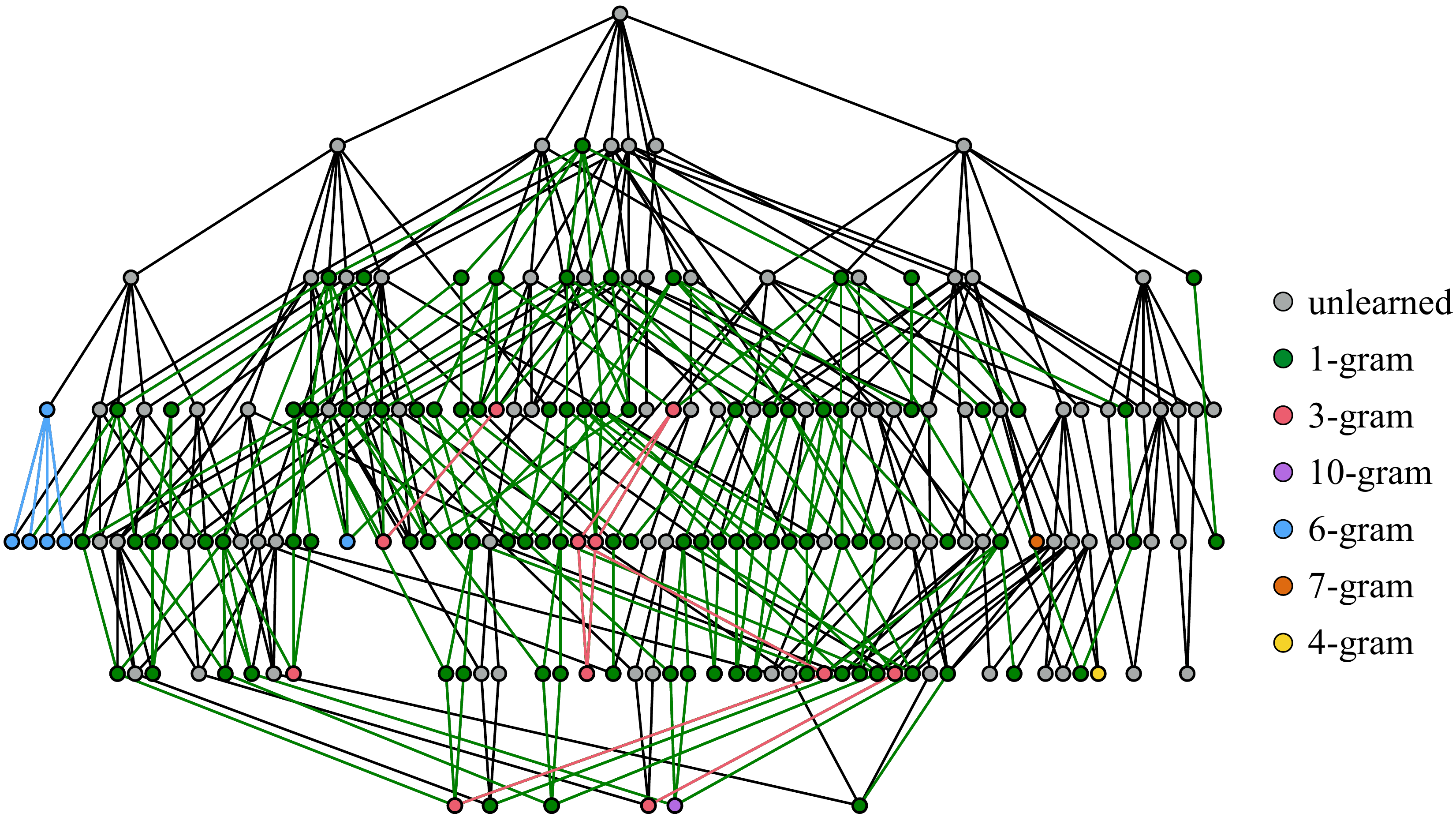}
\end{center}
\caption{Visualization of Bach's music mind for writing chorales. The underlying directed acyclic graph signifies an  upside-down information lattice. (Note: edges are oriented upwards according to the convention of a partition lattice; the coarsest partition at the bottom is omitted.) Colors are used to differentiate rule activations from different $n$-gram settings.}
\label{fig:bach_brain_ngram_14_loop10}
\end{figure}

\newpage
\acks{This work was funded in part by the IBM-Illinois Center for Cognitive Computing Systems Research (C3SR), a research collaboration as part of the IBM AI Horizons Network; and in part by grant number 2018-182794 from the Chan Zuckerberg Initiative DAF, an advised fund of Silicon Valley Community Foundation.}


\bibliography{abrv,conf_abrv,group_abstr}

\renewcommand{\theHsection}{A\arabic{section}}
\appendix

\section{Mathematical Preliminaries}

\subsection{For Section~\ref{sec:abstraction-as-partition}}
\label{app:abstraction-as-partition}
A \emph{partition} $\pcal$ of a set $X$ is a collection of mutually disjoint non-empty subsets of $X$ whose union is $X$.
Elements in $\pcal$ are called \emph{cells} (or less formally, \emph{clusters}); the size of $\pcal$ is $|\pcal|$, \ie the number of cells in $\pcal$.
An \emph{equivalence relation} on a set $X$, denoted $\sim$, is a binary relation satisfying reflexivity, symmetry, and transitivity.
An equivalence relation $\sim$ on $X$ induces a partition of $X$: $\pcal = X{/\!\!\sim} := \{[x]\mid x \in X\}$, where the quotient $X{/\!\!\sim}$ is the set of \emph{equivalence classes} $[x] := \{x' \in X \mid x \sim x'\}$. Conversely, a partition $\pcal$ of $X$ also induces an equivalence relation $\sim$ on $X$: $x\sim x'$ if and only if $x,x'$ are in the same cell in $\pcal$.

\subsection{For Section~\ref{sec:abstraction-universe-as-partition-lattice}}
\label{app:abstraction-universe-as-partition-lattice}
A \emph{partial order} is a binary relation that satisfies reflexivity, antisymmetry, and transitivity. A \emph{lattice} is a partially ordered set (or a \emph{poset}) in which every pair of elements has a unique supremum (\ie least upper bound) called the \emph{join} and a unique infimum (\ie greatest lower bound) called the \emph{meet}. For any pair of elements $p,q$ in a lattice, we denote their join and meet by $p \vee q$ and $p \wedge q$, respectively.
A \emph{sublattice} is a nonempty subset of a lattice, which is closed under join and meet.
A \emph{join-semilattice} (resp.\ \emph{meet-semilattice}) is a poset in which every pair of elements has a join (resp.\ meet). So, a lattice is both a join-semilattice and a meet-semilattice. A lattice is \emph{bounded} if it has a greatest element and a least element.

\subsection{For Section~\ref{sec:symmetry-generated-abstraction}}
\label{app:symmetry-generated-abstraction}
A \emph{group} is a pair $(G, *)$ where $G$ is a set and $*: G \times G \to G$ is a binary operation satisfying the \emph{group axioms}: associativity, the existence of identity (denoted $e$), and the existence of inverse. We also directly say that $G$ is a group, whenever the group operation is understood. Given a group $(G,*)$, a subset $H \subseteq G$ is a \emph{subgroup}, denoted $H \leq G$, if $(H,*)$ is a group.
The singleton $\{e\}$ is a subgroup of any group, called the \emph{trivial group}.
Given a group $G$ and a subset $S \subseteq G$, the subgroup (of $G$) generated by $S$, denoted $\langle S \rangle$, is the smallest subgroup of $G$ containing $S$; equivalently, $\langle S \rangle$ is the set of all finite products of elements in $S \cup S^{-1}$ where $S^{-1} := \{s^{-1} \mid s \in S\}$.
The subgroup generated by a singleton $S = \{s\}$ is called a \emph{cyclic group}; for simplicity, we also call it the subgroup generated by $s$ (an element), denoted $\langle s \rangle$.
Let $G$ be a group and $X$ be a set, then a \emph{group action} of $G$ on $X$ (or $G$-action on $X$) is a function $\cdot: G\times X \to X$ that satisfies identity ($e\cdot x = x, \forall x \in X$) and compatibility ($g\cdot (h\cdot x) = (g*h)\cdot x, \forall g,h \in G, \forall x \in X$).
In this paper, we adopt by default the multiplicative notation for group operations and actions, in which $\cdot$ or $*$ or both may be omitted.
For any $G$-action on $X$, the \emph{orbit} of $x$ under $G$ is the set $Gx: = \{g \cdot x \mid g\in G\}$, and the \emph{quotient} of $X$ by $G$-action is the set consisting of all orbits $X/G := \{Gx \mid x \in X\}$.

\subsection{For Section~\ref{sec:duality-subgroup-to-abstraction}}
\label{app:duality_subgroup_to_abstraction}
Let $G$ be a group. We use $\hcal_G^{*}$ to denote the collection of all subgroups of $G$. The binary relation ``a subgroup of'' on $\hcal_G^{*}$, denoted $\leq$, is a partial order. $(\hcal_G^{*},\leq)$ is a lattice, called the \emph{lattice of all subgroup} of $G$, or the \emph{(complete) subgroup lattice} for $G$ in short.
For any pair of subgroups $A,B \in \hcal_G^{*}$, the join $A \vee B = \langle A\cup B \rangle$ is the smallest subgroup containing $A$ and $B$; the meet $A \wedge B = A \cap B$ is the largest subgroup contained in $A$ and $B$.

\subsection{For Section~\ref{sec:more-duality-subgroup-to-abstraction}}
\label{app:more_duality_subgroup_to_abstraction}
Let $G$ be a group. We say that two elements $a,b \in G$ are \emph{conjugate} to each other, if there exists a $g \in G$ such that $b = gag^{-1}$, and two subsets $A,B \subseteq G$ are \emph{conjugate} to each other, if there exists a $g \in G$ such that $B = gAg^{-1}$. In either case, conjugacy is an equivalence relation on $G$ (resp.\ $2^G$, \ie the power set of $G$), where the equivalence class of $a \in G$ (resp.\ $A \subseteq G$) is called the \emph{conjugacy class} of $a$ (resp.\ $A$). In particular, we can restrict the above equivalence relation on $2^G$ to the collection of all subgroups $\hcal_{G}^{*}$ which is a subset of $2^G$.

\subsection{For Section~\ref{sec:the-top-down-approach-sp-subgroups}}
\label{app:the-top-down-approach-sp-subgroups}
Let $(G,*)$ and $(H,\cdot)$ be two groups. A function $\phi: G \to H$ is called a \emph{homomorphsim} if $\phi(a*b) = \phi(a)\cdot \phi(b)$ for all $a,b \in G$. An \emph{isomorphism} is a bijective homomorphism. We say two groups $G$ and $H$ are \emph{homomorphic} if there exists a homomorphism $\phi: G \to H$, and say they are \emph{isomorphic}, denoted $G \cong H$, if there exists an isomorphism $\phi: G \to H$.
Let $S$ be a subset of a group $G$; then ${\rm N}_G(S) := \{g \in G \mid gSg^{-1} = S\}$ is called the \emph{normalizer} of $S$ in $G$, which is a subgroup of $G$. We say a subset $T \subseteq G$ \emph{normalizes} another subset $S \subseteq G$ if $T \subseteq {\rm N}_G(S)$.
We say a subgroup $N$ of a group $G$ is a \emph{normal subgroup} of $G$, denoted $N \trianglelefteq G$, if $G$ normalizes $N$, \ie $G = {\rm N}_G(N)$.
Let $G$ be a group, $N \trianglelefteq G$, $H \leq G$, $N \cap H = \{e\}$, and $G = NH$; then $NH$ is the \emph{inner semi-direct product} of $N$ and $H$, and $N \rtimes H$ is the \emph{outer semi-direct product} of $N$ and $H$. The outer semi-direct product $N \rtimes H$ is the group of all ordered pairs $(n,h) \in N\times H$ with group operation defined by $(n,h)(n',h') = (nhn'h^{-1},hh')$. The inner and outer semi-direct products are isomorphic, \ie $NH \cong N \rtimes H$. The semi-direct product equation $G = NH$ gives a decomposition of $G$ into ``nearly non-overlapping'' (\ie with trivial intersection) subgroups; moreover, for any $g \in G$, these exist a unique $n \in N$ and $h \in H$ such that $g = nh$.


\section{Mathematical Proofs}

\subsection{Theorem~\ref{thm:group-action-layer-up}}
\label{app:group-action-layer-up}
\begin{proof}
Pick any $g \in G$ and $Y \in 2^X$. For any $x \in g\cdot Y$, we have $x = g\cdot y$ for some $y \in Y$. Since $Y \in 2^X$, \ie $Y \subseteq X$, then $y \in X$. This implies that $x = g\cdot y \in X$. Therefore, $g\cdot Y \subseteq X$, \ie $g \cdot Y \in 2^X$. To see the corresponding function $\cdot: G \times 2^X \to 2^X$ is a $G$-action on $2^X$, we first check that the identity element $e \in G$ satisfies $e \cdot Y = \{e \cdot y \mid y \in Y\} = \{y \mid y \in Y\} = Y$; then check that for any $g, h \in G$, $g \cdot (h \cdot Y) = \{g\cdot z \mid z \in \{h \cdot y \mid y \in Y\}\} = \{g \cdot (h\cdot y) \mid y \in Y\} = \{(gh)\cdot y \mid y \in Y\} = (gh) \cdot Y$.

Pick any $g \in G$ and $\pcal \in \mathfrak{P}_X^{*}$. For any distinct elements $Q,Q' \in g\cdot \pcal$, we have $Q = g \cdot P$ and $Q' = g \cdot P'$ for some distinct $P,P' \in \pcal$, respectively. Since $P,P'$ are two distinct cells in partition $\pcal$, $P\cap P' = \emptyset$. We claim that $Q \cap Q' = \emptyset$. Assume otherwise, then there exists a $q \in Q \cap Q'$ and $q = g \cdot p = g\cdot p'$ for some $p \in P, p' \in P'$. This implies that $p = (g^{-1}g)\cdot p = g^{-1} \cdot (g \cdot p) = g^{-1} \cdot (g \cdot p') = (g^{-1}g) \cdot p' = p' \in P \cap P'$, which contradicts the fact that $P \cap P' = \emptyset$. For any $x \in X$, $g^{-1}\cdot x \in X$, then there exists a cell $P \in \pcal$ such that $g^{-1}\cdot x \in P$. This implies that $x = (gg^{-1})x = g \cdot (g^{-1}\cdot x) \in g\cdot P$ which is an element in $g \cdot \pcal$. Therefore, the union of all elements in $g\cdot \pcal$ covers $X$, or more precisely, equals $X$, since every element in $g\cdot \pcal$ is a subset of $X$. Hence, $g \cdot \pcal$ is indeed a partition of $X$, \ie $g \cdot \pcal \in \mathfrak{P}_X^{*}$.
To see the corresponding function $\cdot: G \times \mathfrak{P}_X^{*} \to \mathfrak{P}_X^{*}$ is a $G$-action on $\mathfrak{P}_X^{*}$, we first check that the identity element $e \in G$ satisfies $e \cdot \pcal = \{e \cdot P \mid P \in \pcal\} = \{P \mid P \in \pcal\} = \pcal$; then check that for any $g, h \in G$, $g\cdot(h\cdot \pcal) = \{g\cdot Q \mid Q \in \{h\cdot P \mid P \in \pcal\}\} = \{g \cdot (h \cdot P) \mid P \in \pcal\} = \{(gh)\cdot P \mid P \in \pcal\} = (gh) \cdot \pcal$.
\end{proof}

\subsection{Theorem~\ref{thm:in-the-same-coset-means-same-linear-part}}
\label{app:in-the-same-coset-means-same-linear-part}
\begin{proof}
It is straightforward to check that
\begin{align*}
T \circ h = T \circ h' \iff h'\circ h^{-1} \in T \iff \ell(h'\circ h^{-1}) = I \iff \ell(h') = \ell(h).
\end{align*}
The last if-and-only-if condition holds because $\ell(h'\circ h^{-1}) = \ell(h')\ell(h)^{-1}$ by Lemma~\ref{lemma:linear-part-is-homomorphism}.
\end{proof}

\subsection{Theorem~\ref{thm:in-the-same-coset-translation-part-diff}}
\label{app:in-the-same-coset-translation-part-diff}
\begin{proof}
Let $h,h' \in H$ be any two affine transformations in the same coset in $H/T$, then this means $T \circ h = T \circ h'$, or equivalently $h'\circ h^{-1} \in T$. By Equation~\eqref{eqn:composition-and-inverse-of-affine-pair}, we have
\begin{align*}
\tau(h' \circ h^{-1}) = \tau(h')+\ell(h')\tau(h^{-1}) = \tau(h') + \ell(h')(-\ell(h)^{-1}\tau(h)) = \tau(h') - \tau(h),
\end{align*}
where the last equality holds by Theorem~\ref{thm:in-the-same-coset-means-same-linear-part}.
Therefore, $\tau(h') - \tau(h) \in \tau(T)$.
\end{proof}

\subsection{Theorem~\ref{thm:compatibility-via-group-action}}
\label{app:compatibility-via-group-action}
\begin{proof}
For any $A \in \ell(H)$ and $v \in \tau(T)$, there exists an $f_{A,u} \in H$ and an $f_{I,v} \in T$. Since $T \trianglelefteq H$, then $f_{A,u} \circ f_{I,v} \circ f_{A,u}^{-1} \in T$. By Equation~\eqref{eqn:composition-and-inverse-of-affine-pair} we have that $f_{A,u} \circ f_{I,v} \circ f_{A,u}^{-1} = f_{I,Av}$, so $f_{I,Av} \in T$, \ie $Av = \tau(f_{I,Av}) \in \tau(T)$.
To see $\cdot: \ell(H) \times \tau(T) \to \tau(T)$ defines a group action of $\ell(H)$ on $\tau(T)$ is then easy, since it is a matrix-vector multiplication. A quick check shows that for any $v \in \tau(T)$, $I\cdot v = v$; for any $v \in \tau(T)$ and $A,B \in \ell(H)$, $A\cdot(B\cdot v) = (AB)\cdot v$.
\end{proof}

\subsection{Lemma~\ref{lemma:vector-system-properties}}
\label{app:vector-system-properties}
\begin{proof}
For any $A \in L$, $\xi(A) = \xi(IA) = \xi(I) + I\xi(A) = \xi(I) + \xi(A)$. Note that $\xi(A), \xi(I) \in \reals^n/V$, so $\xi(A) = V + a$ and $\xi(I) = V + b$ for some $a,b \in \reals^n$. Thus,
\begin{align*}
\xi(A) = \xi(I) + \xi(A) \implies V + a = V + (b+a).
\end{align*}
This further implies that $b \in V$ and $\xi(I) = V + b = V$.

For any $A \in L$, $V = \xi(A^{-1}A) = \xi(A^{-1}) + A^{-1}\xi(A)$. Note that $\xi(A),\xi(A^{-1}) \in \reals^n/V$, so $\xi(A) = V + a$ and $\xi(A^{-1}) = V + c$ for some $a,c\in \reals^n$. Thus,
\begin{align*}
V = \xi(A^{-1}) + A^{-1}\xi(A) \implies V = V + (c + A^{-1}a).
\end{align*}
This further implies that $c+A^{-1}a \in V$, or equivalently, $c \in V + (-A^{-1}a)$. Therefore, $c$ and $-A^{-1}a$ are in the same coset and $\xi(A^{-1}) = V + c = V + (-A^{-1}a) = -A^{-1}\xi(A)$.
\end{proof}

\subsection{Theorem~\ref{thm:affine-group-identification}}
\label{app:affine-group-identification}
\begin{proof}
Let $\Psi: \hcal_{\affine(\reals^n)}^{*} \to \Sigma$ be the function defined by
\begin{align*}
\Psi(H) := (\ell(H),~ \tau(T),~ \xi_H) \quad \mbox{ for any } H \in \hcal_{\affine(\reals^n)}^{*},
\end{align*}
where $T := \tra(\reals^n) \cap H$, and $\xi_H: \ell(H) \to \reals^n/\tau(T)$ is given by $\xi_H(A) = \tau(\bar{\ell}^{-1}(A))$ with $\bar{\ell}: H/T \to \ell(H)$ being the isomorphism defined in the proof of Theorem~\ref{thm:characterize-linear-part}.
We first show $\Psi$ is well-defined, and then show it is bijective.
The entire proof is divided into four parts.

\paragraph{1. Check that $\xi_H$ is well-defined.}
More specifically, we want to show that
\begin{align*}
\xi_H(A) \in \reals^n/\tau(T) \quad \mbox{ for any } H \in \hcal_{\affine(\reals^n)}^{*} \mbox{ and } A \in \ell(H).
\end{align*}
For any $A \in \ell(H)$, $\bar{\ell}^{-1}(A)$ is the coset $T\circ h$ in $H/T$ such that $\ell(h) = A$.
Pick any $h \in \bar{\ell}^{-1}(A)$ which is possible since as a coset $\bar{\ell}^{-1}(A) \neq \emptyset$.
For any $h' \in \bar{\ell}^{-1}(A)$, by Theorem~\ref{thm:in-the-same-coset-translation-part-diff}, $\tau(h')-\tau(h) \in \tau(T)$, \ie $\tau(h') \in \tau(T)+\tau(h)$, so $\tau(\bar{\ell}^{-1}(A)) \subseteq \tau(T) + \tau(h)$.
Conversely, for any $w \in \tau(T) + \tau(h)$, there exists a $v \in \tau(T)$ such that $w = v+\tau(h)$. Note that the pure translation $t_v \in T\leq H$ and $h \in \bar{\ell}^{-1}(A) \subseteq H$, so their composition $t_v \circ h \in H$. Further, it is an easy check that $\ell(t_v \circ h) = A$ and $\tau(t_v \circ h) = v + \tau(h) = w$. This implies that we have found $h' := t_v \circ h \in \bar{\ell}^{-1}(A)$ and $\tau(h') = w$, thus, $w \in \tau(\bar{\ell}^{-1}(A))$. This finally yields that $\tau(T) + \tau(h) \subseteq \tau(\bar{\ell}^{-1}(A))$. Combining the two directions, we have $\tau(\bar{\ell}^{-1}(A)) = \tau(T) + \tau(h)$; so, $\xi_H(A) = \tau(\bar{\ell}^{-1}(A)) \in \reals^n/\tau(T)$. This implies $\xi_H$ is well-defined.

\paragraph{2. Check that $\Psi$ is well-defined.}
More specifically, we want to show that
\begin{align*}
\Psi(H) \in \Sigma \quad \mbox{ for any } H \in \hcal_{\affine(\reals^n)}^{*}.
\end{align*}
For any $H \in \hcal_{\affine(\reals^n)}^{*}$, it is clear that $\ell(H) \leq \genlin_n(\reals)$, $\tau(T) \leq \reals^n$, and they are compatible (Theorem~\ref{thm:compatibility-via-group-action}); therefore, it suffices to show that $\xi_H \in \Xi_{\ell(H), \tau(T)}$.
Note that, for any $A,A' \in \ell(H)$, the product of two cosets $\bar{\ell}^{-1}(A)\bar{\ell}^{-1}(A') = (T \circ f_{A,u})(T \circ f_{A',u'}) = T \circ (f_{A,u} \circ f_{A',u'}) = T \circ f_{AA', u+Au'}$, for some $f_{A,u}, f_{A',u'} \in H$.
Therefore,
\begin{align*}
\xi_H(AA') = \tau(\bar{\ell}^{-1}(AA')) = \tau(\bar{\ell}^{-1}(A)\bar{\ell}^{-1}(A')) = \tau(T \circ f_{AA', u+Au'}) = \tau(T) + u + Au'.
\end{align*}
On the other hand,
\begin{align*}
\xi_H(A)+A\xi_H(A') = (\tau(T) + u) + A(\tau(T) + u') = \tau(T) + u + Au'.
\end{align*}
Therefore, $\xi_H(AA') = \xi_H(A)+A\xi_H(A')$ and $\xi_H \in \Xi_{\ell(H),\tau(T)}$.
This implies that for any $H \in \hcal_{\affine(\reals^n)}^{*}$, $\Psi(H) \in \Sigma$, so $\Psi$ is well-defined.

\paragraph{3. Check that $\Psi$ is injective.}
Pick any $H,H' \in \hcal_{\affine(\reals^n)}^{*}$ and suppose $\Psi(H) = \Psi(H')$, \ie $(\ell(H), \tau(T), \xi_H) = (\ell(H'), \tau(T'), \xi_{H'})$, where $T: = \tra(\reals^n) \cap H$ and $T' := \tra(\reals^n) \cap H'$.
For any $f_{A,u} \in H$, $A = \ell(f_{A,u}) \in \ell(H) = \ell(H')$; thus, there exists some $f_{A,u'} \in H'$.
Let $\bar{\ell}: H/T \to \ell(H)$ and $\bar{\ell}': H'/T' \to \ell(H')$ be the isomorphisms similarly defined as in Theorem~\ref{thm:characterize-linear-part}.
As proved earlier, we have
\begin{align*}
\xi_H(A) &= \tau(\bar{\ell}^{-1}(A)) = \tau(T) + \tau(f_{A,u}) = \tau(T) + u, \\
\xi_{H'}(A) &= \tau(\bar{\ell}'^{-1}(A)) = \tau(T') + \tau(f_{A,u'}) = \tau(T) + u'.
\end{align*}
Therefore, $\tau(T) + u = \tau(T) + u'$.
This implies that $\tau(f_{A,u}) = u \in \tau(T) + u' = \tau(\bar{\ell}'^{-1}(A))$.
So, $f_{A,u} \in \bar{\ell}'^{-1}(A) \subseteq H'$, and $H \subseteq H'$.
By a completely symmetrical process, $H' \subseteq H$.
Therefore, $H = H'$, which implies that $\Psi$ is injective.

\paragraph{4. Check that $\Psi$ is surjective.}
Pick any $(L,V,\xi) \in \Sigma$ and let
\begin{align*}
H := \{f_{A,u} \in \affine(\reals^n) \mid A \in L, u \in \xi(A)\}.
\end{align*}
We first show that $H \leq \affine(\reals^n)$ by a subgroup test. It is clear $H \subseteq \affine(\reals^n)$. The identity matrix $I \in L$ and $\bm{0} \in V = \xi(I)$, so the identity transformation $\id = f_{I,\bm{0}} \in H$. For any $f_{A,u}, f_{A',u'} \in H$, we have $A,A' \in L$ and $u \in \xi(A), u' \in \xi(A')$, which respectively implies that $AA' \in L$ and $u+Au' \in \xi(A)+A\xi(A') = \xi(AA')$. So, $f_{A,u} \circ f_{A',u'} = f_{AA', u+Au'} \in H$.
For any $f_{A,u} \in H$, we have $A \in L$ and $u \in \xi(A)$, which respectively implies that $A^{-1}\in L$ and $-A^{-1}u \in -A^{-1}\xi(A) = \xi(A^{-1})$. So, $f_{A,u}^{-1} = f_{A^{-1}, -A^{-1}u} \in H$.
Therefore, $H \leq \affine(\reals^n)$.
Now we show that $\Psi(H) = (\ell(H),~ \tau(T),~ \xi_H) = (L,V,\xi)$.
First, for any $A \in \ell(H)$, there exists an $f_{A,u} \in H$, so $A \in L$ which implies that $\ell(H) \subseteq L$.
Conversely, for any $A \in L$, $\xi(A)$ is a coset in $\reals^n/V$, so $\xi(A) \neq \emptyset$. Pick any $u \in \xi(A)$, then $f_{A,u} \in H$, so $A = \ell(f_{A,u}) \in \ell(H)$ which implies $L \subseteq \ell(H)$. Combining both directions yields $\ell(H) = L$.
Second, note that $T = \tra(\reals^n) \cap H = \{f_{I,u} \in \affine(\reals^n) \mid u \in \xi(I) = V\}$, so $\tau(T) = \{u \mid u \in V\} = V$.
Third, note that $\xi_H: \ell(H) \to \reals^n/\tau(T)$ and $\xi: L \to \reals^n/V$. We have shown that $\ell(H) = L$ and $\tau(T) = V$, so $\xi_H$ and $\xi$ have the same domain and codomain. Further, for any $A \in L$, $\xi_H(A) = \tau(\bar{\ell}^{-1}(A)) = \tau(\{f_{A,u} \in \affine(\reals^n) \mid u \in \xi(A)\}) = \{u \mid u \in \xi(A)\} = \xi(A)$. So, $\xi_H = \xi$. Now we have $\Psi(H) = (L,V,\xi)$. Therefore, $\Psi$ is surjective.
\end{proof}

\subsection{Theorem~\ref{thm:fy-fxyy}}
\label{app:fy-fxyy}
\begin{proof}
Pick any $f' \in \transf(Y)$, and let $f: X \to X$ be the function given by
\begin{align*}
f(x) = \begin{cases}
f'(x), &  x \in Y; \\
x, &  x \in X\backslash Y.
\end{cases}
\end{align*}
Then it is clear that $f(Y) = f'(Y) = Y$ and $f(X\backslash Y) = X\backslash Y$.
For any $x,x'\in X$ and $f(x)=f(x')$: if $f(x) \in X\backslash Y$ then $x=x'$; otherwise $x,x'\in Y$ and $f'(x)=f'(x')$ which yields $x=x'$ since $f'$ is injective. This implies that $f$ is injective. $f$ is also surjective, since $f(X) = f(Y\cup (X\backslash Y)) = f(Y)\cup f(X\backslash Y) = Y \cup (X\backslash Y) = X$. So $f \in \transf(X)$. Further, the fact that $f(Y) = f'(Y) = Y$ implies that $f \in \transf(X)_Y$ and $f' = f|_Y \in \rstab{\transf}{X}{Y}$. Therefore, $\transf(Y) \subseteq \rstab{\transf}{X}{Y}$.
Conversely, pick any $f|_Y \in \rstab{\transf}{X}{Y}$. $f|_Y$ is injective since $f \in \transf(X)_Y \subseteq \transf(X)$ is injective; $f|_Y$ is surjective since $f|_Y(Y) = f(Y) = Y$.
So $f|_Y \in \transf(Y)$.
This implies that $\rstab{\transf}{X}{Y} \subseteq \transf(Y)$.
\end{proof}

\subsection{Theorem~\ref{thm:trz-trrzz}}
\label{app:trz-trrzz}
\begin{proof}
Pick any $t'_u \in \tra(\integers^n)$, then by definition, $u \in \integers^n$, and $t'_u(x) = x + u$, for any $x \in \integers^n$.
Let $t: \reals^n \to \reals^n$ be the function given by $t(x) = x + u$. Since $u \in \integers^n$, then $u \in \reals^n$, so $t \in \tra(\reals^n)$.
Further, note that $t(\integers^n) = \integers^n$; therefore, $t \in \tra(\reals^n)_{\integers^n}$.
It follows that $t'_u = t|_{\integers^n} \in \rstab{\tra}{\reals^n}{\integers^n}$, so $\tra(\integers^n) \subseteq \rstab{\tra}{\reals^n}{\integers^n}$.

Pick any $t' \in \rstab{\tra}{\reals^n}{\integers^n}$, then by definition, there exists a $t_u \in \tra(\reals^n)$ where $u \in \reals^n$ such that $t_u(\integers^n) = \integers^n$ and $t' = t_u|_{\integers^n}$, \ie $t'(x) = x + u$, for any $x \in \integers^n$. The condition $t_u(\integers^n) = \integers^n$ implies in particular $t_u(\bm{0}) = u \in \integers^n$. It follows that $t' \in \tra(\integers^n)$, so $\rstab{\tra}{\reals^n}{\integers^n} \subseteq {\tra}(\integers^n)$.

Pick any $r'_A \in \rot(\integers^n)$, then by definition, $A \in \orthmat_n(\integers)$, and $r'_A(x) = Ax$, for any $x \in \integers^n$. Let $r: \reals^n \to \reals^n$ be the function given by $r(x) = Ax$. Since $A \in \orthmat_n(\integers)$, then $A \in \orthmat_n(\reals)$, so $r \in \rot(\reals^n)$. Further, note that $r(\integers^n) = \integers^n$; therefore, $r \in \rot(\reals^n)_{\integers^n}$. It follows that $r'_A = r|_{\integers^n} \in \rstab{\rot}{\reals^n}{\integers^n}$, so $\rot(\integers^n) \subseteq \rstab{\rot}{\reals^n}{\integers^n}$.

Pick any $r' \in \rstab{\rot}{\reals^n}{\integers^n}$, then by definition, there exists a $r_A \in \rot(\reals^n)$ where $A \in \orthmat_n(\reals)$ such that $r_A(\integers^n) = \integers^n$ and $r' = r_A|_{\integers^n}$, \ie $r'(x) = Ax$, for any $x \in \integers^n$. The condition $r_A(\integers^n) = \integers^n$ implies in particular $A\bm{e_i} \in \integers^n$ for all $i$, \ie the columns of $A$ are from $\integers^n$. So $A \in \orthmat_n(\integers)$. It follows that $r' \in \rot(\integers^n)$, so $\rstab{\rot}{\reals^n}{\integers^n} \subseteq \rot(\integers^n)$.
\end{proof}

\subsection{Theorem~\ref{thm:isoz-isorzz}}
\label{app:isoz-isorzz}
\begin{proof}
Pick any $h' \in \rstab{\iso}{\reals^n}{\integers^n}$, then by definition, there exists an $h \in \iso(\reals^n)$ such that $h(\integers^n) = \integers^n$ and $h' = h|_{\integers^n}$. For any $x,y \in \integers^n$,
\begin{align*}
\dist(h'(x),h'(y)) = \dist(h|_{\integers^n}(x),h|_{\integers^n}(y))= \dist(h(x),h(y)) = \dist(x,y).
\end{align*}
This implies that $h' \in \iso(\integers^n)$. So, $\rstab{\iso}{\reals^n}{\integers^n} \subseteq \iso(\integers^n)$.

Conversely, pick any $h' \in \iso(\integers^n)$ and let $h'_0 = h'-h'(\bm{0})$. Note that $h'_0 \in \iso(\integers^n)$ and $h'_0(\bm{0}) = \bm{0}$. This implies that $\|h'_0(x)\|_2 = \dist(h'_0(x),h'_0(\bm{0})) = \dist(x,\bm{0}) = \|x\|_2$, for any $x \in \integers^n$. Further, for any $x,y \in \integers^n$, expanding the distance-preserving equation $\|h'_0(x)-h'_0(y)\|_2^2 = \|x-y\|_2^2$ and cancelling equal terms (\ie $\|h'_0(x)\|_2^2 = \|x\|_2^2$ and $\|h'_0(y)\|_2^2 = \|y\|_2^2$) yields
\begin{align*}
\langle h'_0(x), h'_0(y) \rangle = \langle x, y \rangle \quad \mbox{ for all } x,y \in \integers^n.
\end{align*}
Now let $h: \reals^n \to \reals^n$ be the function given by $h(x) = Ax + u$, where $A = [h'_0(\bm{e_1}),\cdots, h'_0(\bm{e_n})] \in \integers^{n\times n}$ and $u = h'(\bm{0}) \in \integers^n$.
Moreover, $\langle h'_0(\bm{e_i}), h'_0(\bm{e_j}) \rangle = \langle \bm{e_i}, \bm{e_j} \rangle = \delta_{ij}$.
So, $A$ is orthogonal, \ie $A \in \orthmat_n(\integers) \subseteq \orthmat_n(\reals)$. This implies $h \in \iso(\reals^n)$. We claim $h(\integers^n) = \integers^n$ and $h' = h|_{\integers^n}$.

To see $h(\integers^n) = \integers^n$, first pick any $x \in h(\integers^n)$, then there exists $y \in \integers^n$ such that $x = h(y) = Ay+u$. Since $A \in \integers^{n\times n}$ and $u \in \integers^n$, $x \in \integers^n$ which implies $h(\integers^n) \subseteq \integers^n$. Conversely, pick any $x \in \integers^n$. Let $y = A^\top(x-u) $, then $y \in \integers^n$ and $h(y) = Ay+u = x$. So $x \in h(\integers^n)$ which implies $\integers^n \subseteq h(\integers^n)$.

To see $h' = h|_{\integers^n}$, pick any $x \in \integers^n$, then $\langle h'_0(\bm{e_i}), h'_0(x) \rangle = \langle \bm{e_i}, x \rangle = \langle A\bm{e_i}, Ax \rangle = \langle h'_0(\bm{e_i}),Ax \rangle$, for all $i  = 1, \ldots, n$.
So, $\langle h'_0(\bm{e_i}), h'_0(x)-Ax \rangle = 0$, for all $i = 1, \ldots, n$, that is
\begin{align*}
A^\top(h'_0(x)-Ax) = \bm{0}.
\end{align*}
Multiplying both sides by $A$ yields $h'_0(x) = Ax$, for all $x \in \integers^n$. Hence,
\begin{align*}
h(x) = Ax + u = h'_0(x) + h'(\bm{0}) = h'(x) \quad \mbox{ for all } x \in \integers^n.
\end{align*}
That is, $h' = h|_{\integers^n}$. It follows that $h' \in \rstab{\iso}{\reals^n}{\integers^n}$. So, $\iso(\integers^n) \subseteq \rstab{\iso}{\reals^n}{\integers^n}$.
\end{proof}

\subsection{Theorem~\ref{thm:minimal-generating-set}}
\label{app:minimal-generating-set}
\begin{proof}
Suppose for any $s \in S$, $s \notin \langle S \backslash \{s\} \rangle$. However, $s \in \langle S \rangle$; so, $\langle S \backslash \{s\} \rangle \neq \langle S \rangle$. By definition, $S$ is a minimal generating set. On the other hand, suppose there exists an $s \in S$ such that $s \in \langle S \backslash \{s\} \rangle$, \ie $s = s_k \circ \cdots \circ s_1$ for some $k$ where $s_k, \ldots, s_1 \in (S \backslash \{s\})\cup (S \backslash \{s\})^{-1}$. Pick any $s' \in \langle S \rangle$, $s' = s'_{k'} \circ \cdots \circ s'_1$ for some $k'$ where $s'_{k'}, \ldots, s'_1 \in S \cup S^{-1}$. For any $i \in \{1, \ldots, k'\}$, if $s'_i = s$, replace it with $s_k \circ \cdots \circ s_1$; if $s'_i = s^{-1}$, replace it with $s_1^{-1} \circ \cdots \circ s_k^{-1}$; otherwise $s'_i \in (S \backslash \{s\})\cup (S \backslash \{s\})^{-1}$, do nothing. This results in an expression of $s'$ as the composition of finitely many elements in $(S \backslash \{s\})\cup (S \backslash \{s\})^{-1}$, \ie $s' \in \langle S \backslash \{s\} \rangle$. So, $\langle S \rangle \subseteq \langle S \backslash \{s\} \rangle$. It is trivial to see that $\langle S \backslash \{s\} \rangle \subseteq \langle S \rangle$ since $S\backslash \{s\} \subseteq S$. Therefore, $\langle S \backslash \{s\} \rangle = \langle S \rangle$. By definition, $S$ is not a minimal generating set.
\end{proof}

\end{document}